%% file: main.tex
\documentclass[english]{article}
\usepackage[T1]{fontenc}
\usepackage[utf8]{inputenc}
\usepackage{babel}
\usepackage{float}
\usepackage{units}
\usepackage{mathtools}
\usepackage{dsfont}
\usepackage{amsmath}
\usepackage{amsthm}
\usepackage{amssymb}
\usepackage[authoryear,round]{natbib}
\usepackage[pdfusetitle,
 bookmarks=true,bookmarksnumbered=false,bookmarksopen=false,
 breaklinks=false,pdfborder={0 0 1},backref=false,colorlinks=false]
 {hyperref}

\makeatletter

\floatstyle{ruled}
\newfloat{algorithm}{tbp}{loa}
\providecommand{\algorithmname}{Algorithm}
\floatname{algorithm}{\protect\algorithmname}

\theoremstyle{plain}
\newtheorem{thm}{\protect\theoremname}[section]
\theoremstyle{plain}
\newtheorem{lem}[thm]{\protect\lemmaname}
\theoremstyle{remark}
\newtheorem{rem}[thm]{\protect\remarkname}

\@ifundefined{date}{}{\date{}}

\usepackage[paperwidth=8.5in, paperheight=11in, margin=1in]{geometry}

\makeatletter
\newcommand{\setref}[1]{%
  \phantomsection        
  \edef\@currentlabel{#1}
}
\makeatother

\makeatother

\providecommand{\lemmaname}{Lemma}
\providecommand{\remarkname}{Remark}
\providecommand{\theoremname}{Theorem}

\begin{document}
\global\long\def\E{\mathbb{E}}%
\global\long\def\F{\mathcal{F}}%
\global\long\def\R{\mathbb{R}}%
\global\long\def\N{\mathbb{N}}%
\global\long\def\na{\nabla}%
\global\long\def\dom{\mathcal{X}}%
\global\long\def\pa{\partial}%
\global\long\def\hp{\widehat{\partial}}%
\global\long\def\argmin{\mathrm{argmin}}%
\global\long\def\hg{\widehat{g}}%

\title{Revisiting the Last-Iterate Convergence of Stochastic Gradient Methods\thanks{The preliminary version has been accepted at ICLR 2024. The extended
version was finished in November 2023 (V1), revised in March 2024
with typos corrected (V2), and revised again in November 2025 with
some improvements (V3). In the current March 2026 version (V4), further
improvements and polishing have been made.}}
\author{Zijian Liu\thanks{Stern School of Business, New York University, zl3067@stern.nyu.edu.}\and
Zhengyuan Zhou\thanks{Stern School of Business, New York University, zzhou@stern.nyu.edu.}}
\maketitle
\begin{abstract}
In the past several years, the last-iterate convergence of the Stochastic
Gradient Descent (SGD) algorithm has triggered people's interest due
to its good performance in practice but lack of theoretical understanding.
For Lipschitz convex functions, different works have established the
optimal $O(\log(1/\delta)\log T/\sqrt{T})$ or $O(\sqrt{\log(1/\delta)/T})$
high-probability convergence rates for the final iterate, where $T$
is the time horizon and $\delta$ is the failure probability. However,
to prove these bounds, all the existing works are either limited to
compact domains or require almost surely bounded noise. It is natural
to ask whether the last iterate of SGD can still guarantee the optimal
convergence rate but without these two restrictive assumptions. Besides
this important question, there are still lots of theoretical problems
lacking an answer. For example, compared with the last-iterate convergence
of SGD for non-smooth problems, only few results for smooth optimization
have yet been developed. Additionally, the existing results are all
limited to a non-composite objective and the standard Euclidean norm.
It still remains unclear whether the last-iterate convergence can
be provably extended to wider composite optimization and non-Euclidean
norms. In this work, to address the issues mentioned above, we revisit
the last-iterate convergence of stochastic gradient methods and provide
the first unified way to prove the convergence rates both in expectation
and in high probability to accommodate general domains, composite
objectives, non-Euclidean norms, Lipschitz conditions, smoothness,
and (strong) convexity simultaneously. Additionally, we extend our
analysis to obtain last-iterate convergence under heavy-tailed and
sub-Weibull noise.
\end{abstract}

\input{introduction.tex}

\input{preliminaries.tex}

\input{algorithm.tex}

\input{heavy.tex}

\input{weibull.tex}

\input{conclusion.tex}

\bibliographystyle{abbrvnat}
\bibliography{ref}

\clearpage

\appendix
\input{appendix.tex}

\end{document}

%% file: introduction.tex
\section{Introduction}

In this paper, we consider the constrained composite optimization
problem, $\min_{x\in\dom}F(x)\coloneqq f(x)+h(x)$, where both $f(x)$
and $h(x)$ are convex (but possibly satisfying additional conditions
such as strong convexity, smoothness, etc.) and $\dom\subseteq\R^{d}$
is a nonempty closed convex set. Since a true gradient is computationally
prohibitive to obtain (e.g., large-scale machine learning tasks) or
even infeasible to access (e.g., streaming data), the classic Stochastic
Gradient Descent (SGD) \citep{robbins1951stochastic} algorithm has
emerged to be the gold standard for a light-weight yet effective computational
procedure commonly adopted in production for the majority of machine
learning tasks: SGD only requires a stochastic first-order oracle
$\hp f(x)$ satisfying $\E[\hp f(x)\mid x]\in\pa f(x)$ where $\pa f(x)$
denotes the set of subgradients at $x$ and guarantees provable convergence
under certain conditions (e.g., Lipschitz condition for $f(x)$ and
finite variance on the stochastic oracle).

A particularly important problem in this area is to understand the
last-iterate convergence of SGD, which has been motivated by experimental
studies suggesting that returning the final iterate of SGD (or sometimes
the average of the last few iterates) -- rather than a running average
-- often yields a solution that works well in practice (e.g., \citet{shalev2007pegasos}).
As such, a fruitful line of literature \citep{rakhlin2011making,pmlr-v28-shamir13,pmlr-v99-harvey19a,orabona2020blog,doi:10.1137/19M128908X}
developed an extensive theoretical understanding of the non-asymptotic
last-iterate convergence rate. Loosely speaking, two optimal upper
bounds, $\widetilde{O}(1/\sqrt{T})$ for Lipschitz convex functions
and $\widetilde{O}(1/T)$ for Lipschitz strongly convex functions,
have been established for both in-expectation and high-probability
convergence when $h(x)=0$ (see Subsection \ref{subsec:related-work}
for a detailed discussion). However, to prove the high-probability
rates, existing works rely on restrictive assumptions: compact domains
or almost surely bounded noise (or both), which can simplify the analysis
but are unrealistic in lots of problems. Until today, whether these
two assumptions can be relaxed simultaneously or not still remains
unclear. Naturally, we want to ask the following question:
\begin{center}
\textit{Q1: Is it possible to prove the high-probability last-iterate
convergence of SGD for Lipschitz (strongly) convex functions without
the compact domain assumption and beyond the bounded noise?}
\par\end{center}

Compared with the fast development of non-smooth problems, the understanding
of the last-iterate convergence of SGD for smooth problems (i.e.,
the gradients of $f(x)$ are Lipschitz) is much slower. The best in-expectation
bound for smooth convex optimization under $\dom=\R^{d}$ until now
is still $O(1/\sqrt[3]{T})$ due to \citet{NIPS2011_40008b9a}, which
is far from the optimal rate $O(1/\sqrt{T})$ of the averaging output
\citep{lan2020first}. However, temporarily suppose the domain is
compact, one can immediately improve the rate from $O(1/\sqrt[3]{T})$
to $\widetilde{O}(1/\sqrt{T})$ by noticing that we can reduce the
smooth problem to the Lipschitz problem\footnote{To see why gradients are bounded in this case, we first fix a point
$x_{0}$ in the domain. Then by smoothness, there is $\left\Vert \na f(x)-\na f(x_{0})\right\Vert _{2}=O\left(\left\Vert x-x_{0}\right\Vert _{2}\right)$
for any other point $x$, which immediately implies $\left\Vert \na f(x)\right\Vert _{2}=O\left(\left\Vert x-x_{0}\right\Vert _{2}+\left\Vert \na f(x_{0})\right\Vert _{2}\right)=O\left(D+\left\Vert \na f(x_{0})\right\Vert _{2}\right)$
where $D$ is the domain diameter.} and use the known bounds from non-smooth convex optimization. Hence,
one may expect the last-iterate convergence rate of SGD for smooth
convex optimization should still be $O(1/\sqrt{T})$ for any kind
of domain. If one further considers smooth and strongly convex problems,
as far as we know, there is no formal result has been established
for the final iterate of SGD in a general domain except for the $O(1/T)$
in-expectation rate when $\dom=\R^{d}$ under the PL-condition (which
is known as a relaxation for strong convexity) \citep{pmlr-v130-gower21a,khaled2022better}.
The above discussion thereby leads us to the second main question:
\begin{center}
\textit{Q2: Does the last iterate of SGD provably converge in the
rate of $O(1/\sqrt{T})$ for smooth and convex functions and $O(1/T)$
for smooth and strongly convex functions in a general domain?}
\par\end{center}

Besides the two aforementioned questions, there are still several
important missing parts. First, recalling that our original goal is
to optimize the composite objective $F(x)=f(x)+h(x)$, it is still
unclear whether -- and if so, how -- the last-iterate convergence
of this harder problem can be proved. Moreover, the previous works
are limited to the standard Euclidean norm. Whereas, in lots of specialized
tasks, it may be beneficial to employ a general norm instead of the
$\ell_{2}$ norm to capture the non-Euclidean structure. However,
whether this extension can be done remains open. Additionally, the
proof techniques in the existing works vary in different settings,
which builds a barrier for researchers to better understand the convergence
of the last iterate of SGD. Motivated by these challenges, we would
like to ask the final question:
\begin{center}
\textit{Q3: Is there a unified way to analyze the last-iterate convergence
of stochastic gradient methods both in expectation and in high probability
to accommodate general domains, composite objectives, non-Euclidean
norms, Lipschitz conditions, smoothness, and (strong) convexity at
once?}
\par\end{center}

\subsection{Our Contributions}

We provide affirmative answers to the above three questions and establish
several new results by revisiting a simple algorithm, Composite Stochastic
Mirror Descent (CSMD) \citep{duchi2010composite}, which is based
on the famous Mirror Descent (MD) algorithm \citep{nemirovskij1983problem,BECK2003167}
and includes SGD as a special case. Specifically, our contributions
are as follows.
\begin{itemize}
\item We establish the first high-probability convergence result for the
last iterate of CSMD in general domains under sub-Gaussian noise to
answer \textit{Q}1 affirmatively.
\item We prove the last iterate of CSMD can converge in the rate of $O(1/\sqrt{T})$
for smooth convex optimization and $O(1/T)$ for smooth strongly convex
problems both in expectation and in high probability for any general
domain $\dom$, hence resolving \textit{Q2.}
\item We present a simple unified analysis that differs from the prior works
and can be directly applied to various scenarios simultaneously, thus
leading to a positive answer to \textit{Q3}.
\end{itemize}
By extending the proof idea:
\begin{itemize}
\item In Section \ref{sec:main-heavy}, we provide the first in-expectation
convergence bound under heavy-tailed noise for the last iterate of
the CSMD algorithm.
\item In Section \ref{sec:main-weibull}, we prove the first high-probability
convergence rate under sub-Weibull noise for the last iterate of the
CSMD algorithm.
\end{itemize}

\subsection{Related Work\label{subsec:related-work}}

We review the literature related to the last-iterate convergence of
plain stochastic gradient methods\footnote{To clarify, we mean the algorithm does not contain momentum or averaging
operations.} measured by the function value gap (see Subsection \ref{subsec:criterion}
for why we use this criterion) for both Lipschitz and smooth (strongly)
convex optimization. We only focus on the algorithms without momentum
or averaging since it is already known that, without further special
assumptions, both operations cannot help to improve the lower order
term $O(1/\sqrt{T})$ for general convex functions and $O(1/T)$ for
strongly convex functions. For the last iterate of accelerated or
averaging based stochastic gradient methods, we refer the reader to
\citet{nesterov2015quasi,lan2020first,orabona2021parameter} for in-expectation
rates and \citet{pmlr-v125-davis20a,NEURIPS2020_abd1c782,pmlr-v202-liu23aa,pmlr-v202-sadiev23a}
for high-probability bounds. As for the last iterate of stochastic
gradient methods for structured problems (e.g., linear regression),
the reader can refer to \citet{lei2017analysis,NEURIPS2019_2f4059ce,NEURIPS2021_b4a0e0fb,pan2022eigencurve,pmlr-v162-wu22p}
for recent progress.

\textbf{Last iterate for Lipschitz (strongly) convex functions:} \citet{rakhlin2011making}
is the first to show an in-expectation $O(1/T)$ convergence rate
for strongly convex functions. But such a bound is obtained under
the additional assumption, smoothness with respect to optimum\footnote{This means $\exists L>0$ such that $f(x)-f(x^{*})\leq\frac{L}{2}\|x-x^{*}\|^{2},\forall x\in\dom$
where $x^{*}\in\argmin_{x\in\dom}f(x)$.}, meaning their result does not hold in general. Later on, \citet{pmlr-v28-shamir13}
proves the first in-expectation last-iterate rates $O(\log T/\sqrt{T})$
and $O(\log T/T)$ for convex and strongly convex objectives, respectively.
The high-probability bounds turn out to be much harder than the in-expectation
rates. After several years, \citet{pmlr-v99-harvey19a} is the first
to establish a high-probability bound in the rate of $O(\log(1/\delta)\log T/\sqrt{T})$
and $O(\log(1/\delta)\log T/T)$ for convex and strongly convex problems
where $\delta$ is the probability of failure. Afterward, \citet{doi:10.1137/19M128908X}
improves the previous two rates to $O(\sqrt{\log(1/\delta)/T})$ and
$O(\log(1/\delta)/T)$ but with a non-standard step size schedule.
They also prove the rates $O(1/\sqrt{T})$ and $O(1/T)$ in expectation
under the new step size.

However, a main drawback for the general convex case in all the above
papers is requiring a compact domain. To our best knowledge, \citet{orabona2020blog}
is the first and the only work showing how to shave off this restriction,
and thereby obtains an $O(\log T/\sqrt{T})$ rate in expectation for
general domains yet it is unclear whether his proof can be extended
to the high-probability case or not. Until recently, \citet{doi:10.1137/24M1717762}
exhibits a new proof on how to obtain the convergence rate for the
last iterate but only for the deterministic case. Lastly, we would
like to mention that all of these prior results are built for a non-composite
objective $f(x)$ with the standard Euclidean norm.

\textbf{Last iterate for smooth (strongly) convex functions: }Compared
with Lipschitz problems, much less work is done for smooth optimization.
As far as we know, the only result showing a non-asymptotic rate for
smooth convex functions dates back to \citet{NIPS2011_40008b9a},
in which the authors prove that the last iterate of SGD on $\R^{d}$
enjoys an in-expectation rate $O(1/\sqrt[3]{T})$ under additional
restrictive assumptions (e.g., mean squared smoothness). As for the
strongly convex case, the $O(1/T)$ rate in expectation under the
PL-condition (which is known as a relaxation for strong convexity)
has been established, but only for non-composite optimization under
the Euclidean norm on the domain $\dom=\R^{d}$ \citep{pmlr-v130-gower21a,khaled2022better}.

\textbf{Lower bounds for last iterate:} Under the requirement $d=T$
where $d$ is the dimension of the problem, \citet{pmlr-v99-harvey19a}
is the first to provide lower bounds $\Omega(\log T/\sqrt{T})$ under
the step size $\Theta(1/\sqrt{t})$ for non-smooth convex functions
and $\Omega(\log T/T)$ under the step size $\Theta(1/t)$ when strong
convexity is additionally assumed. Note that these two rates are both
proved for deterministic optimization meaning that they can be also
applied to the in-expectation lower bounds. Subsequently, when $d<T$
holds, \citet{liu2021convergence} extends the above two lower bounds
to $\Omega(\log d/\sqrt{T})$ (this bound is also true for the step
size $\Theta(1/\sqrt{T})$) and $\Omega(\log d/T)$ under the same
step size in \citet{pmlr-v99-harvey19a}. As a consequence, lower
bounds $\Omega(\log(d\land T)/T)$ and $\Omega(\log(d\land T)/\sqrt{T})$
have been established for both convex and strongly convex problems
under the Lipschitz condition. For the high-probability bounds, \citet{pmlr-v99-harvey19a}
shows their two deterministic bounds will incur an extra multiplicative
factor $\Omega(\log(1/\delta))$, namely, $\Omega(\log(1/\delta)\log T/\sqrt{T})$
and $\Omega(\log(1/\delta)\log T/T)$. However, under more sophisticated
designed step sizes, better upper bounds without the $\Omega(\log T)$
factor are possible, for example, see \citet{doi:10.1137/19M128908X}
as mentioned above.

Another highly related work is \citet{pmlr-v202-liu23aa}, which presents
a generic approach to establish the high-probability convergence of
the \textit{average iterate} under sub-Gaussian noise. We will show
that their idea can be further used to prove the high-probability
convergence for the \textit{last iterate}.

Additional works about heavy-tailed noise and sub-Weibull noise will
be provided in Sections \ref{sec:main-heavy} and \ref{sec:main-weibull}.

%% file: preliminaries.tex
\section{Preliminaries\label{sec:Preliminaries}}

\textbf{Notations:} $\N$ is the set of natural numbers (excluding
$0$). $\left[d\right]\coloneqq\left\{ 1,2,\cdots,d\right\} $ for
any $d\in\N$. $a\lor b$ and $a\land b$ are defined as $\max\left\{ a,b\right\} $
and $\min\left\{ a,b\right\} $, respectively. $\left\langle \cdot,\cdot\right\rangle $
is the standard Euclidean inner product on $\R^{d}$. $\left\Vert \cdot\right\Vert $
represents a general norm on $\R^{d}$ and $\left\Vert \cdot\right\Vert _{*}$
is its dual norm. Given a set $A\subseteq\R^{d}$, $\mathrm{int}(A)$
stands for its interior points. For a function $f$, $\partial f(x)$
denotes the set of subgradients at $x$.

We focus on the following optimization problem in this work\textbf{
\[
\min_{x\in\dom}F(x)\coloneqq f(x)+h(x),
\]
}where $f$ and $h$ are both convex. $\dom\subseteq\mathrm{int}(\mathrm{dom}(f))\subseteq\R^{d}$
is a closed convex set. The requirement of $\dom\subseteq\mathrm{int}(\mathrm{dom}(f))$
is only to guarantee the existence of $\partial f(x)$ for every point
$x$ in $\dom$ with no special reason. We emphasize that there is
no compactness requirement on $\dom$. Additionally, given $\psi$
being a differentiable and $1$-strongly convex function with respect
to $\left\Vert \cdot\right\Vert $ on $\dom$ (i.e., $\psi(x)\geq\psi(y)+\left\langle \na\psi(y),x-y\right\rangle +\frac{1}{2}\left\Vert x-y\right\Vert ^{2},\forall x,y\in\dom$\footnote{Rigorously speaking, $y$ should be in $\mathrm{int}(\dom)$. But
one can think $\dom\subseteq\mathrm{int}(\mathrm{dom}(\psi))$ to
avoid this issue.}), the Bregman divergence with respect to $\psi$ is defined as $D_{\psi}(x,y)\coloneqq\psi(x)-\psi(y)-\left\langle \na\psi(y),x-y\right\rangle $.
Throughout this paper, we assume that $\argmin_{x\in\dom}h(x)+\left\langle g,x-y\right\rangle +\frac{D_{\psi}(x,y)}{\eta}$
can be solved efficiently for any $g\in\R^{d}$, $y\in\dom$, $\eta>0$.

Next, we list the assumptions used in our analysis:
\begin{enumerate}
\item[\textbf{1.}] \setref{1}\label{enu:A1}\textbf{Existence of a local minimizer:}
$\exists x^{*}\in\argmin_{x\in\dom}F(x)$ satisfying $F(x^{*})>-\infty$.\footnote{In fact, this assumption is not necessary for most of our results
(except theorems and lemmas in Sections \ref{sec:main-weibull} and
\ref{sec:weibull-analysis}) since we can bound $F(x^{T+1})-F(x)$
for any $x\in\dom$. However, we keep it here and will use it when
stating theorems in the main text. When it is possible, we will drop
this assumption in the intermediate lemmas and theorems in the appendix.}
\item[\textbf{2.}] \setref{2}\label{enu:A2}$(\mu_{f},\mu_{h})$\textbf{-strongly convex:}
For $k=f$ and $k=h$, $\exists\mu_{k}\geq0$ such that $\mu_{k}D_{\psi}(x,y)\leq k(x)-k(y)-\left\langle g,x-y\right\rangle ,\forall x,y\in\dom,g\in\pa k(y)$.
Moreover, we assume at least one of $(\mu_{f},\mu_{h})$ is zero.
\item[\textbf{3.}] \setref{3}\label{enu:A3}\textbf{General} $(L,M)$\textbf{-smooth:}
$\exists L\geq0,M\geq0$ such that $f(x)-f(y)-\left\langle g,x-y\right\rangle \leq\frac{L}{2}\left\Vert x-y\right\Vert ^{2}+M\left\Vert x-y\right\Vert ,\forall x,y\in\dom,g\in\pa f(y)$.
\item[\textbf{4.}] \setref{4}\label{enu:A4}\textbf{Unbiased gradient estimator:} For
a given $x^{t}\in\dom$ in the $t$-th iteration, we can access an
unbiased gradient estimator $\hg^{t}$, i.e., $\E\left[\hg^{t}\mid\F^{t-1}\right]\in\pa f(x^{t})$,
where $\F^{t}\coloneqq\sigma\left(\hg^{s},s\in\left[t\right]\right)$
is the natural filtration.
\item[\textbf{5A.}] \setref{5A}\label{enu:A5A}\textbf{Finite variance:} $\exists\sigma\geq0$
denoting the noise level such that $\E\left[\left\Vert \xi^{t}\right\Vert _{*}^{2}\right]\leq\sigma^{2}$
where $\xi^{t}\coloneqq\hg^{t}-\E\left[\hg^{t}\mid\F^{t-1}\right]$.
\item[\textbf{5B.}] \setref{5B}\label{enu:A5B}\textbf{Sub-Gaussian noise:} $\exists\sigma\geq0$
denoting the noise level such that $\E\left[\exp\left(\lambda\left\Vert \xi^{t}\right\Vert _{*}^{2}\right)\mid\F^{t-1}\right]\leq\exp\left(\lambda\sigma^{2}\right),\forall\lambda\in\left[0,\sigma^{-2}\right]$.
\end{enumerate}
We briefly discuss the assumptions here. Assumptions \ref{enu:A1},
\ref{enu:A4}, and \ref{enu:A5A} are standard in the stochastic optimization
literature. Assumption \ref{enu:A2} is known as relative strong convexity
appeared in previous works \citep{JMLR:v15:hazan14a,doi:10.1137/16M1099546}.
We use it here since the last-iterate convergence rate will be derived
for the CSMD algorithm, which employs Bregman divergence to exploit
the non-Euclidean geometry. In particular, when $\left\Vert \cdot\right\Vert $
is the standard $\ell_{2}$ norm, we can take $\psi(x)=\frac{1}{2}\left\Vert x\right\Vert ^{2}$
to recover the common definition of strong convexity. Assumption \ref{enu:A3}
is borrowed from Section 4.2 in \citet{lan2020first}. Note that both
$L$-smooth functions (by taking $M=0$) and $G$-Lipschitz functions
(by taking $L=0$ and $M=2G$) are subclasses of Assumption \ref{enu:A3}.
Lastly, Assumption \ref{enu:A5B} is used for the high-probability
convergence bound.

Our proofs for the high-probability convergence rely on the following
simple fact for the centered sub-Gaussian random vector. Similar results
have been proved in prior works \citep{vershynin2018high,pmlr-v202-liu23aa}.
For completeness, we include the proof in Appendix \ref{sec:tech}.
\begin{lem}
\label{lem:gaussian}Given a sigma algebra $\F$ and a random vector
$Z\in\R^{d}$ that is $\F$-measurable, if $\xi\in\R^{d}$ is a random
vector satisfying $\E\left[\xi\mid\F\right]=0$ and $\E\left[\exp\left(\lambda\left\Vert \xi\right\Vert _{*}^{2}\right)\mid\F\right]\leq\exp\left(\lambda\sigma^{2}\right),\forall\lambda\in\left[0,\sigma^{-2}\right]$,
then
\[
\E\left[\exp\left(\left\langle \xi,Z\right\rangle \right)\mid\F\right]\leq\exp\left(\sigma^{2}\left\Vert Z\right\Vert ^{2}\right).
\]
\end{lem}

\subsection{Convergence Criterion\label{subsec:criterion}}

We always measure the convergence via the function value gap, i.e.,
$F(x)-F(x^{*})$. There are several reasons to stick to this criterion.
First, for the general convex case, the function value gap is the
standard metric. Next, for strongly convex functions, the function
value gap is always a stronger measurement than the squared distance
to the optimal solution since $\left\Vert x-x^{*}\right\Vert ^{2}=O\left(F(x)-F(x^{*})\right)$
holds by strong convexity. Even if $F(x)$ is additionally assumed
to be ($L,0$)-smooth (e.g., $f(x)$ is ($L,0$)-smooth and $h(x)=0$),
the bound on $\left\Vert x-x^{*}\right\Vert ^{2}$ cannot be converted
to the bound on $F(x)-F(x^{*})$ since $F(x)-F(x^{*})\leq\left\langle \na F(x^{*}),x-x^{*}\right\rangle +\frac{L}{2}\left\Vert x-x^{*}\right\Vert ^{2}=O(\left\Vert \na F(x^{*})\right\Vert _{*}\left\Vert x-x^{*}\right\Vert +\left\Vert x-x^{*}\right\Vert ^{2})$,
which is probably worse than $O(\left\Vert x-x^{*}\right\Vert ^{2})$
as $x^{*}$ is only a local minimizer meaning $\left\Vert \na F(x^{*})\right\Vert _{*}$
possibly to be non-zero. Moreover, the function value gap is important
in both the theoretical and practical sides of modern machine learning
(e.g., the generalization error).

%% file: algorithm.tex
\section{Last-Iterate Convergence of Stochastic Gradient Methods\label{sec:main-algo}}

\begin{algorithm}[h]
\caption{\label{alg:CSMD}Composite Stochastic Mirror Descent (CSMD)}

\textbf{Input:} $x^{1}\in\dom$, $\eta_{t\in\left[T\right]}>0$.

\textbf{for} $t=1$ \textbf{to} $T$ \textbf{do}

$\quad$$x^{t+1}=\argmin_{x\in\dom}h(x)+\left\langle \hg^{t},x-x^{t}\right\rangle +\frac{D_{\psi}(x,x^{t})}{\eta_{t}}$

\textbf{Return $x^{T+1}$}
\end{algorithm}

The algorithm, Composite Stochastic Mirror Descent, is presented in
Algorithm \ref{alg:CSMD}. When $h(x)=0$, Algorithm \ref{alg:CSMD}
degenerates to the standard Stochastic Mirror Descent algorithm. If
we further consider the case $\left\Vert \cdot\right\Vert =\left\Vert \cdot\right\Vert _{2}$,
Algorithm \ref{alg:CSMD} can recover the standard projected SGD by
taking $\psi(x)=\frac{1}{2}\left\Vert x\right\Vert _{2}^{2}$. We
assume $T\geq2$ throughout the following paper to avoid some algebraic
issues in the proof. The full version of every following theorem with
its proof is deferred to the appendix.

\subsection{General Convex Functions\label{subsec:main-cvx}}

In this section, we focus on the last-iterate convergence of Algorithm
\ref{alg:CSMD} for general convex functions (i.e., $\mu_{f}=\mu_{h}=0$).
First, the in-expectation convergence rates are shown in Theorem \ref{thm:main-cvx-exp}.
\begin{thm}
\label{thm:main-cvx-exp}Under Assumptions \ref{enu:A1}-\ref{enu:A4}
and \ref{enu:A5A} with $\mu_{f}=\mu_{h}=0$:

If $T$ is unknown, by taking $\eta_{t\in\left[T\right]}=\frac{1}{2L}\land\frac{\eta}{\sqrt{t}}$
with $\eta=\Theta\left(\sqrt{\frac{D_{\psi}(x^{*},x^{1})}{M^{2}+\sigma^{2}}}\right)$,
there is
\[
\E\left[F(x^{T+1})-F(x^{*})\right]\leq O\left(\frac{LD_{\psi}(x^{*},x^{1})}{T}+\frac{(M+\sigma)\sqrt{D_{\psi}(x^{*},x^{1})}\log T}{\sqrt{T}}\right).
\]

If $T$ is known, by taking $\eta_{t\in\left[T\right]}=\frac{1}{2L}\land\frac{\eta}{\sqrt{T}}$
with $\eta=\Theta\left(\sqrt{\frac{D_{\psi}(x^{*},x^{1})}{(M^{2}+\sigma^{2})\log T}}\right)$,
there is
\[
\E\left[F(x^{T+1})-F(x^{*})\right]\leq O\left(\frac{LD_{\psi}(x^{*},x^{1})}{T}+\frac{(M+\sigma)\sqrt{D_{\psi}(x^{*},x^{1})\log T}}{\sqrt{T}}\right).
\]
\end{thm}

Before moving on to the high-probability bounds, we would like to
talk more about these in-expectation convergence results. First, the
constant $\eta$ here is optimized to obtain the best dependence on
the parameters $M,\sigma$ and $D_{\psi}(x^{*},x^{1})$. Indeed, the
last iterate provably converges for arbitrary $\eta>0$, but with
a worse dependence on $M,\sigma$ and $D_{\psi}(x^{*},x^{1})$. We
refer the reader to Theorem \ref{thm:cvx-exp} in the appendix for
a full version of Theorem \ref{thm:main-cvx-exp} with any $\eta>0$. 

Next, by taking $L=0$, we immediately get the (nearly) optimal $\widetilde{O}(1/\sqrt{T})$
convergence rate of the last iterate for non-smooth functions. Note
that our bounds are better than \citet{pmlr-v28-shamir13} since it
only works for bounded domains and non-composite optimization. Besides,
when considering smooth problems (taking $M=0$), to our best knowledge,
our $\widetilde{O}(L/T+\sigma/\sqrt{T})$ bound is the first improvement
since the $O(1/\sqrt[3]{T})$ rate by \citet{NIPS2011_40008b9a}.
Moreover, compared to \citet{NIPS2011_40008b9a}, Theorem \ref{thm:main-cvx-exp}
does not rely on some restrictive assumptions like bounded stochastic
gradients or $x^{*}$ being a global optimal point but is able to
be used for the more general composite problems. Additionally, it
is worth remarking that the $\widetilde{O}(L/T+\sigma/\sqrt{T})$
rate matches the optimal $O(L/T+\sigma/\sqrt{T})$ rate for the averaged
output $x_{avg}^{T+1}=(\sum_{t=2}^{T+1}x^{t})/T$ \citep{lan2020first}
up to an extra logarithmic factor. Notably, our bounds are also adaptive
to the noise $\sigma$ in this case. In other words, we can recover
the well-known $O(L/T)$ rate for the last iterate of the GD algorithm
in the noiseless case. Last but most importantly, our proof is unified
and thus can be applied to different settings (e.g., general domains,
($L,M$)-smoothness, non-Euclidean norms, etc.) simultaneously.
\begin{rem}
\citet{orabona2020blog} exhibited a circuitous method based on comparing
the last iterate with the averaged output to show the in-expectation
last-iterate convergence for non-composite non-smooth convex optimization
in general domains. However, it did not explicitly generalize to the
broader problems considered in this paper. Moreover, our method is
done in a direct manner (see Section \ref{sec:analysis}).
\end{rem}

\begin{thm}
\label{thm:main-cvx-hp}Under Assumptions \ref{enu:A1}-\ref{enu:A4}
and \ref{enu:A5B} with $\mu_{f}=\mu_{h}=0$ and let $\delta\in(0,1)$:

If $T$ is unknown, by taking $\eta_{t\in\left[T\right]}=\frac{1}{2L}\land\frac{\eta}{\sqrt{t}}$
with $\eta=\Theta\left(\sqrt{\frac{D_{\psi}(x^{*},x^{1})}{M^{2}+\sigma^{2}\log\frac{1}{\delta}}}\right)$,
then with probability at least $1-\delta$, there is
\[
F(x^{T+1})-F(x^{*})\leq O\left(\frac{LD_{\psi}(x^{*},x^{1})}{T}+\frac{(M+\sigma\sqrt{\log\frac{1}{\delta}})\sqrt{D_{\psi}(x^{*},x^{1})}\log T}{\sqrt{T}}\right).
\]

If $T$ is known, by taking $\eta_{t\in\left[T\right]}=\frac{1}{2L}\land\frac{\eta}{\sqrt{T}}$
with $\eta=\Theta\left(\sqrt{\frac{D_{\psi}(x^{*},x^{1})}{(M^{2}+\sigma^{2}\log\frac{1}{\delta})\log T}}\right)$,
then with probability at least $1-\delta$, there is
\[
F(x^{T+1})-F(x^{*})\leq O\left(\frac{LD_{\psi}(x^{*},x^{1})}{T}+\frac{(M+\sigma\sqrt{\log\frac{1}{\delta}})\sqrt{D_{\psi}(x^{*},x^{1})\log T}}{\sqrt{T}}\right).
\]
\end{thm}

In Theorem \ref{thm:main-cvx-hp}, we present the high-probability
bounds for $(L,M)$-smooth functions. Again, the constant $\eta$
is picked to get the best dependence on the parameters $M,\sigma,D_{\psi}(x^{*},x^{1})$
and $\log(1/\delta)$. The full version of Theorem \ref{thm:main-cvx-hp}
with arbitrary $\eta$, Theorem \ref{thm:cvx-hp}, is deferred to
the appendix. Compared with Theorem \ref{thm:main-cvx-exp}, the high-probability
rates only incur an extra $O(\sqrt{\log(1/\delta)})$ factor (or $O(\log(1/\delta))$
for arbitrary $\eta$, which is known to be optimal for $L=0$ \citep{pmlr-v99-harvey19a}). 

In contrast to the previous bounds \citep{pmlr-v99-harvey19a,doi:10.1137/19M128908X}
that only work for Lipschitz functions in a compact domain, our results
are the first to describe the high-probability behavior of Algorithm
\ref{alg:CSMD} for the wider $(L,M)$-smooth function class in a
general domain even with sub-Gaussian noise, not to mention composite
objectives and non-Euclidean norms. Even in the special smooth case
(setting $M=0$), as far as we know, this is also the first last-iterate
high-probability bound being adaptive to the noise $\sigma$ at the
same time for plain stochastic gradient methods. Unlike the previous
proofs employing some new probability tools (e.g., the generalized
Freedman's inequality in \citet{pmlr-v99-harvey19a}), our high-probability
argument is simple and only based on the basic property of sub-Gaussian
random vectors (see Lemma \ref{lem:gaussian}). Therefore, we believe
our work can bring some new insights to researchers to gain a better
understanding of the last-iterate convergence of stochastic gradient
methods.

\subsubsection{Optimal Rates via the Linearly Decaying Step Size\label{subsec:main-cvx-optimal}}

In this section, we show that the $O(\sqrt{\log T})$ factor appearing
in Theorems \ref{thm:main-cvx-exp} and \ref{thm:main-cvx-hp} for
known $T$ can be completely removed. Even restricted to the special
case of Lipschitz objectives (i.e., $L=0$ in Assumption \ref{enu:A3}),
this is a strict improvement compared to \citet{doi:10.1137/19M128908X},
in which the $O(1/\sqrt{T})$ rate is obtained under the bounded noise
assumption on a compact domain $\dom$.
\begin{thm}
\label{thm:main-cvx-exp-optimal}Under Assumptions \ref{enu:A1}-\ref{enu:A4}
and \ref{enu:A5A} with $\mu_{f}=\mu_{h}=0$, if $T$ is known, by
taking $\eta_{t\in\left[T\right]}=\frac{T-t+1}{2LT}\land\frac{\eta(T-t+1)}{T^{\frac{3}{2}}}$
with $\eta=\Theta\left(\sqrt{\frac{D_{\psi}(x^{*},x^{1})}{M^{2}+\sigma^{2}}}\right)$,
there is
\[
\E\left[F(x^{T+1})-F(x^{*})\right]\leq O\left(\frac{LD_{\psi}(x^{*},x^{1})}{T}+\frac{(M+\sigma)\sqrt{D_{\psi}(x^{*},x^{1})}}{\sqrt{T}}\right).
\]
\end{thm}

\begin{thm}
\label{thm:main-cvx-hp-optimal}Under Assumptions \ref{enu:A1}-\ref{enu:A4}
and \ref{enu:A5B} with $\mu_{f}=\mu_{h}=0$ and let $\delta\in(0,1)$,
if $T$ is known, by taking $\eta_{t\in\left[T\right]}=\frac{T-t+1}{2LT}\land\frac{\eta(T-t+1)}{T^{\frac{3}{2}}}$
with $\eta=\Theta\left(\sqrt{\frac{D_{\psi}(x^{*},x^{1})}{M^{2}+\sigma^{2}\log\frac{1}{\delta}}}\right)$,
then with probability at least $1-\delta$, there is
\[
F(x^{T+1})-F(x^{*})\leq O\left(\frac{LD_{\psi}(x^{*},x^{1})}{T}+\frac{(M+\sigma\sqrt{\log\frac{1}{\delta}})\sqrt{D_{\psi}(x^{*},x^{1})}}{\sqrt{T}}\right).
\]
\end{thm}

The linearly decaying step size was introduced by \citet{doi:10.1137/24M1717762},
in which the authors proved that it ensures the last iterate of the
subgradient method converges at the rate $O(1/\sqrt{T})$ for non-composite
non-smooth convex objectives with the Euclidean norm. Here, we show
it can also guarantee the optimal $O(L/T+(M+\sigma)/\sqrt{T})$ rate
in our setting. Again, the constant $\eta$ can be set to any positive
number. The full statements of Theorems \ref{thm:main-cvx-exp-optimal}
and \ref{thm:main-cvx-hp-optimal} are provided in Appendix \ref{subsec:cvx-optimal}.

\subsection{Strongly Convex Functions\label{subsec:main-str}}

Now we turn our attention to strongly convex functions. Due to the
space limitation, we only provide the results for the case of $\mu_{f}>0$
and $\mu_{h}=0$. The other case, $\mu_{f}=0$ and $\mu_{h}>0$, will
be delivered in Appendix \ref{subsec:str-h}.
\begin{thm}
\label{thm:main-str-exp}Under Assumptions \ref{enu:A1}-\ref{enu:A4}
and \ref{enu:A5A} with $\mu_{f}>0$ and $\mu_{h}=0$, let $\kappa_{f}\coloneqq\frac{L}{\mu_{f}}\geq0$:

If $T$ is unknown, by taking either $\eta_{t\in\left[T\right]}=\frac{1}{\mu_{f}(t+2\kappa_{f})}$
or $\eta_{t\in\left[T\right]}=\frac{2}{\mu_{f}(t+1+4\kappa_{f})}$,
there is
\[
\E\left[F(x^{T+1})-F(x^{*})\right]\leq\begin{cases}
O\left(\frac{LD_{\psi}(x^{*},x^{1})}{T}+\frac{(M^{2}+\sigma^{2})\log T}{\mu_{f}(T+\kappa_{f})}\right) & \eta_{t\in\left[T\right]}=\frac{1}{\mu_{f}(t+2\kappa_{f})}\\
O\left(\frac{L(1+\kappa_{f})D_{\psi}(x^{*},x^{1})}{T(T+\kappa_{f})}+\frac{(M^{2}+\sigma^{2})\log T}{\mu_{f}(T+\kappa_{f})}\right) & \eta_{t\in\left[T\right]}=\frac{2}{\mu_{f}(t+1+4\kappa_{f})}
\end{cases}.
\]

If $T$ is known, by taking $\eta_{t\in\left[T\right]}=\begin{cases}
\frac{1}{\mu_{f}(1+2\kappa_{f})} & t\leq\tau\\
\frac{2}{\mu_{f}(t-\tau+2+4\kappa_{f})} & t\geq\tau+1
\end{cases}$ with $\tau\coloneqq\left\lceil \frac{T}{2}\right\rceil $, there
is
\[
\E\left[F(x^{T+1})-F(x^{*})\right]\leq O\left(\frac{LD_{\psi}(x^{*},x^{1})}{\exp\left(\frac{T}{2+4\kappa_{f}}\right)}+\frac{(M^{2}+\sigma^{2})\log T}{\mu_{f}(T+\kappa_{f})}\right).
\]
\end{thm}

The in-expectation rates are stated in Theorem \ref{thm:main-str-exp}.
We would like to remind the reader that $\kappa_{f}\geq1$ is not
necessary, as we are considering the general $(L,M)$-smooth functions.
Hence, it can be zero.

As before, we first take $L=0$ to consider the special Lipschitz
case. Due to $\kappa_{f}=0$ now, all bounds will degenerate to $O(\log T/T)$,
which is known to be optimal for the step size $\nicefrac{1}{\mu_{f}t}$
\citep{pmlr-v99-harvey19a} and only incurs an extra $O(\log T)$
factor compared with the best $O(1/T)$ bound when $T$ is known \citep{doi:10.1137/19M128908X}.
We would also like to mention that Theorem \ref{thm:main-str-exp}
is the first to give the in-expectation last-iterate bound for the
step size $\nicefrac{2}{\mu_{f}(t+1)}$. Interestingly, the extra
$O(\log T)$ factor appears again compared to the known $O(1/T)$
bound on the function value gap for the non-uniform averaging strategy
under this step size \citep{lacoste2012simpler}. Besides, \citet{lacoste2012simpler}
also shows $\E\left[\left\Vert x^{T+1}-x^{*}\right\Vert _{2}^{2}\right]=O(1/T)$.
Whereas it is currently unknown whether our $\E\left[F(x^{T+1})-F(x^{*})\right]=O(\log T/T)$
bound can be improved to match the $O(1/T)$ rate under the step size
$\nicefrac{2}{\mu_{f}(t+1)}$.

For the general $(L,M)$-smooth case (even for $(L,0)$-smoothness),
our bounds are the first convergence results for the last iterate
of stochastic gradient methods with respect to the function value
gap\footnote{Note that the rates under the PL-condition (e.g., \citet{pmlr-v130-gower21a,khaled2022better})
are incompatible with our settings since they can be only applied
to non-constrained, non-composite and $(L,0)$-smooth optimization
problems with the Euclidean norm.}. Remarkably, all of these rates do not require prior knowledge of
$M$ or $\sigma$ to set the step size. In particular, the bound for
known $T$ is adaptive to $\sigma$ when $M=0$, i.e., it can recover
the well-known linear convergence rate $O(\exp(-T/\kappa_{f}))$ when
$\sigma=0$.
\begin{thm}
\label{thm:main-str-hp}Under Assumptions \ref{enu:A1}-\ref{enu:A4}
and \ref{enu:A5B} with $\mu_{f}>0$ and $\mu_{h}=0$, let $\kappa_{f}\coloneqq\frac{L}{\mu_{f}}\geq0$
and $\delta\in(0,1)$:

If $T$ is unknown, by taking either $\eta_{t\in\left[T\right]}=\frac{1}{\mu_{f}(t+2\kappa_{f})}$
or $\eta_{t\in\left[T\right]}=\frac{2}{\mu_{f}(t+1+4\kappa_{f})}$,
then with probability at least $1-\delta$, there is
\[
F(x^{T+1})-F(x^{*})\leq\begin{cases}
O\left(\frac{\mu_{f}(1+\kappa_{f})D_{\psi}(x^{*},x^{1})}{T}+\frac{(M^{2}+\sigma^{2}\log\frac{1}{\delta})\log T}{\mu_{f}(T+\kappa_{f})}\right) & \eta_{t\in\left[T\right]}=\frac{1}{\mu_{f}(t+2\kappa_{f})}\\
O\left(\frac{\mu_{f}(1+\kappa_{f})^{2}D_{\psi}(x^{*},x^{1})}{T(T+\kappa_{f})}+\frac{(M^{2}+\sigma^{2}\log\frac{1}{\delta})\log T}{\mu_{f}(T+\kappa_{f})}\right) & \eta_{t\in\left[T\right]}=\frac{2}{\mu_{f}(t+1+4\kappa_{f})}
\end{cases}.
\]

If $T$ is known, by taking $\eta_{t\in\left[T\right]}=\begin{cases}
\frac{1}{\mu_{f}(1+2\kappa_{f})} & t=1\\
\frac{1}{\mu_{f}(\eta+2\kappa_{f})} & 2\leq t\leq\tau\\
\frac{2}{\mu_{f}(t-\tau+2+4\kappa_{f})} & t\geq\tau+1
\end{cases}$ with $\eta\coloneqq1.5$ and $\tau\coloneqq\left\lceil \frac{T}{2}\right\rceil $,
then with probability at least $1-\delta$, there is
\[
F(x^{T+1})-F(x^{*})\leq O\left(\frac{\mu_{f}(1+\kappa_{f})D_{\psi}(x^{*},x^{1})}{\exp\left(\frac{T}{3+4\kappa_{f}}\right)}+\frac{(M^{2}+\sigma^{2}\log\frac{1}{\delta})\log T}{\mu_{f}(T+\kappa_{f})}\right).
\]
\end{thm}

Now, we provide the high-probability convergence results in Theorem
\ref{thm:main-str-hp}. The constant $\eta=1.5$ is set without any
particular reason. The full statement with general $\eta$, Theorem
\ref{thm:str-f-hp}, can be found in the appendix. Besides, $\kappa_{f}$
is possible to be zero as mentioned above. Compared with Theorem \ref{thm:main-str-exp},
only an additional $O(\log(1/\delta))$ factor appears. Such an extra
loss is known to be inevitable for $L=0$ due to \citet{pmlr-v99-harvey19a}.

For the Lipschitz case (i.e., $L=\kappa_{f}=0$), by noticing $D_{\psi}(x^{*},x^{1})=O(M^{2}/\mu_{f}^{2})$\footnote{This holds now due to $\mu_{f}\left\Vert x^{*}-x^{1}\right\Vert ^{2}/2\leq\mu_{f}D_{\psi}(x^{*},x^{1})\leq M\left\Vert x^{*}-x^{1}\right\Vert $.},
all of these bounds will degenerate to $O(\log(1/\delta)\log T/T)$
matching the best-known last-iterate bound proved by \citet{pmlr-v99-harvey19a}
for the step size $\nicefrac{1}{\mu_{f}t}$. For the step size $\nicefrac{2}{\mu_{f}(t+1)}$,
\citet{harvey2019simple} has proved the high-probability bound $O(\log(1/\delta)/T)$
for the non-uniform averaging output instead of the last iterate.
Hence, as far as we know, our high-probability rate for the step size
$\nicefrac{2}{\mu_{f}(t+1)}$ is new.

Finally, let us go back to the general $(L,M)$-smooth case. To our
best knowledge, our results are the first to prove that the last iterate
of plain stochastic gradient methods enjoys the provable high-probability
convergence (even in the smooth case, i.e., $M=0$). Hence, we believe
our work closes the gap between the lack of theoretical understanding
and good performance of the last iterate of SGD for smooth and strongly
convex functions. Lastly, the same as the in-expectation bound for
known $T$ in Theorem \ref{thm:main-str-exp}, our high-probability
bound is also adaptive to $\sigma$ when $M=0$.

\subsubsection{Optimal Rates via a New Step Size\label{subsec:main-str-optimal}}

In this section, we show that the $O(\log T)$ factor appearing in
Theorems \ref{thm:main-str-exp} and \ref{thm:main-str-hp} for known
$T$ can be completely removed. Even restricted to the special case
of Lipschitz objectives (i.e., $L=0$ in Assumption \ref{enu:A3}),
this is a strict improvement compared to \citet{doi:10.1137/19M128908X},
in which the $O(1/T)$ rate is obtained under the bounded noise assumption
on a compact domain $\dom$ (though $\dom$ has to be bounded now,
our proof does not rely on this fact). Moreover, to avoid some algebraic
issues in the proof, we assume $T\geq4$ in this part.
\begin{thm}
\label{thm:main-str-exp-optimal}Under Assumptions \ref{enu:A1}-\ref{enu:A4}
and \ref{enu:A5A} with $\mu_{f}>0$ and $\mu_{h}=0$, let $\kappa_{f}\coloneqq\frac{L}{\mu_{f}}\geq0$,
if $T$ is known, by taking $\eta_{t\in\left[T\right]}=\begin{cases}
\frac{1}{\mu_{f}(1+2\kappa_{f})} & t\leq\tau_{1}\\
\frac{2}{\mu_{f}(t-\tau_{1}+2+4\kappa_{f})} & \tau_{1}+1\leq t\leq\tau_{2}\\
\frac{T-t+1}{2L(T-\tau_{2})}\land\frac{T-t+1}{\mu_{f}(T-\tau_{2})(T+\kappa_{f})} & t\geq\tau_{2}+1
\end{cases}$ with $\tau_{1}\coloneqq\left\lceil \frac{T}{4}\right\rceil $ and
$\tau_{2}\coloneqq\left\lceil \frac{T}{2}\right\rceil $, there is
\[
\E\left[F(x^{T+1})-F(x^{*})\right]\leq O\left(\frac{LD_{\psi}(x^{*},x^{1})}{\exp\left(\frac{T}{4+8\kappa_{f}}\right)}+\frac{M^{2}+\sigma^{2}}{\mu_{f}(T+\kappa_{f})}\right).
\]
\end{thm}

\begin{thm}
\label{thm:main-str-hp-optimal}Under Assumptions \ref{enu:A1}-\ref{enu:A4}
and \ref{enu:A5B} with $\mu_{f}>0$ and $\mu_{h}=0$, let $\kappa_{f}\coloneqq\frac{L}{\mu_{f}}\geq0$
and $\delta\in(0,1)$, if $T$ is known, by taking $\eta_{t\in\left[T\right]}=\begin{cases}
\frac{1}{\mu_{f}(\eta+2\kappa_{f})} & t\leq\tau_{1}\\
\frac{2}{\mu_{f}(t-\tau_{1}+2+4\kappa_{f})} & \tau_{1}+1\leq t\leq\tau_{2}\\
\frac{T-t+1}{2L(T-\tau_{2})}\land\frac{T-t+1}{\mu_{f}(T-\tau_{2})(T+2\kappa_{f})} & t\geq\tau_{2}+1
\end{cases}$ with $\eta\coloneqq1.5$, $\tau_{1}\coloneqq\left\lceil \frac{T}{4}\right\rceil $
and $\tau_{2}\coloneqq\left\lceil \frac{T}{2}\right\rceil $, then
with probability at least $1-\delta$, there is
\[
F(x^{T+1})-F(x^{*})\leq O\left(\frac{\mu_{f}(1+\kappa_{f})D_{\psi}(x^{*},x^{1})}{\exp\left(\frac{T}{6+8\kappa_{f}}\right)}+\frac{M^{2}+\sigma^{2}\log\frac{1}{\delta}}{\mu_{f}(T+\kappa_{f})}\right).
\]
\end{thm}

As far as we know, Theorems \ref{thm:main-str-exp-optimal} and \ref{thm:main-str-hp-optimal}
are the first results to obtain the optimal $O(1/T)$ rate for the
last iterate without any restrictive assumptions. We highlight that
the step size developed here differs significantly from the schedule
introduced by \citet{doi:10.1137/19M128908X}. Concretely, our new
step size is a combination of the classical one (as used in Theorems
\ref{thm:main-str-exp} and \ref{thm:main-str-hp}) and the linearly
decaying step size proposed in \citet{doi:10.1137/24M1717762}. Moreover,
we note that the constant $\eta=1.5$ is picked without any particular
reason. The full statements of Theorems \ref{thm:main-str-exp-optimal}
and \ref{thm:main-str-hp-optimal} are provided in Appendix \ref{subsec:str-f-optimal}.

\section{Unified Theoretical Analysis\label{sec:analysis}}

In this section, we introduce the ideas in our analysis and present
three important lemmas, all the missing proofs of which are deferred
to Appendix \ref{sec:proofs}.

The key insight in our proofs is to utilize the convexity of $F(x)$,
which is highly inspired by the recent work of \citet{doi:10.1137/24M1717762}.
To be more precise, using the classic convergence analysis for non-composite
Lipschitz convex problems as an example, people always consider to
upper bound the function value gap $f(x^{t})-f(x^{*})$ (probably
with some weight before it) then sum them over time to obtain the
ergodic rate. Whereas, in such an argument, convexity is not necessary
in fact (except if one wants to bound the average iterate in the last
step). Hence, if the convexity of $f$ can be utilized somewhere,
it is reasonable to expect a last-iterate convergence guarantee. Actually,
this thought is possible as shown by \citet{doi:10.1137/24M1717762},
in which the authors upper bound the quantity $f(x^{t})-f(z^{t})$
where $z^{t}$ is a carefully chosen convex combination of other points
and finally obtain the last-iterate rate by lower bounding $-f(z^{t})$
via convexity. 

Though \citet{doi:10.1137/24M1717762} only shows how to prove the
last-iterate convergence for deterministic non-composite Lipschitz
convex optimization under the Euclidean norm, we can catch the most
important message conveyed by their paper and apply it to our settings.
Formally speaking, we will upper bound the term $F(x^{t+1})-F(z^{t})$
for a well-designed $z^{t}$ rather than directly bound the function
value gap $F(x^{t+1})-F(x^{*})$. This idea can finally help us construct
a unified proof and obtain several novel results without prior restrictive
assumptions. Through careful calculations, the new analysis leads
us to the following most important and unified result, Lemma \ref{lem:core-general}.
\begin{lem}
\label{lem:core-general}Under Assumptions \ref{enu:A2} and \ref{enu:A3},
suppose $\eta_{t\in\left[T\right]}\leq\frac{1}{2L\lor\mu_{f}}$ and
let $\gamma_{t\in\left[T\right]}\coloneqq\eta_{t}\prod_{s=2}^{t}\frac{1+\mu_{h}\eta_{s-1}}{1-\mu_{f}\eta_{s}}$,
if $w_{t\in\left[T\right]}\geq0$ is a non-increasing sequence and
$v_{t\in\left\{ 0\right\} \cup\left[T\right]}>0$ is defined as $v_{t}\coloneqq\frac{w_{T}\gamma_{T}}{\sum_{s=t\lor1}^{T}w_{s}\gamma_{s}}$,
then for any $x\in\dom$, we have
\begin{align*}
 & w_{T}\gamma_{T}v_{T}\left(F(x^{T+1})-F(x)\right)\\
\leq & w_{1}(1-\mu_{f}\eta_{1})v_{0}D_{\psi}(x,x^{1})+\sum_{t=1}^{T}2w_{t}\gamma_{t}\eta_{t}v_{t}(M^{2}+\left\Vert \xi^{t}\right\Vert _{*}^{2})\\
 & +\sum_{t=1}^{T}w_{t}\gamma_{t}v_{t-1}\left\langle \xi^{t},z^{t-1}-x^{t}\right\rangle +\sum_{t=2}^{T}(w_{t}-w_{t-1})\gamma_{t}(\eta_{t}^{-1}-\mu_{f})v_{t-1}D_{\psi}(z^{t-1},x^{t}),
\end{align*}
where $\xi^{t}\coloneqq\hg^{t}-\E\left[\hg^{t}\mid\F^{t-1}\right],\forall t\in\left[T\right]$
and $z^{t}\coloneqq\frac{v_{0}}{v_{t}}x+\sum_{s=1}^{t}\frac{v_{s}-v_{s-1}}{v_{t}}x^{s},\forall t\in\left\{ 0\right\} \cup\left[T\right]$.
\end{lem}

Let us discuss Lemma \ref{lem:core-general} more here. The requirement
of the step size $\eta_{t}$ having an upper bound $\nicefrac{1}{2L\lor\mu_{f}}$
is common in the optimization literature. $\gamma_{t}$ is used to
ensure we can telescope sum some terms. For the special case $\mu_{f}=\mu_{h}=0$,
it degenerates to $\eta_{t}$. $\xi^{t}$ naturally shows up as we
are considering stochastic optimization. The most important sequences
are $w_{t},v_{t}$ and $z^{t}$. As mentioned above, the appearance
of $z^{t}$ is to make sure to get the last-iterate convergence. For
how to find such a sequence, we refer the reader to our proofs in
Appendix \ref{sec:proofs} for details.

We would like to say more about the sequence $w_{t}$ before moving
on. Suppose we are in the deterministic case temporarily, i.e., $\xi^{t}=0$,
then a natural choice is to set $w_{t}=1,\forall t\in\left[T\right]$
to remove the last residual summation. It turns out this is the correct
choice even for the following in-expectation bound in Lemma \ref{lem:core-exp}.
So why do we still need this redundant $w_{t}$? The reason is that
setting $w_{t}$ to be one is not enough for the high-probability
analysis. More precisely, if we still choose $w_{t}=1,\forall t\in\left[T\right]$,
then there will be some extra positive terms after the concentration
argument in the R.H.S. of the inequality in Lemma \ref{lem:core-general}.
To deal with this issue, we borrow the idea recently developed by
\citet{pmlr-v202-liu23aa}, in which the authors employ an extra sequence
$w_{t}$ to give a clear proof for the high-probability bound for
stochastic gradient methods. We refer the reader to \citet{pmlr-v202-liu23aa}
for a detailed explanation of this technique.
\begin{lem}
\label{lem:core-exp}Under Assumptions \ref{enu:A2}-\ref{enu:A4}
and \ref{enu:A5A}, suppose $\eta_{t\in\left[T\right]}\leq\frac{1}{2L\lor\mu_{f}}$
and let $\gamma_{t\in\left[T\right]}\coloneqq\eta_{t}\prod_{s=2}^{t}\frac{1+\mu_{h}\eta_{s-1}}{1-\mu_{f}\eta_{s}}$,
then for any $x\in\dom$, we have
\[
\E\left[F(x^{T+1})-F(x)\right]\leq\frac{(1-\mu_{f}\eta_{1})D_{\psi}(x,x^{1})}{\sum_{t=1}^{T}\gamma_{t}}+2(M^{2}+\sigma^{2})\sum_{t=1}^{T}\frac{\gamma_{t}\eta_{t}}{\sum_{s=t}^{T}\gamma_{s}}.
\]
\end{lem}

Suppose Lemma \ref{lem:core-general} holds, Lemma \ref{lem:core-exp}
is immediately obtained by setting $w_{t}=1,\forall t\in\left[T\right]$
and using Assumptions \ref{enu:A4} and \ref{enu:A5A}. This unified
result for the in-expectation last-iterate convergence can be applied
to many different settings, such as composite optimization and non-Euclidean
norms, without any restrictive assumptions.
\begin{lem}
\label{lem:core-hp}Under Assumptions \ref{enu:A2}-\ref{enu:A4}
and \ref{enu:A5B}, suppose $\eta_{t\in\left[T\right]}\leq\frac{1}{2L\lor\mu_{f}}$
and let $\gamma_{t\in\left[T\right]}\coloneqq\eta_{t}\prod_{s=2}^{t}\frac{1+\mu_{h}\eta_{s-1}}{1-\mu_{f}\eta_{s}}$,
then for any $x\in\dom$ and $\delta\in(0,1)$, with probability at
least $1-\delta$, we have
\[
F(x^{T+1})-F(x)\leq2\left(1+\max_{2\leq t\leq T}\frac{1}{1-\mu_{f}\eta_{t}}\right)\left[\frac{D_{\psi}(x,x^{1})}{\sum_{t=1}^{T}\gamma_{t}}+\left(M^{2}+\sigma^{2}\left(1+2\log\frac{2}{\delta}\right)\right)\sum_{t=1}^{T}\frac{\gamma_{t}\eta_{t}}{\sum_{s=t}^{T}\gamma_{s}}\right].
\]
\end{lem}

To get Lemma \ref{lem:core-hp}, we need some extra effort to find
the correct $w_{t}$ and invoke a simple property of sub-Gaussian
random vectors (Lemma \ref{lem:gaussian}). The details can be found
in Appendix \ref{sec:proofs}. Compared with prior works, this unified
high-probability bound can be applied to various scenarios including
general domains and sub-Gaussian noise.

Equipped with Lemma \ref{lem:core-exp} and Lemma \ref{lem:core-hp},
we can prove all the theorems provided in Section \ref{sec:main-algo}
by plugging in different step sizes for different cases.

%% file: heavy.tex
\section{Last-Iterate Convergence under Heavy-Tailed Noise\label{sec:main-heavy}}

In Section \ref{sec:main-algo}, we have shown how to obtain the convergence
rate of the last iterate of CSMD in a unified manner. However, the
in-expectation bounds require the finite variance assumption, which
may be too optimistic in modern machine learning. As pointed out by
recent studies, e.g., \citet{mirek2011heavy,pmlr-v97-simsekli19a,NEURIPS2020_b05b57f6,NEURIPS2020_f3f27a32,pmlr-v130-gurbuzbalaban21a,pmlr-v139-camuto21a,pmlr-v139-hodgkinson21a},
noises in different tasks are observed to have heavy tails, meaning
that they only have a finite $p$-th moment where $p\in\left(1,2\right)$.

Naturally, one would ask whether the last iterate of stochastic gradient
methods can still converge under this new challenging assumption.
Surprisingly, the answer is still yes. In this section, we extend
the idea used before to obtain the in-expectation last-iterate convergence
rate of Algorithm \ref{alg:CSMD} under heavy-tailed noise (see Assumption
\ref{enu:A5C} below). As far as we know, this is a new result in
the related area.

\subsection{Additional Related Work}

To the best of our knowledge, for stochastic gradient methods without
momentum or averaging operations, there exists no work that proves
a non-asymptotic last-iterate convergence rate for the function value
gap. Hence, we only review prior works related to the average iterate.

\textbf{Average-iterate convergence under heavy-tailed noise: }\citet{NEURIPS2020_b05b57f6}
proves that SGD with clipping guarantees an $O(T^{\frac{2}{p}-2})$
rate in expectation for Lipschitz strongly convex objectives. Without
clipping and strong convexity, \citet{pmlr-v178-vural22a} establishes
the $O(T^{\frac{1}{p}-1})$ in-expectation convergence for the stochastic
mirror descent algorithm. For the smooth problem, \citet{pmlr-v202-sadiev23a}
shows that SGD with clipping also admits nearly optimal high-probability
convergence rates $\widetilde{O}(T^{\frac{1}{p}-1})$ and $\widetilde{O}(T^{\frac{2}{p}-2})$
for convex and strongly convex functions, respectively. For the non-smooth
problem, \citet{liu2023stochastic} is the first to prove that the
clipping stochastic gradient method ensures the optimal rates $O(T^{\frac{1}{p}-1})$
under convexity and $O(T^{\frac{2}{p}-2})$ under strong convexity
(for both in-expectation and high-probability convergence). \citet{NEURIPS2023_4c454d34}
later obtains the optimal high-probability rate $O(T^{\frac{1}{p}-1})$
for the smooth problem.

\textbf{Lower bounds under heavy-tailed noise:} For Lipschitz convex
functions, \citet{nemirovskij1983problem,5394945,pmlr-v178-vural22a}
have proved a lower bound of $\Omega(T^{\frac{1}{p}-1})$. For smooth
and strongly convex functions (in the sense of the Euclidean norm),
a lower bound of $\Omega(T^{\frac{2}{p}-2})$ is provided in \citet{NEURIPS2020_b05b57f6}.

\subsection{New Assumption and Some Discussions}

We formally introduce Assumption \ref{enu:A5C} to describe the noise
that has a heavy tail. Note that Assumption \ref{enu:A5C} will be
the same as Assumption \ref{enu:A5A} if we enforce $p=2$.
\begin{enumerate}
\item[\textbf{5C.}] \setref{5C}\label{enu:A5C}\textbf{Finite $p$-th moment:} $\exists p\in\left(1,2\right)$,
$\sigma\geq0$ denoting the noise level such that $\E\left[\left\Vert \xi^{t}\right\Vert _{*}^{p}\right]\leq\sigma^{p}$.
\end{enumerate}
In order to fit the new assumption, we also need to make a slight
change to the mirror map $\psi$. Instead of being $1$-strongly convex,
we now require $\psi$ to be $(1,\frac{p}{p-1})$ uniformly convex,
i.e.,
\begin{equation}
D_{\psi}(x,y)=\psi(x)-\psi(y)-\left\langle \na\psi(y),x-y\right\rangle \geq\frac{p-1}{p}\left\Vert x-y\right\Vert ^{\frac{p}{p-1}},\forall x,y\in\dom.\label{eq:uniformly-cvx}
\end{equation}
Note that when $p=2$, this condition reduces to the standard $1$-strong
convexity. The idea of using a uniformly convex mirror map was proposed
in \citet{pmlr-v178-vural22a} before.

Moreover, we will only consider general convex functions (i.e., $\mu_{f}=\mu_{h}=0$
in Assumption \ref{enu:A2}) in this section. The main reason is that
if considering the case $\mu_{f}>0$ in Assumption \ref{enu:A2},
we can immediately observe the following fact,
\[
\frac{\mu_{f}(p-1)}{p}\left\Vert x-y\right\Vert ^{\frac{p}{p-1}}\leq\mu_{f}D_{\psi}(x,y)\leq f(x)-f(y)-\left\langle g,x-y\right\rangle ,\forall x,y\in\dom,g\in\pa f(y),
\]
which implies $f$ itself should be $(\mu_{f},\frac{p}{p-1})$ uniformly
convex. However, this condition is less common and cannot even be
satisfied in the case of a strongly convex $f$ when considering the
Euclidean norm. The other point is that even if $f$ indeed satisfies
Assumption \ref{enu:A2} with $\mu_{f}>0$, then by Assumption \ref{enu:A3},
there is
\[
\frac{\mu_{f}(p-1)}{p}\left\Vert x-y\right\Vert ^{\frac{p}{p-1}}\leq\frac{L}{2}\left\Vert x-y\right\Vert ^{2}+M\left\Vert x-y\right\Vert ,\forall x,y\in\dom,
\]
which implies $\sup_{x,y\in\dom}\left\Vert x-y\right\Vert <\infty$
due to $\frac{p}{p-1}>2$. Note that, with a compact domain, it is
possible to obtain convergence bounds but in a non-unified way. However,
this is contrary to the purpose of this paper. Due to these two points,
we do not consider the case of $\mu_{f}>0$. We remark that $\mu_{h}>0$
will not lead to a result of a bounded domain. But it may also suffer
from the first issue (e.g., when $h(x)=\frac{1}{2}\left\Vert x\right\Vert _{2}^{2}$).
Therefore, the case of $\mu_{h}>0$ is also discarded.
\begin{rem}
We clarify that if $\mu_{f}>0$ or $\mu_{h}>0$ indeed holds (for
example, taking $h(x)=\psi(x)$ implies $\mu_{h}=1$), one can still
apply our analysis framework to obtain the last-iterate convergence
rate of the CSMD algorithm. This potential work is left to the interested
reader.
\end{rem}

Lastly, as mentioned before, we will only provide in-expectation convergence
rates for Algorithm \ref{alg:CSMD}. The key reason is that we do
not consider the clipping technique. As far as we know, all the high-probability
results for the averaged output are established with clipping (even
under the finite variance assumption). How to obtain the $O(\mathrm{polylog}\frac{1}{\delta})$
dependence for an algorithm without clipping in a single run remains
unknown. Moreover, there is some evidence showing that the dependence
of $\Omega(\mathrm{poly}\frac{1}{\delta})$ is inevitable in certain
cases (e.g., Theorem 2.1 in \citet{pmlr-v202-sadiev23a}). Hence,
we leave this challenging problem (finding the high-probability last-iterate
bound without clipping) as a future direction.

\subsection{General Convex Functions under Heavy-Tailed Noise}

Now we are ready to provide the convergence results under heavy-tailed
noise.
\begin{thm}
\label{thm:main-heavy-cvx-exp}Under Assumptions \ref{enu:A1}-\ref{enu:A4}
and \ref{enu:A5C} with $\mu_{f}=\mu_{h}=0$:

If $T$ is unknown, by taking $\eta_{t\in\left[T\right]}=\frac{\eta_{*}}{t^{\frac{2-p}{p}}}\land\frac{\eta}{t^{\frac{1}{p}}}$
with $\eta_{*}=\Theta\left(\frac{\left[D_{\psi}(x^{*},x^{1})\right]^{\frac{2-p}{p}}}{L}\right)$
and $\eta=\Theta\left(\left[\frac{D_{\psi}(x^{*},x^{1})}{M^{p}+\sigma^{p}}\right]^{\frac{1}{p}}\right)$,
there is
\[
\E\left[F(x^{T+1})-F(x^{*})\right]\leq O\left(\frac{L\left[D_{\psi}(x^{*},x^{1})\right]^{2-\frac{2}{p}}\log T}{T^{2-\frac{2}{p}}}+\frac{(M+\sigma)\left[D_{\psi}(x^{*},x^{1})\right]^{1-\frac{1}{p}}\log T}{T^{1-\frac{1}{p}}}\right).
\]

If $T$ is known, by taking $\eta_{t\in\left[T\right]}=\frac{\eta_{*}}{T^{\frac{2-p}{p}}}\land\frac{\eta}{T^{\frac{1}{p}}}$
with $\eta_{*}=\Theta\left(\frac{1}{L}\left[\frac{D_{\psi}(x^{*},x^{1})}{\log T}\right]^{\frac{2-p}{p}}\right)$
and $\eta=\Theta\left(\left[\frac{D_{\psi}(x^{*},x^{1})}{(M^{p}+\sigma^{p})\log T}\right]^{\frac{1}{p}}\right)$,
there is
\[
\E\left[F(x^{T+1})-F(x^{*})\right]\leq O\left(\frac{L\left[D_{\psi}(x^{*},x^{1})\right]^{2-\frac{2}{p}}(\log T)^{\frac{2}{p}-1}}{T^{2-\frac{2}{p}}}+\frac{(M+\sigma)\left[D_{\psi}(x^{*},x^{1})\right]^{1-\frac{1}{p}}(\log T)^{\frac{1}{p}}}{T^{1-\frac{1}{p}}}\right).
\]
\end{thm}

Theorem \ref{thm:main-heavy-cvx-exp} is the first theoretical result
proving the last-iterate convergence of stochastic gradient methods
under heavy-tailed noise. Note that the rate is near optimal, matching
the lower bound $\Omega(1/T^{1-\frac{1}{p}})$ \citep{nemirovskij1983problem,5394945,pmlr-v178-vural22a}
up to a logarithmic factor. In Theorem \ref{thm:main-heavy-cvx-exp-optimal},
we further show how to remove the unsatisfied extra logarithmic factor
when $T$ is known. Moreover, it is worth pointing out that Theorem
\ref{thm:main-heavy-cvx-exp} recovers Theorem \ref{thm:main-cvx-exp}
perfectly if enforcing $p=2$. In Theorem \ref{thm:heavy-cvx-exp}
in the appendix, we provide the convergence theory for arbitrary positive
$\eta_{*}$ and $\eta$.

\subsubsection{Optimal Rate under Heavy-Tailed Noise}
\begin{thm}
\label{thm:main-heavy-cvx-exp-optimal}Under Assumptions \ref{enu:A1}-\ref{enu:A4}
and \ref{enu:A5C} with $\mu_{f}=\mu_{h}=0$, if $T$ is known, by
taking $\eta_{t\in\left[T\right]}=\frac{\eta_{*}(T-t+1)}{T^{\frac{2}{p}}}\land\frac{\eta(T-t+1)}{T^{\frac{1}{p}+1}}$
with $\eta_{*}=\Theta\left(\frac{\left[D_{\psi}(x^{*},x^{1})\right]^{\frac{2-p}{p}}}{L}\right)$
and $\eta=\Theta\left(\left[\frac{D_{\psi}(x^{*},x^{1})}{M^{p}+\sigma^{p}}\right]^{\frac{1}{p}}\right)$,
there is
\[
\E\left[F(x^{T+1})-F(x^{*})\right]\leq O\left(\frac{L\left[D_{\psi}(x^{*},x^{1})\right]^{2-\frac{2}{p}}}{T^{2-\frac{2}{p}}}+\frac{(M+\sigma)\left[D_{\psi}(x^{*},x^{1})\right]^{1-\frac{1}{p}}}{T^{1-\frac{1}{p}}}\right).
\]
\end{thm}

The optimal convergence rate for known $T$ is presented in Theorem
\ref{thm:main-heavy-cvx-exp-optimal}. The new step size generalizes
the learning rate used in Theorem \ref{thm:main-cvx-exp-optimal}
(but still exhibits a linear decay) to handle heavy-tailed noise.
Hence, for stochastic convex optimization under heavy-tailed noise,
we show that the convergence rate of the last iterate can match the
lower bound perfectly. The full version with any positive $\eta_{*}$
and $\eta$ is provided in Theorem \ref{thm:heavy-cvx-exp-optimal}
in the appendix.

%% file: weibull.tex
\section{Last-Iterate Convergence under Sub-Weibull Noise\label{sec:main-weibull}}

In Section \ref{sec:main-algo}, we establish high-probability last-iterate
convergence under the assumption of sub-Gaussian noise. It is natural
to ask whether it is possible to relax it to, for example, sub-exponential
noise. In this section, we provide an affirmative answer by considering
a general family of noise, sub-Weibull noise with a parameter $p\in\left(0,2\right)$
(see Assumption \ref{enu:A5D} below), which contains sub-exponential
noise as a special case ($p=1$). The sub-Weibull distribution was
originally proposed by \citet{vladimirova2020sub,kuchibhotla2022moving}.
It has already been studied under different problems (e.g., bandit
problems \citep{NEURIPS2019_412758d0,NEURIPS2019_aff0a6a4}).

\subsection{Additional Related Work}

\textbf{Convergence under sub-Weibull noise: }As far as we know, when
considering convex optimization, only few convergence results have
been established under sub-Weibull noise. Hence, we include references
that also consider non-convex optimization or online optimization,
e.g., \citet{JMLR:v25:23-0466,9594083,pmlr-v162-li22q,9758044,liu2023on}.

\subsection{New Assumption and Some Discussions}

The following assumption describes what sub-Weibull noise is:
\begin{enumerate}
\item[\textbf{5D.}] \setref{5D}\label{enu:A5D}\textbf{Sub-Weibull noise:} $\exists p\in\left(0,2\right)$,
$\sigma\geq0$ denoting the noise level such that $\E\left[\exp\left(\lambda\left\Vert \xi^{t}\right\Vert _{*}^{p}\right)\mid\F^{t-1}\right]\leq\exp\left(\lambda\sigma^{p}\right),\forall\lambda\in\left[0,\sigma^{-p}\right]$.
\end{enumerate}
If we enforce $p=2$ in Assumption \ref{enu:A5D}, it recovers Assumption
\ref{enu:A5B} (i.e., sub-Gaussian noise). Intuitively, one can think
of the tail of sub-Weibull noise with $p<2$ decays strictly slower
than the tail of sub-Gaussian noise. We emphasize that it is also
possible to obtain the high-probability bound for $p>2$. However,
the case of $p>2$ is not harder than the sub-Gaussian case ($p=2$).
So, we specify $p\in(0,2)$ in this section.

Next, we only provide the convergence theorems for general convex
functions in this section for simplicity. In the unified analysis
(see Section \ref{sec:weibull-analysis} in the appendix), the strongly
convex case is stated as well. Therefore, the interested reader can
still apply our Lemma \ref{lem:weibull-core-hp} to obtain the high-probability
bound for the strongly convex case, following similar steps used in
proving Theorem \ref{thm:main-str-hp}.

Moreover, under the new Assumption \ref{enu:A5D}, certain barriers
will arise when applying the technique of \citet{pmlr-v202-liu23aa},
i.e., the auxiliary weight sequence $w_{t\in\left[T\right]}$ used
in Lemmas \ref{lem:core-general} and \ref{lem:core-hp}. To deal
with this newly risen issue in the high-probability analysis, we will
introduce a different technique, borrowed from \citet{pmlr-v202-ivgi23a}
and \citet{liu2023stochastic}, to replace the weight sequence $w_{t\in\left[T\right]}$
previously used. A detailed discussion is provided in Appendix \ref{sec:weibull-analysis}.

Lastly, we state a useful technical tool, Lemma \ref{lem:weibull},
the proof of which essentially follows the same way as proving concentration
bounds for sub-Gaussian and sub-exponential random variables in \citet{vershynin2018high}.
For details, see Section \ref{sec:tech} in the appendix.
\begin{lem}
\label{lem:weibull}Given a sigma algebra $\F$ and a random vector
$Z\in\R^{d}$ that is $\F$-measurable, if $\xi\in\R^{d}$ is a random
vector satisfying $\E\left[\xi\mid\F\right]=0$ and $\E\left[\exp\left(\lambda\left\Vert \xi\right\Vert _{*}^{p}\right)\mid\F\right]\leq\exp\left(\lambda\sigma^{p}\right),\forall\lambda\in\left[0,\sigma^{-p}\right]$
for $p\in\left[1,2\right)$, then
\[
\E\left[\exp\left(\left\langle \xi,Z\right\rangle \right)\mid\F\right]\leq\begin{cases}
\exp\left(6\sigma^{2}\left\Vert Z\right\Vert ^{2}\right) & \text{if }\left\Vert Z\right\Vert \leq\frac{1}{2\sigma}\text{ almost surely}\\
\exp\left(6\sigma^{2}\left\Vert Z\right\Vert ^{2}+2^{q-1}\sigma^{q}\left\Vert Z\right\Vert ^{q}\right) & p\in\left(1,2\right),q\coloneqq\frac{p}{p-1}
\end{cases}.
\]
\end{lem}

\subsection{General Convex Functions under Sub-Weibull Noise}

We first introduce a function $C(\delta,p)=\begin{cases}
O\left(1+\log^{\frac{2}{p}}\frac{1}{\delta}\right) & p\in\left[1,2\right)\\
O\left(1+\log^{1+\frac{2}{p}}\frac{1}{\delta}\right) & p\in\left(0,1\right)
\end{cases}$, which is used to reflect the dependence on $\log\frac{1}{\delta}$
for different $p$, the detailed definition is provided in (\ref{eq:weibull-core-hp-C})
in Appendix \ref{sec:weibull-analysis}. When $p\in\left(0,1\right)$,
it is possible to replace $C(\delta,p)=O\left(1+\log^{1+\frac{2}{p}}\frac{1}{\delta}\right)$
by $C(\delta,p)=O\left(1+\log^{\frac{2}{p}}\frac{T}{\delta}\right)$
by applying the concentration bound from \citet{JMLR:v25:23-0466}.
However, we keep the factor $\mathrm{polylog\frac{1}{\delta}}$ to
get rid of $T$. Now we can state the convergence results.
\begin{thm}
\label{thm:main-weibull-cvx-hp}Under Assumptions \ref{enu:A1}-\ref{enu:A4}
and \ref{enu:A5D} with $\mu_{f}=\mu_{h}=0$ and let $\delta\in(0,1)$:

If $T$ is unknown, by taking $\eta_{t\in\left[T\right]}=\frac{1}{2L}\land\frac{\eta}{\sqrt{t}}$
with $\eta=\Theta\left(\sqrt{\frac{D_{\psi}(x^{*},x^{1})}{M^{2}+\sigma^{2}C(\delta,p)}}\right)$,
then with probability at least $1-\delta$, there is
\[
F(x^{T+1})-F(x^{*})\leq O\left(\frac{LD_{\psi}(x^{*},x^{1})}{T}+\frac{(M+\sigma\sqrt{C(\delta,p)})\sqrt{D_{\psi}(x^{*},x^{1})}\log T}{\sqrt{T}}\right).
\]

If $T$ is known, by taking $\eta_{t\in\left[T\right]}=\frac{1}{2L}\land\frac{\eta}{\sqrt{T}}$
with $\eta=\Theta\left(\sqrt{\frac{D_{\psi}(x^{*},x^{1})}{(M^{2}+\sigma^{2}C(\delta,p))\log T}}\right)$,
then with probability at least $1-\delta$, there is
\[
F(x^{T+1})-F(x^{*})\leq O\left(\frac{LD_{\psi}(x^{*},x^{1})}{T}+\frac{(M+\sigma\sqrt{C(\delta,p)})\sqrt{D_{\psi}(x^{*},x^{1})\log T}}{\sqrt{T}}\right).
\]
\end{thm}

As far as we know, Theorem \ref{thm:main-weibull-cvx-hp} is the first
to give the last-iterate convergence under sub-Weibull noise. Its
full statement is provided in Appendix \ref{sec:weibull-cvx}. Compared
with Theorem \ref{thm:main-cvx-hp}, we only incur an extra $O\left(\mathrm{polylog}\frac{1}{\delta}\right)$
factor. However, we want to point out that $\sqrt{C(\delta,p)}$ is
not continuous in $p$, as there is a difference when $p$ is approaching
$1$ from two sides.

\subsubsection{Optimal Rate under Sub-Weibull Noise}
\begin{thm}
\label{thm:main-weibull-cvx-hp-optimal}Under Assumptions \ref{enu:A1}-\ref{enu:A4}
and \ref{enu:A5D} with $\mu_{f}=\mu_{h}=0$ and let $\delta\in(0,1)$,
if $T$ is known, by taking $\eta_{t\in\left[T\right]}=\frac{T-t+1}{2LT}\land\frac{\eta(T-t+1)}{T^{\frac{3}{2}}}$
with $\eta=\Theta\left(\sqrt{\frac{D_{\psi}(x^{*},x^{1})}{M^{2}+\sigma^{2}C(\delta,p)}}\right)$,
then with probability at least $1-\delta$, there is
\[
F(x^{T+1})-F(x^{*})\leq O\left(\frac{LD_{\psi}(x^{*},x^{1})}{T}+\frac{(M+\sigma\sqrt{C(\delta,p)})\sqrt{D_{\psi}(x^{*},x^{1})}}{\sqrt{T}}\right).
\]
\end{thm}

With almost the same step size in Theorem \ref{thm:main-cvx-hp-optimal},
we can also get rid of the extra $O\left(\sqrt{\log T}\right)$ factor
in Theorem \ref{thm:main-weibull-cvx-hp} to obtain the optimal dependence
in time $T$. The detailed version for any positive $\eta$ can be
found in Appendix \ref{subsec:weibull-cvx-hp-optimal}.

%% file: conclusion.tex
\section{Conclusion\label{sec:conclusion}}

In this work, we present a unified analysis for the last-iterate convergence
of stochastic gradient methods and obtain several new results. More
specifically, we establish the (nearly) optimal convergence of the
last iterate of the CSMD algorithm both in expectation and in high
probability. Our proofs can not only handle different function classes
simultaneously but also be applied to composite problems with non-Euclidean
norms on general domains. We believe our work develops a deeper understanding
of stochastic gradient methods. However, there still remain many directions
worth exploring. For example, it could be interesting to see whether
our proof can be extended to adaptive gradient methods like AdaGrad
\citep{McMahanS10,JMLR:v12:duchi11a}. We leave this important question
as future work and expect it to be addressed.

%% file: appendix.tex
\section{Technical Lemmas\label{sec:tech}}

\subsection{Proof of Lemma \ref{lem:gaussian}}

Before giving the proof of Lemma \ref{lem:gaussian}, we need the
following property of sub-Gaussian vectors. This result is already
known before (see \citet{vershynin2018high}). We provide a proof
here to make the paper self-consistent.
\begin{lem}
\label{lem:sub-gaussian-moment}Given a sigma algebra $\F$, if $\xi\in\R^{d}$
is a random vector satisfying $\E\left[\exp\left(\lambda\left\Vert \xi\right\Vert _{*}^{2}\right)\mid\F\right]\leq\exp\left(\lambda\sigma^{2}\right),\forall\lambda\in\left[0,\sigma^{-2}\right]$,
then for any integer $k\geq1$, we have
\[
\E\left[\left\Vert \xi\right\Vert _{*}^{2k}\mid\F\right]\leq\begin{cases}
\sigma^{2} & k=1\\
e(k!)\sigma^{2k} & k\geq2
\end{cases}.
\]
\end{lem}

\begin{proof}
For the case $k=1$, given any $\lambda\in\left[0,\sigma^{-2}\right]$,
there is
\[
\exp\left(\E\left[\lambda\left\Vert \xi\right\Vert _{*}^{2}\mid\F\right]\right)\leq\E\left[\exp\left(\lambda\left\Vert \xi\right\Vert _{*}^{2}\right)\mid\F\right]\leq\exp\left(\lambda\sigma^{2}\right)\Rightarrow\E\left[\left\Vert \xi\right\Vert _{*}^{2}\mid\F\right]\leq\sigma^{2}.
\]
For $k\geq2,$ we have
\begin{align*}
\E\left[\left\Vert \xi\right\Vert _{*}^{2k}\mid\F\right] & =\E\left[\int_{0}^{\infty}2kt^{2k-1}\mathds{1}\left[\left\Vert \xi\right\Vert _{*}\geq t\right]\mathrm{d}t\mid\F\right]=\int_{0}^{\infty}2kt^{2k-1}\E\left[\mathds{1}\left[\left\Vert \xi\right\Vert _{*}\geq t\right]\mid\F\right]\mathrm{d}t\\
 & \leq\int_{0}^{\infty}2kt^{2k-1}\E\left[\frac{\exp\left(\sigma^{-2}\left\Vert \xi\right\Vert _{*}^{2}\right)}{\exp(\sigma^{-2}t^{2})}\mid\F\right]\mathrm{d}t\overset{(a)}{\leq}\int_{0}^{\infty}2ekt^{2k-1}\exp(-\sigma^{-2}t^{2})\mathrm{d}t\\
 & \overset{(b)}{=}\int_{0}^{\infty}ek\sigma^{2k}s^{k-1}\exp(-s)\mathrm{d}s=ek\sigma^{2k}\Gamma(k)=e(k!)\sigma^{2k},
\end{align*}
where $(a)$ is by $\E\left[\exp\left(\sigma^{-2}\left\Vert \xi\right\Vert _{*}^{2}\right)\mid\F\right]\leq\exp(\sigma^{-2}\sigma^{2})=e$
and $(b)$ is by the change of variable $t=\sigma\sqrt{s}$.
\end{proof}

Now we are ready to prove Lemma \ref{lem:gaussian}.

\begin{proof}[Proof of Lemma \ref{lem:gaussian}]
Note that
\begin{align*}
\E\left[\exp\left(\left\langle \xi,Z\right\rangle \right)\mid\F\right] & =\E\left[1+\left\langle \xi,Z\right\rangle +\sum_{k=2}^{\infty}\frac{\left(\left\langle \xi,Z\right\rangle \right)^{k}}{k!}\mid\F\right]\leq\E\left[\left\langle \xi,Z\right\rangle \mid\F\right]+\E\left[1+\sum_{k=2}^{\infty}\frac{\left\Vert \xi\right\Vert _{*}^{k}\left\Vert Z\right\Vert ^{k}}{k!}\mid\F\right]\\
 & \overset{(a)}{=}\E\left[1+\sum_{k=1}^{\infty}\frac{\left\Vert \xi\right\Vert _{*}^{2k}\left\Vert Z\right\Vert ^{2k}}{(2k)!}+\sum_{k=1}^{\infty}\frac{\left\Vert \xi\right\Vert _{*}^{2k+1}\left\Vert Z\right\Vert ^{2k+1}}{(2k+1)!}\mid\F\right]\\
 & \overset{(b)}{\leq}\E\left[1+\sum_{k=1}^{\infty}\frac{\left\Vert \xi\right\Vert _{*}^{2k}\left\Vert Z\right\Vert ^{2k}}{(2k)!}+\sum_{k=1}^{\infty}\frac{\left\Vert \xi\right\Vert _{*}^{2k}\left\Vert Z\right\Vert ^{2k}+\left\Vert \xi\right\Vert _{*}^{2k+2}\left\Vert Z\right\Vert ^{2k+2}/4}{(2k+1)!}\mid\F\right]\\
 & =\E\left[1+\frac{2\left\Vert \xi\right\Vert _{*}^{2}\left\Vert Z\right\Vert ^{2}}{3}+\sum_{k=2}^{\infty}\left\Vert \xi\right\Vert _{*}^{2k}\left\Vert Z\right\Vert ^{2k}\left(\frac{1}{4(2k-1)!}+\frac{1}{(2k)!}+\frac{1}{(2k+1)!}\right)\mid\F\right]\\
 & =\E\left[1+\frac{2\left\Vert \xi\right\Vert _{*}^{2}\left\Vert Z\right\Vert ^{2}}{3}+\sum_{k=2}^{\infty}\left\Vert \xi\right\Vert _{*}^{2k}\left\Vert Z\right\Vert ^{2k}\frac{1+k/2+1/(2k+1)}{(2k)!}\mid\F\right]\\
 & \overset{(c)}{\leq}1+\frac{2\sigma^{2}\left\Vert Z\right\Vert ^{2}}{3}+\sum_{k=2}^{\infty}\frac{\sigma^{2k}\left\Vert Z\right\Vert ^{2k}}{k!}\cdot\frac{e(1+k/2+1/(2k+1))}{\binom{2k}{k}}\\
 & \overset{(d)}{\leq}1+\sigma^{2}\left\Vert Z\right\Vert ^{2}+\sum_{k=2}^{\infty}\frac{\sigma^{2k}\left\Vert Z\right\Vert ^{2k}}{k!}=\exp\left(\sigma^{2}\left\Vert Z\right\Vert ^{2}\right),
\end{align*}
where $(a)$ is by $\E\left[\left\langle \xi,Z\right\rangle \mid\F\right]=\left\langle \E\left[\xi\mid\F\right],Z\right\rangle =0$,
$(b)$ holds due to AM-GM inequality, $(c)$ is by applying Lemma
\ref{lem:sub-gaussian-moment} to $\E\left[\left\Vert \xi\right\Vert _{*}^{2k}\left\Vert Z\right\Vert ^{2k}\mid\F\right]=\left\Vert Z\right\Vert ^{2k}\E\left[\left\Vert \xi\right\Vert _{*}^{2k}\mid\F\right]$
and $(d)$ is by $2/3<1$ and
\[
\max_{k\geq2,k\in\N}\frac{e(1+k/2+1/(2k+1))}{\binom{2k}{k}}=\frac{e(1+1+1/5)}{6}<1.
\]
\end{proof}

\subsection{Proof of Lemma \ref{lem:weibull}}

To prove Lemma \ref{lem:weibull}, we need the following moment bound
on sub-Weibull random variables, which could be viewed as a generalization
of Lemma \ref{lem:sub-gaussian-moment}. Similar results can also
be found in, for example, \citet{vladimirova2020sub}.
\begin{lem}
\label{lem:sub-weibull-moment}Given a sigma algebra $\F$, if $\xi\in\R^{d}$
is a random vector satisfying $\E\left[\exp\left(\lambda\left\Vert \xi\right\Vert _{*}^{p}\right)\mid\F\right]\leq\exp\left(\lambda\sigma^{p}\right),\forall\lambda\in\left[0,\sigma^{-p}\right]$
for some $p\in\left(0,2\right)$, then for any $k\geq1$ (not necessarily
an integer), we have
\[
\E\left[\left\Vert \xi\right\Vert _{*}^{k}\mid\F\right]\leq\begin{cases}
e(k!)\sigma^{k} & k\in\N\text{ and }p\in\left[1,2\right)\\
e\left(\frac{k}{p}\right)^{\frac{k}{p}}\sigma^{k} & k\geq p
\end{cases}.
\]
\end{lem}

\begin{proof}
Following a similar way of proving Lemma \ref{lem:sub-gaussian-moment},
we can obtain that for any $k\geq1$,
\[
\E\left[\left\Vert \xi\right\Vert _{*}^{k}\mid\F\right]\leq e\sigma^{k}\Gamma\left(\frac{k}{p}+1\right).
\]
By the recursive definition of the Gamma function, there is
\[
\Gamma\left(\frac{k}{p}+1\right)=\Gamma\left(\frac{k}{p}-\left\lfloor \frac{k}{p}\right\rfloor +1\right)\prod_{i=0}^{\left\lfloor \frac{k}{p}\right\rfloor -1}\left(\frac{k}{p}-i\right)\leq\prod_{i=0}^{\left\lfloor \frac{k}{p}\right\rfloor -1}\left(\frac{k}{p}-i\right),
\]
where the last step is by $\Gamma\left(\frac{k}{p}-\left\lfloor \frac{k}{p}\right\rfloor +1\right)\leq1$
due to $1\leq\frac{k}{p}-\left\lfloor \frac{k}{p}\right\rfloor +1<2$.
Finally, when $p\in\left[1,2\right)$ and $k\in\N$, we can bound
\[
\prod_{i=0}^{\left\lfloor \frac{k}{p}\right\rfloor -1}\left(\frac{k}{p}-i\right)\leq\prod_{i=0}^{\left\lfloor \frac{k}{p}\right\rfloor -1}\left(k-i\right)\leq\prod_{i=0}^{k-1}\left(k-i\right)=k!.
\]
If $k\geq p$, we can bound
\[
\prod_{i=0}^{\left\lfloor \frac{k}{p}\right\rfloor -1}\left(\frac{k}{p}-i\right)\leq\left(\frac{k}{p}\right)^{\left\lfloor \frac{k}{p}\right\rfloor }\leq\left(\frac{k}{p}\right)^{\frac{k}{p}}.
\]
\end{proof}

Now we are ready to prove Lemma \ref{lem:weibull}.

\begin{proof}[Proof of Lemma \ref{lem:weibull}]
Note that
\begin{align}
\E\left[\exp\left(\left\langle \xi,Z\right\rangle \right)\mid\F\right] & =\E\left[1+\left\langle \xi,Z\right\rangle +\sum_{k=2}^{\infty}\frac{\left(\left\langle \xi,Z\right\rangle \right)^{k}}{k!}\mid\F\right]\overset{(a)}{\leq}\E\left[1+\sum_{k=2}^{\infty}\frac{\left\Vert \xi\right\Vert _{*}^{k}\left\Vert Z\right\Vert ^{k}}{k!}\mid\F\right]\nonumber \\
 & \overset{(b)}{\leq}1+e\sum_{k=2}^{\infty}\sigma^{k}\left\Vert Z\right\Vert ^{k}\leq\exp\left(e\sum_{k=2}^{\infty}\sigma^{k}\left\Vert Z\right\Vert ^{k}\right),\label{eq:sub-weibull-exponential-1}
\end{align}
where $(a)$ is due to $\E\left[\left\langle \xi,Z\right\rangle \mid\F\right]=\left\langle \E\left[\xi\mid\F\right],Z\right\rangle =0$
and $\left(\left\langle \xi,Z\right\rangle \right)^{k}\leq\left(\left|\left\langle \xi,Z\right\rangle \right|\right)^{k}\leq\left\Vert \xi\right\Vert _{*}^{k}\left\Vert Z\right\Vert ^{k}$,
and $(b)$ is by applying Lemma \ref{lem:sub-weibull-moment} to $\E\left[\left\Vert \xi\right\Vert _{*}^{k}\left\Vert Z\right\Vert ^{k}\mid\F\right]=\left\Vert Z\right\Vert ^{k}\E\left[\left\Vert \xi\right\Vert _{*}^{k}\mid\F\right]$.

Therefore, if $\left\Vert Z\right\Vert \leq\frac{1}{2\sigma}$ almost
surely, we have
\begin{equation}
e\sum_{k=2}^{\infty}\sigma^{k}\left\Vert Z\right\Vert ^{k}\leq2e\sigma^{2}\left\Vert Z\right\Vert ^{2}\leq6\sigma^{2}\left\Vert Z\right\Vert ^{2}\Rightarrow\exp\left(e\sum_{k=2}^{\infty}\sigma^{k}\left\Vert Z\right\Vert ^{k}\right)\leq\exp\left(6\sigma^{2}\left\Vert Z\right\Vert ^{2}\right).\label{eq:sub-weibull-exponential-2}
\end{equation}

When $p\in\left(1,2\right)$, recall that $q=\frac{p}{p-1}$, we can
also bound
\[
\left\langle \xi,Z\right\rangle \leq\left\Vert \xi\right\Vert _{*}\left\Vert Z\right\Vert \leq\frac{\left\Vert \xi\right\Vert _{*}^{p}}{2p\sigma^{p}}+\frac{2^{q/p}\sigma^{q}\left\Vert Z\right\Vert ^{q}}{q},
\]
where the first step is due to Cauchy-Schwarz inequality and the second
step is by Young's inequality. Therefore, it can be obtained
\begin{equation}
\E\left[\exp\left(\left\langle \xi,Z\right\rangle \right)\mid\F\right]\leq\E\left[\exp\left(\frac{\left\Vert \xi\right\Vert _{*}^{p}}{2p\sigma^{p}}+\frac{2^{q/p}\sigma^{q}\left\Vert Z\right\Vert ^{q}}{q}\right)\mid\F\right]\leq\exp\left(\frac{1}{2p}+\frac{2^{q/p}\sigma^{q}\left\Vert Z\right\Vert ^{q}}{q}\right),\label{eq:sub-weibull-exponential-3}
\end{equation}
where the last step is due to $Z\in\F$ and $\E\left[\exp\left(\frac{\left\Vert \xi\right\Vert _{*}^{p}}{2p\sigma^{p}}\right)\mid\F\right]\leq\exp\left(\frac{1}{2p}\right)$
since $\frac{1}{2p}\leq1$. Now, let us define the event $A\coloneqq\left\{ \left\Vert Z\right\Vert \leq\frac{1}{2\sigma}\right\} $
and bound
\begin{align*}
\E\left[\exp\left(\left\langle \xi,Z\right\rangle \right)\mid\F\right] & =\E\left[\exp\left(\left\langle \xi,Z\right\rangle \right)\mid\F\right]\left(\mathds{1}\left[A\right]+\mathds{1}\left[A^{c}\right]\right)\\
 & \overset{(\ref{eq:sub-weibull-exponential-1}),(\ref{eq:sub-weibull-exponential-3})}{\leq}\exp\left(e\sum_{k=2}^{\infty}\sigma^{k}\left\Vert Z\right\Vert ^{k}\right)\mathds{1}\left[A\right]+\exp\left(\frac{1}{2p}+\frac{2^{q/p}\sigma^{q}\left\Vert Z\right\Vert ^{q}}{q}\right)\mathds{1}\left[A^{c}\right]\\
 & \overset{(c)}{\leq}\exp\left(6\sigma^{2}\left\Vert Z\right\Vert ^{2}\right)\mathds{1}\left[A\right]+\exp\left(2^{q-1}\sigma^{q}\left\Vert Z\right\Vert ^{q}\right)\mathds{1}\left[A^{c}\right]\leq\exp\left(6\sigma^{2}\left\Vert Z\right\Vert ^{2}+2^{q-1}\sigma^{q}\left\Vert Z\right\Vert ^{q}\right),
\end{align*}
where in $(c)$ we use (\ref{eq:sub-weibull-exponential-2}) when
$\left\Vert Z\right\Vert \leq\frac{1}{2\sigma}$, and $\frac{1}{2p}+\frac{2^{q/p}\sigma^{q}\left\Vert Z\right\Vert ^{q}}{q}\leq\frac{2^{q}\sigma^{q}\left\Vert Z\right\Vert ^{q}}{2p}+\frac{2^{q/p}\sigma^{q}\left\Vert Z\right\Vert ^{q}}{q}=2^{q-1}\sigma^{q}\left\Vert Z\right\Vert ^{q}$
when $\left\Vert Z\right\Vert >\frac{1}{2\sigma}$, where the last
equation is by $q/p=q-1$ and $1/p+1/q=1$.
\end{proof}

\section{Missing Proofs in Section \ref{sec:analysis}\label{sec:proofs}}

In this section, we provide the missing proofs of the three most important
lemmas.

\subsection{Proof of Lemma \ref{lem:core-general}}

\begin{proof}[Proof of Lemma \ref{lem:core-general}]
Inspired by \citet{doi:10.1137/24M1717762}, we first introduce the
following auxiliary sequence
\[
z^{t}=\begin{cases}
\left(1-\frac{v_{t-1}}{v_{t}}\right)x^{t}+\frac{v_{t-1}}{v_{t}}z^{t-1} & t\in\left[T\right]\\
x & t=0
\end{cases}\Leftrightarrow z^{t}=\frac{v_{0}}{v_{t}}x+\sum_{s=1}^{t}\frac{v_{s}-v_{s-1}}{v_{t}}x^{s},\forall t\in\left\{ 0\right\} \cup\left[T\right],
\]
where we recall that $v_{t\in\left\{ 0\right\} \cup\left[T\right]}=\frac{w_{T}\gamma_{T}}{\sum_{s=t\lor1}^{T}w_{s}\gamma_{s}}\geq0$
is non-decreasing. Note that $z^{t}$ always falls in the domain $\dom$
by its definition.

Now, we start the proof from the $(L,M)$-smoothness of $f$,
\begin{align}
f(x^{t+1})-f(x^{t})\leq & \left\langle g^{t},x^{t+1}-x^{t}\right\rangle +\frac{L}{2}\left\Vert x^{t+1}-x^{t}\right\Vert ^{2}+M\left\Vert x^{t+1}-x^{t}\right\Vert \nonumber \\
= & \left\langle \xi^{t},z^{t}-x^{t}\right\rangle +\underbrace{\left\langle \xi^{t},x^{t}-x^{t+1}\right\rangle }_{\text{I}}+\underbrace{\left\langle \hg^{t},x^{t+1}-z^{t}\right\rangle }_{\text{II}}\nonumber \\
 & +\underbrace{\left\langle g^{t},z^{t}-x^{t}\right\rangle }_{\text{III}}+\underbrace{\frac{L}{2}\left\Vert x^{t+1}-x^{t}\right\Vert ^{2}+M\left\Vert x^{t+1}-x^{t}\right\Vert }_{\text{IV}},\label{eq:core-general-1}
\end{align}
where $g^{t}\coloneqq\E\left[\hg^{t}\vert\F^{t-1}\right]\in\partial f(x^{t})$
and $\xi^{t}=\hg^{t}-g^{t}$. Next, we bound these four terms respectively.
\begin{itemize}
\item For term I, by applying Cauchy-Schwarz inequality, the $1$-strong
convexity of $\psi$ and AM-GM inequality, we can get the following
upper bound
\begin{equation}
\text{I}\leq\left\Vert \xi^{t}\right\Vert _{*}\left\Vert x^{t}-x^{t+1}\right\Vert \leq\left\Vert \xi^{t}\right\Vert _{*}\sqrt{2D_{\psi}(x^{t+1},x^{t})}\leq2\eta_{t}\left\Vert \xi^{t}\right\Vert _{*}^{2}+\frac{D_{\psi}(x^{t+1},x^{t})}{4\eta_{t}}.\label{eq:core-general-I}
\end{equation}
\item For term II, we recall that the update rule is $x^{t+1}=\argmin_{x\in\dom}h(x)+\left\langle \hg^{t},x-x^{t}\right\rangle +\frac{D_{\psi}(x,x^{t})}{\eta_{t}}$.
Hence, by the optimality condition of $x^{t+1}$, there exists $h^{t+1}\in\partial h(x^{t+1})$
such that for any $y\in\dom$
\[
\left\langle h^{t+1}+\hg^{t}+\frac{\na\psi(x^{t+1})-\na\psi(x^{t})}{\eta_{t}},x^{t+1}-y\right\rangle \leq0,
\]
which implies
\begin{align*}
\left\langle \hg^{t},x^{t+1}-y\right\rangle  & \leq\frac{\left\langle \na\psi(x^{t})-\na\psi(x^{t+1}),x^{t+1}-y\right\rangle }{\eta_{t}}+\left\langle h^{t+1},y-x^{t+1}\right\rangle \\
 & \leq\frac{D_{\psi}(y,x^{t})-D_{\psi}(y,x^{t+1})-D_{\psi}(x^{t+1},x^{t})}{\eta_{t}}+h(y)-h(x_{t+1})-\mu_{h}D_{\psi}(y,x^{t+1}),
\end{align*}
where the last inequality holds due to $\left\langle \na\psi(x^{t})-\na\psi(x^{t+1}),x^{t+1}-y\right\rangle =D_{\psi}(y,x^{t})-D_{\psi}(y,x^{t+1})-D_{\psi}(x^{t+1},x^{t})$
and $\left\langle h^{t+1},y-x^{t+1}\right\rangle \leq h(y)-h(x_{t+1})-\mu_{h}D_{\psi}(y,x^{t+1})$
by the $\mu_{h}$-strong convexity of $h$. We substitute $y$ with
$z^{t}$ to obtain
\begin{equation}
\text{II}\leq\frac{D_{\psi}(z^{t},x^{t})-D_{\psi}(z^{t},x^{t+1})-D_{\psi}(x^{t+1},x^{t})}{\eta_{t}}+h(z_{t})-h(x_{t+1})-\mu_{h}D_{\psi}(z^{t},x^{t+1}).\label{eq:core-general-II}
\end{equation}
\item For term III, we simply use the $\mu_{f}$-strong convexity of $f$
to get
\begin{equation}
\text{III}\leq f(z^{t})-f(x^{t})-\mu_{f}D_{\psi}(z^{t},x^{t}).\label{eq:core-general-III}
\end{equation}
\item For term IV, we have
\begin{align}
\text{IV} & \leq LD_{\psi}(x^{t+1},x^{t})+M\sqrt{2D_{\psi}(x^{t+1},x^{t})}\nonumber \\
 & \leq LD_{\psi}(x^{t+1},x^{t})+2\eta_{t}M^{2}+\frac{D_{\psi}(x^{t+1},x^{t})}{4\eta_{t}},\label{eq:core-general-IV}
\end{align}
where the first inequality holds by the $1$-strong convexity of $\psi$
again and the second one is due to AM-GM inequality.
\end{itemize}
By plugging the bounds (\ref{eq:core-general-I}), (\ref{eq:core-general-II}),
(\ref{eq:core-general-III}) and (\ref{eq:core-general-IV}) into
(\ref{eq:core-general-1}), we obtain
\begin{align*}
f(x^{t+1})-f(x^{t})\leq & \left\langle \xi^{t},z^{t}-x^{t}\right\rangle +2\eta_{t}\left\Vert \xi^{t}\right\Vert _{*}^{2}+\frac{D_{\psi}(x^{t+1},x^{t})}{4\eta_{t}}\\
 & +\frac{D_{\psi}(z^{t},x^{t})-D_{\psi}(z^{t},x^{t+1})-D_{\psi}(x^{t+1},x^{t})}{\eta_{t}}+h(z_{t})-h(x_{t+1})-\mu_{h}D_{\psi}(z^{t},x^{t+1})\\
 & +f(z^{t})-f(x^{t})-\mu_{f}D_{\psi}(z^{t},x^{t})+LD_{\psi}(x^{t+1},x^{t})+2\eta_{t}M^{2}+\frac{D_{\psi}(x^{t+1},x^{t})}{4\eta_{t}}.
\end{align*}
Rearranging the terms to get
\begin{align}
 & F(x^{t+1})-F(z^{t})\nonumber \\
\leq & \left\langle \xi^{t},z^{t}-x^{t}\right\rangle +(\eta_{t}^{-1}-\mu_{f})D_{\psi}(z^{t},x^{t})-(\eta_{t}^{-1}+\mu_{h})D_{\psi}(z^{t},x^{t+1})\nonumber \\
 & +\left(L-\frac{1}{2\eta_{t}}\right)D_{\psi}(x^{t+1},x^{t})+2\eta_{t}(M^{2}+\left\Vert \xi^{t}\right\Vert _{*}^{2})\nonumber \\
\overset{(a)}{\leq} & \left\langle \xi^{t},z^{t}-x^{t}\right\rangle +(\eta_{t}^{-1}-\mu_{f})D_{\psi}(z^{t},x^{t})-(\eta_{t}^{-1}+\mu_{h})D_{\psi}(z^{t},x^{t+1})+2\eta_{t}(M^{2}+\left\Vert \xi^{t}\right\Vert _{*}^{2})\nonumber \\
\overset{(b)}{=} & \frac{v_{t-1}}{v_{t}}\left\langle \xi^{t},z^{t-1}-x^{t}\right\rangle +(\eta_{t}^{-1}-\mu_{f})D_{\psi}(z^{t},x^{t})-(\eta_{t}^{-1}+\mu_{h})D_{\psi}(z^{t},x^{t+1})+2\eta_{t}(M^{2}+\left\Vert \xi^{t}\right\Vert _{*}^{2})\nonumber \\
\overset{(c)}{\leq} & \frac{v_{t-1}}{v_{t}}\left\langle \xi^{t},z^{t-1}-x^{t}\right\rangle +(\eta_{t}^{-1}-\mu_{f})\frac{v_{t-1}}{v_{t}}D_{\psi}(z^{t-1},x^{t})-(\eta_{t}^{-1}+\mu_{h})D_{\psi}(z^{t},x^{t+1})+2\eta_{t}(M^{2}+\left\Vert \xi^{t}\right\Vert _{*}^{2}),\label{eq:core-general-2}
\end{align}
where $(a)$ is by $\eta_{t\in\left[T\right]}\leq\frac{1}{2L\lor\mu_{f}}\leq\frac{1}{2L}\Rightarrow L-\frac{1}{2\eta_{t}}\leq0$,
$(b)$ holds due to the definition of $z^{t}=\left(1-\frac{v_{t-1}}{v_{t}}\right)x^{t}+\frac{v_{t-1}}{v_{t}}z^{t-1}$
implying $z^{t}-x^{t}=\frac{v_{t-1}}{v_{t}}\left(z^{t-1}-x^{t}\right)$,
$(c)$ is by noticing $\eta_{t\in\left[T\right]}\leq\frac{1}{2L\lor\mu_{f}}\leq\frac{1}{\mu_{f}}\Rightarrow\eta_{t}^{-1}-\mu_{f}\geq0$
and
\[
D_{\psi}(z^{t},x^{t})\overset{(d)}{\leq}\left(1-\frac{v_{t-1}}{v_{t}}\right)D_{\psi}(x^{t},x^{t})+\frac{v_{t-1}}{v_{t}}D_{\psi}(z^{t-1},x^{t})=\frac{v_{t-1}}{v_{t}}D_{\psi}(z^{t-1},x^{t}),
\]
where $(d)$ is by the convexity of the first argument in $D_{\psi}(\cdot,\cdot)$.

Multiplying both sides of (\ref{eq:core-general-2}) by $w_{t}\gamma_{t}v_{t}$
(all of these three terms are non-negative) and summing up from $t=1$
to $T$, we obtain
\begin{align}
 & \sum_{t=1}^{T}w_{t}\gamma_{t}v_{t}\left(F(x^{t+1})-F(z^{t})\right)\nonumber \\
\leq & \sum_{t=1}^{T}w_{t}\gamma_{t}v_{t-1}\left\langle \xi^{t},z^{t-1}-x^{t}\right\rangle +\sum_{t=1}^{T}2w_{t}\gamma_{t}\eta_{t}v_{t}(M^{2}+\left\Vert \xi^{t}\right\Vert _{*}^{2})\nonumber \\
 & +\sum_{t=1}^{T}w_{t}\gamma_{t}(\eta_{t}^{-1}-\mu_{f})v_{t-1}D_{\psi}(z^{t-1},x^{t})-w_{t}\gamma_{t}(\eta_{t}^{-1}+\mu_{h})v_{t}D_{\psi}(z^{t},x^{t+1})\nonumber \\
= & w_{1}\gamma_{1}(\eta_{1}^{-1}-\mu_{f})v_{0}D_{\psi}(z^{0},x^{1})-w_{T}\gamma_{T}(\eta_{T}^{-1}+\mu_{h})v_{T}D_{\psi}(z^{T},x^{T+1})+\sum_{t=1}^{T}2w_{t}\gamma_{t}\eta_{t}v_{t}(M^{2}+\left\Vert \xi^{t}\right\Vert _{*}^{2})\nonumber \\
 & +\sum_{t=1}^{T}w_{t}\gamma_{t}v_{t-1}\left\langle \xi^{t},z^{t-1}-x^{t}\right\rangle +\sum_{t=2}^{T}\left(w_{t}\gamma_{t}(\eta_{t}^{-1}-\mu_{f})-w_{t-1}\gamma_{t-1}(\eta_{t-1}^{-1}+\mu_{h})\right)v_{t-1}D_{\psi}(z^{t-1},x^{t})\nonumber \\
\overset{(e)}{=} & w_{1}(1-\mu_{f}\eta_{1})v_{0}D_{\psi}(x,x^{1})-w_{T}\gamma_{T}(\eta_{T}^{-1}+\mu_{h})v_{T}D_{\psi}(z^{T},x^{T+1})+\sum_{t=1}^{T}2w_{t}\gamma_{t}\eta_{t}v_{t}(M^{2}+\left\Vert \xi^{t}\right\Vert _{*}^{2})\nonumber \\
 & +\sum_{t=1}^{T}w_{t}\gamma_{t}v_{t-1}\left\langle \xi^{t},z^{t-1}-x^{t}\right\rangle +\sum_{t=2}^{T}(w_{t}-w_{t-1})\gamma_{t}(\eta_{t}^{-1}-\mu_{f})v_{t-1}D_{\psi}(z^{t-1},x^{t}),\label{eq:core-general-3}
\end{align}
where $(e)$ holds due to $\gamma_{1}(\eta_{1}^{-1}-\mu_{f})=\eta_{1}(\eta_{1}^{-1}-\mu_{f})=1-\mu_{f}\eta_{1}$,
$z^{0}=x$ and $\gamma_{t}(\eta_{t}^{-1}-\mu_{f})=\eta_{t}(\eta_{t}^{-1}-\mu_{f})\prod_{s=2}^{t}\frac{1+\mu_{h}\eta_{s-1}}{1-\mu_{f}\eta_{s}}=(\eta_{t-1}^{-1}+\mu_{h})\eta_{t-1}\prod_{s=2}^{t-1}\frac{1+\mu_{h}\eta_{s-1}}{1-\mu_{f}\eta_{s}}=\gamma_{t-1}(\eta_{t-1}^{-1}+\mu_{h}),\forall t\geq2$.

By the convexity of $F$ and the definition of $z^{t}=\frac{v_{0}}{v_{t}}x+\sum_{s=1}^{t}\frac{v_{s}-v_{s-1}}{v_{t}}x^{s}$,
we have
\[
F(z^{t})\leq\sum_{s=1}^{t}\frac{v_{s}-v_{s-1}}{v_{t}}F(x^{s})+\frac{v_{0}}{v_{t}}F(x),
\]
which implies
\begin{align*}
\sum_{t=1}^{T}w_{t}\gamma_{t}v_{t}\left(F(x^{t+1})-F(z^{t})\right)\geq & \sum_{t=1}^{T}\left[w_{t}\gamma_{t}v_{t}F(x^{t+1})-w_{t}\gamma_{t}\left(\sum_{s=1}^{t}(v_{s}-v_{s-1})F(x^{s})+v_{0}F(x)\right)\right]\\
= & \sum_{t=1}^{T}\left[w_{t}\gamma_{t}v_{t}\left(F(x^{t+1})-F(x)\right)-w_{t}\gamma_{t}\sum_{s=1}^{t}(v_{s}-v_{s-1})\left(F(x^{s})-F(x)\right)\right]\\
= & w_{T}\gamma_{T}v_{T}\left(F(x^{T+1})-F(x)\right)-\left(\sum_{t=1}^{T}w_{t}\gamma_{t}\right)(v_{1}-v_{0})\left(F(x^{1})-F(x)\right)\\
 & +\sum_{t=2}^{T}\left[w_{t-1}\gamma_{t-1}v_{t-1}-\left(\sum_{s=t}^{T}w_{s}\gamma_{s}\right)(v_{t}-v_{t-1})\right]\left(F(x^{t})-F(x)\right).
\end{align*}
Now, by the definition of $v_{t\in\left\{ 0\right\} \cup\left[T\right]}=\frac{w_{T}\gamma_{T}}{\sum_{s=t\lor1}^{T}w_{s}\gamma_{s}}$,
we observe that
\[
\left(\sum_{t=1}^{T}w_{t}\gamma_{t}\right)(v_{1}-v_{0})=\left(\sum_{t=1}^{T}w_{t}\gamma_{t}\right)\left(\frac{w_{T}\gamma_{T}}{\sum_{s=1}^{T}w_{s}\gamma_{s}}-\frac{w_{T}\gamma_{T}}{\sum_{s=1}^{T}w_{s}\gamma_{s}}\right)=0,
\]
and for $2\leq t\leq T$,
\begin{align*}
 & w_{t-1}\gamma_{t-1}v_{t-1}-\left(\sum_{s=t}^{T}w_{s}\gamma_{s}\right)(v_{t}-v_{t-1})=\left(\sum_{s=t-1}^{T}w_{s}\gamma_{s}\right)v_{t-1}-\left(\sum_{s=t}^{T}w_{s}\gamma_{s}\right)v_{t}\\
= & \left(\sum_{s=t-1}^{T}w_{s}\gamma_{s}\right)\frac{w_{T}\gamma_{T}}{\sum_{s=t-1}^{T}w_{s}\gamma_{s}}-\left(\sum_{s=t}^{T}w_{s}\gamma_{s}\right)\frac{w_{T}\gamma_{T}}{\sum_{s=t}^{T}w_{s}\gamma_{s}}=0.
\end{align*}
These two equations immediately imply
\begin{equation}
\sum_{t=1}^{T}w_{t}\gamma_{t}v_{t}\left(F(x^{t+1})-F(z^{t})\right)\geq w_{T}\gamma_{T}v_{T}\left(F(x^{T+1})-F(x)\right).\label{eq:core-general-4}
\end{equation}
Plugging (\ref{eq:core-general-4}) into (\ref{eq:core-general-3}),
we finally get
\begin{align*}
 & w_{T}\gamma_{T}v_{T}\left(F(x^{T+1})-F(x)\right)\\
\leq & w_{1}(1-\mu_{f}\eta_{1})v_{0}D_{\psi}(x,x^{1})-w_{T}\gamma_{T}(\eta_{T}^{-1}+\mu_{h})v_{T}D_{\psi}(z^{T},x^{T+1})+\sum_{t=1}^{T}2w_{t}\gamma_{t}\eta_{t}v_{t}(M^{2}+\left\Vert \xi^{t}\right\Vert _{*}^{2})\\
 & +\sum_{t=1}^{T}w_{t}\gamma_{t}v_{t-1}\left\langle \xi^{t},z^{t-1}-x^{t}\right\rangle +\sum_{t=2}^{T}(w_{t}-w_{t-1})\gamma_{t}(\eta_{t}^{-1}-\mu_{f})v_{t-1}D_{\psi}(z^{t-1},x^{t})\\
\leq & w_{1}(1-\mu_{f}\eta_{1})v_{0}D_{\psi}(x,x^{1})+\sum_{t=1}^{T}2w_{t}\gamma_{t}\eta_{t}v_{t}(M^{2}+\left\Vert \xi^{t}\right\Vert _{*}^{2})\\
 & +\sum_{t=1}^{T}w_{t}\gamma_{t}v_{t-1}\left\langle \xi^{t},z^{t-1}-x^{t}\right\rangle +\sum_{t=2}^{T}(w_{t}-w_{t-1})\gamma_{t}(\eta_{t}^{-1}-\mu_{f})v_{t-1}D_{\psi}(z^{t-1},x^{t}).
\end{align*}
\end{proof}

\subsection{Proof of Lemma \ref{lem:core-exp}}

\begin{proof}[Proof of Lemma \ref{lem:core-exp}]
We invoke Lemma \ref{lem:core-general} with $w_{t\in\left[T\right]}=1$
to get
\[
\gamma_{T}v_{T}\left(F(x^{T+1})-F(x)\right)\leq(1-\mu_{f}\eta_{1})v_{0}D_{\psi}(x,x^{1})+\sum_{t=1}^{T}2\gamma_{t}\eta_{t}v_{t}(M^{2}+\left\Vert \xi^{t}\right\Vert _{*}^{2})+\sum_{t=1}^{T}\gamma_{t}v_{t-1}\left\langle \xi^{t},z^{t-1}-x^{t}\right\rangle .
\]
Take expectations on both sides to obtain
\begin{align*}
 & \gamma_{T}v_{T}\E\left[F(x^{T+1})-F(x)\right]\\
\leq & (1-\mu_{f}\eta_{1})v_{0}D_{\psi}(x,x^{1})+\sum_{t=1}^{T}2\gamma_{t}\eta_{t}v_{t}(M^{2}+\E\left[\left\Vert \xi^{t}\right\Vert _{*}^{2}\right])+\sum_{t=1}^{T}\gamma_{t}v_{t-1}\E\left[\left\langle \xi^{t},z^{t-1}-x^{t}\right\rangle \right]\\
\leq & (1-\mu_{f}\eta_{1})v_{0}D_{\psi}(x,x^{1})+\sum_{t=1}^{T}2\gamma_{t}\eta_{t}v_{t}(M^{2}+\sigma^{2}),
\end{align*}
where the last line is due to $\E\left[\left\Vert \xi^{t}\right\Vert _{*}^{2}\right]\leq\sigma^{2}$
(Assumption \ref{enu:A5A}) and $\E\left[\left\langle \xi^{t},z^{t-1}-x^{t}\right\rangle \right]=\E\left[\left\langle \E\left[\xi^{t}\vert\F^{t-1}\right],z^{t-1}-x^{t}\right\rangle \right]=0$
($z^{t-1}-x^{t}\in\F^{t-1}$ and Assumption \ref{enu:A4}). Finally,
we divide both sides by $\gamma_{T}v_{T}$ and plug in $v_{t\in\left\{ 0\right\} \cup\left[T\right]}=\frac{w_{T}\gamma_{T}}{\sum_{s=t\lor1}^{T}w_{s}\gamma_{s}}=\frac{\gamma_{T}}{\sum_{s=t\lor1}^{T}\gamma_{s}}$
to finish the proof.
\end{proof}

\subsection{Proof of Lemma \ref{lem:core-hp}}

\begin{proof}[Proof of Lemma \ref{lem:core-hp}]
We invoke Lemma \ref{lem:core-general} to get
\begin{align}
w_{T}\gamma_{T}v_{T}\left(F(x^{T+1})-F(x)\right)\leq & w_{1}(1-\mu_{f}\eta_{1})v_{0}D_{\psi}(x,x^{1})+\sum_{t=1}^{T}2w_{t}\gamma_{t}\eta_{t}v_{t}(M^{2}+\left\Vert \xi^{t}\right\Vert _{*}^{2})\nonumber \\
 & +\sum_{t=1}^{T}w_{t}\gamma_{t}v_{t-1}\left\langle \xi^{t},z^{t-1}-x^{t}\right\rangle +\sum_{t=2}^{T}(w_{t}-w_{t-1})\gamma_{t}(\eta_{t}^{-1}-\mu_{f})v_{t-1}D_{\psi}(z^{t-1},x^{t}).\label{eq:core-hp-1}
\end{align}
Let $w_{t\in\left[T\right]}$ be defined as follows
\begin{equation}
w_{t\in\left[T\right]}\coloneqq\frac{1}{\sum_{s=2}^{t}\frac{2\gamma_{s}\eta_{s}\bar{v}_{s}\sigma^{2}}{1-\mu_{f}\eta_{s}}+\sum_{s=1}^{T}2\gamma_{s}\eta_{s}\bar{v}_{s}\sigma^{2}},\label{eq:core-hp-w}
\end{equation}
where
\begin{equation}
\bar{v}_{t\in\left\{ 0\right\} \cup\left[T\right]}\coloneqq\frac{\gamma_{T}}{\sum_{s=t\lor1}^{T}\gamma_{s}}.\label{eq:core-hp-v-bar}
\end{equation}
Note that $w_{t\in\left[T\right]}\geq0$ is non-increasing, from the
definition of $v_{t\in\left\{ 0\right\} \cup\left[T\right]}=\frac{w_{T}\gamma_{T}}{\sum_{s=t\lor1}^{T}w_{s}\gamma_{s}}$,
there is always
\begin{equation}
v_{t}=\frac{w_{T}\gamma_{T}}{\sum_{s=t\lor1}^{T}w_{s}\gamma_{s}}\leq\frac{\gamma_{T}}{\sum_{s=t\lor1}^{T}\gamma_{s}}=\bar{v}_{t},\forall t\in\left\{ 0\right\} \cup\left[T\right].\label{eq:core-hp-v-v-bar}
\end{equation}

Now we consider the following non-negative sequence with $U_{0}\coloneqq1$
and 
\[
U_{s}\coloneqq\exp\left(\sum_{t=1}^{s}2w_{t}\gamma_{t}\eta_{t}v_{t}\left\Vert \xi^{t}\right\Vert _{*}^{2}-2w_{t}\gamma_{t}\eta_{t}v_{t}\sigma^{2}\right)\in\F^{s},\forall s\in\left[T\right].
\]
We claim $U_{t}$ is a supermartingale by observing that 
\begin{align*}
\E\left[U_{t}\mid\F^{t-1}\right] & =U_{t-1}\E\left[\exp\left(2w_{t}\gamma_{t}\eta_{t}v_{t}\left\Vert \xi^{t}\right\Vert _{*}^{2}-2w_{t}\gamma_{t}\eta_{t}v_{t}\sigma^{2}\right)\mid\F^{t-1}\right]\\
 & \overset{(a)}{\leq}U_{t-1}\exp\left(2w_{t}\gamma_{t}\eta_{t}v_{t}\sigma^{2}-2w_{t}\gamma_{t}\eta_{t}v_{t}\sigma^{2}\right)=U_{t-1},
\end{align*}
where $(a)$ holds due to Assumption \ref{enu:A5B} by noticing 
\[
2w_{t}\gamma_{t}\eta_{t}v_{t}\overset{(\ref{eq:core-hp-w})}{=}\frac{2\gamma_{t}\eta_{t}v_{t}}{\sum_{s=2}^{t}\frac{2\gamma_{s}\eta_{s}\bar{v}_{s}\sigma^{2}}{1-\mu_{f}\eta_{s}}+\sum_{s=1}^{T}2\gamma_{s}\eta_{s}\bar{v}_{s}\sigma^{2}}\leq\frac{v_{t}}{\bar{v}_{t}\sigma^{2}}\overset{(\ref{eq:core-hp-v-v-bar})}{\leq}\frac{1}{\sigma^{2}}.
\]
Hence, we know $\E\left[U_{T}\right]\leq U_{0}=1$. Thus, there is
\begin{align}
 & \Pr\left[U_{T}>\frac{2}{\delta}\right]\overset{(b)}{\leq}\frac{\delta}{2}\E\left[U_{T}\right]\leq\frac{\delta}{2}\nonumber \\
\Rightarrow & \Pr\left[\sum_{t=1}^{T}2w_{t}\gamma_{t}\eta_{t}v_{t}\left\Vert \xi^{t}\right\Vert _{*}^{2}\leq\sum_{t=1}^{T}2w_{t}\gamma_{t}\eta_{t}v_{t}\sigma^{2}+\log\frac{2}{\delta}\right]\geq1-\frac{\delta}{2},\label{eq:core-hp-2}
\end{align}
where we use Markov's inequality in $(b)$.

Next, we consider another non-negative sequence with $R_{0}\coloneqq1$
and 
\[
R_{s}\coloneqq\exp\left(\sum_{t=1}^{s}w_{t}\gamma_{t}v_{t-1}\left\langle \xi^{t},z^{t-1}-x^{t}\right\rangle -w_{t}^{2}\gamma_{t}^{2}v_{t-1}^{2}\sigma^{2}\left\Vert z^{t-1}-x^{t}\right\Vert ^{2}\right)\in\F^{s},\forall s\in\left[T\right].
\]
We prove that $R_{t}$ is also a supermartingale by
\begin{align*}
\E\left[R_{t}\mid\F^{t-1}\right] & =R_{t-1}\E\left[\exp\left(w_{t}\gamma_{t}v_{t-1}\left\langle \xi^{t},z^{t-1}-x^{t}\right\rangle -w_{t}^{2}\gamma_{t}^{2}v_{t-1}^{2}\sigma^{2}\left\Vert z^{t-1}-x^{t}\right\Vert ^{2}\right)\mid\F^{t-1}\right]\\
 & \overset{(c)}{\leq}R_{t-1}\exp\left(w_{t}^{2}\gamma_{t}^{2}v_{t-1}^{2}\sigma^{2}\left\Vert z^{t-1}-x^{t}\right\Vert ^{2}-w_{t}^{2}\gamma_{t}^{2}v_{t-1}^{2}\sigma^{2}\left\Vert z^{t-1}-x^{t}\right\Vert ^{2}\right)=R_{t-1},
\end{align*}
where $(c)$ is by applying Lemma \ref{lem:gaussian} (note that $z^{t-1}-x^{t}\in\F^{t-1}=\sigma(\hg^{s},s\in\left[t-1\right])$).
Hence, we have $\E\left[R_{T}\right]\leq R_{0}=1$, which immediately
implies
\begin{align}
 & \Pr\left[R_{T}>\frac{2}{\delta}\right]\overset{(d)}{\leq}\frac{\delta}{2}\E\left[R_{T}\right]\leq\frac{\delta}{2}\nonumber \\
\Rightarrow & \Pr\left[\sum_{t=1}^{T}w_{t}\gamma_{t}v_{t-1}\left\langle \xi^{t},z^{t-1}-x^{t}\right\rangle \leq\sum_{t=1}^{T}w_{t}^{2}\gamma_{t}^{2}v_{t-1}^{2}\sigma^{2}\left\Vert z^{t-1}-x^{t}\right\Vert ^{2}+\log\frac{2}{\delta}\right]\geq1-\frac{\delta}{2}\nonumber \\
\Rightarrow & \Pr\left[\sum_{t=1}^{T}w_{t}\gamma_{t}v_{t-1}\left\langle \xi^{t},z^{t-1}-x^{t}\right\rangle \leq\sum_{t=1}^{T}2w_{t}^{2}\gamma_{t}^{2}v_{t-1}^{2}\sigma^{2}D_{\psi}(z^{t-1},x^{t})+\log\frac{2}{\delta}\right]\geq1-\frac{\delta}{2},\label{eq:core-hp-3}
\end{align}
where $(d)$ is by Markov's inequality and the last line is due to
$\left\Vert z^{t-1}-x^{t}\right\Vert ^{2}\leq2D_{\psi}(z^{t-1},x^{t})$
from the $1$-strong convexity of $\psi$.

Combining (\ref{eq:core-hp-1}), (\ref{eq:core-hp-2}) and (\ref{eq:core-hp-3}),
with probability at least $1-\delta$, there is
\begin{align*}
w_{T}\gamma_{T}v_{T}\left(F(x^{T+1})-F(x)\right)\leq & w_{1}(1-\mu_{f}\eta_{1})v_{0}D_{\psi}(x,x^{1})+2\log\frac{2}{\delta}+\sum_{t=1}^{T}2w_{t}\gamma_{t}\eta_{t}v_{t}(M^{2}+\sigma^{2})\\
 & +\sum_{t=1}^{T}2w_{t}^{2}\gamma_{t}^{2}v_{t-1}^{2}\sigma^{2}D_{\psi}(z^{t-1},x^{t})+\sum_{t=2}^{T}(w_{t}-w_{t-1})\gamma_{t}(\eta_{t}^{-1}-\mu_{f})v_{t-1}D_{\psi}(z^{t-1},x^{t})\\
= & \left[w_{1}(1-\mu_{f}\eta_{1})v_{0}+2w_{1}^{2}\gamma_{1}^{2}v_{0}^{2}\sigma^{2}\right]D_{\psi}(x,x^{1})+2\log\frac{2}{\delta}+\sum_{t=1}^{T}2w_{t}\gamma_{t}\eta_{t}v_{t}(M^{2}+\sigma^{2})\\
 & +\sum_{t=2}^{T}\left[(w_{t}-w_{t-1})(\eta_{t}^{-1}-\mu_{f})+2w_{t}^{2}\gamma_{t}v_{t-1}\sigma^{2}\right]\gamma_{t}v_{t-1}D_{\psi}(z^{t-1},x^{t}).
\end{align*}
Observe that for $t\geq2$
\begin{align*}
 & (w_{t}-w_{t-1})(\eta_{t}^{-1}-\mu_{f})+2w_{t}^{2}\gamma_{t}v_{t-1}\sigma^{2}\\
\overset{(\ref{eq:core-hp-w})}{=} & 2w_{t}^{2}\gamma_{t}v_{t-1}\sigma^{2}-\left(\frac{1}{\sum_{s=2}^{t-1}\frac{2\gamma_{s}\eta_{s}\bar{v}_{s}\sigma^{2}}{1-\mu_{f}\eta_{s}}+\sum_{s=1}^{T}2\gamma_{s}\eta_{s}\bar{v}_{s}\sigma^{2}}-\frac{1}{\sum_{s=2}^{t}\frac{2\gamma_{s}\eta_{s}\bar{v}_{s}\sigma^{2}}{1-\mu_{f}\eta_{s}}+\sum_{s=1}^{T}2\gamma_{s}\eta_{s}\bar{v}_{s}\sigma^{2}}\right)(\eta_{t}^{-1}-\mu_{f})\\
= & 2w_{t}^{2}\gamma_{t}v_{t-1}\sigma^{2}-w_{t}w_{t-1}\times\frac{2\gamma_{t}\eta_{t}\bar{v}_{t}\sigma^{2}}{1-\mu_{f}\eta_{t}}\times(\eta_{t}^{-1}-\mu_{f})=2w_{t}(w_{t}v_{t-1}-w_{t-1}\bar{v}_{t})\gamma_{t}\sigma^{2}\leq0,
\end{align*}
where the last line holds due to $w_{t}\leq w_{t-1}$ and $v_{t-1}\leq v_{t}\leq\bar{v}_{t}$.
So, we know
\begin{align*}
w_{T}\gamma_{T}v_{T}\left(F(x^{T+1})-F(x)\right) & \leq\left[w_{1}(1-\mu_{f}\eta_{1})v_{0}+2w_{1}^{2}\gamma_{1}^{2}v_{0}^{2}\sigma^{2}\right]D_{\psi}(x,x^{1})+2\log\frac{2}{\delta}+\sum_{t=1}^{T}2w_{t}\gamma_{t}\eta_{t}v_{t}(M^{2}+\sigma^{2})\\
 & \overset{(e)}{\leq}w_{1}\left(1-\mu_{f}\eta_{1}+2w_{1}\gamma_{1}^{2}v_{0}\sigma^{2}\right)v_{0}D_{\psi}(x,x^{1})+2\log\frac{2}{\delta}+w_{1}\sum_{t=1}^{T}2\gamma_{t}\eta_{t}v_{t}(M^{2}+\sigma^{2}),
\end{align*}
where $(e)$ is by $w_{t}\leq w_{1},\forall t\in\left[T\right]$.
Dividing both sides by $w_{T}\gamma_{T}v_{T}$, we get
\begin{align*}
F(x^{T+1})-F(x) & \leq\frac{w_{1}}{w_{T}}\left[(1-\mu_{f}\eta_{1}+2w_{1}\gamma_{1}^{2}v_{0}\sigma^{2})\frac{v_{0}}{\gamma_{T}v_{T}}D_{\psi}(x,x^{1})+\frac{2}{w_{1}\gamma_{T}v_{T}}\log\frac{2}{\delta}+2\sum_{t=1}^{T}\frac{\gamma_{t}\eta_{t}v_{t}}{\gamma_{T}v_{T}}(M^{2}+\sigma^{2})\right]\\
 & \overset{(f)}{\leq}\left(1+\max_{2\leq t\leq T}\frac{1}{1-\mu_{f}\eta_{t}}\right)\left[\frac{(2-\mu_{f}\eta_{1})D_{\psi}(x,x^{1})}{\sum_{t=1}^{T}\gamma_{t}}+2\left(M^{2}+\sigma^{2}\left(1+2\log\frac{2}{\delta}\right)\right)\sum_{t=1}^{T}\frac{\gamma_{t}\eta_{t}}{\sum_{s=t}^{T}\gamma_{s}}\right]\\
 & \leq2\left(1+\max_{2\leq t\leq T}\frac{1}{1-\mu_{f}\eta_{t}}\right)\left[\frac{D_{\psi}(x,x^{1})}{\sum_{t=1}^{T}\gamma_{t}}+\left(M^{2}+\sigma^{2}\left(1+2\log\frac{2}{\delta}\right)\right)\sum_{t=1}^{T}\frac{\gamma_{t}\eta_{t}}{\sum_{s=t}^{T}\gamma_{s}}\right],
\end{align*}
where $(f)$ holds due to the following calculations
\begin{align*}
\frac{w_{1}}{w_{T}} & \overset{(\ref{eq:core-hp-w})}{=}\frac{\sum_{s=1}^{T}2\gamma_{s}\eta_{s}\bar{v}_{s}\sigma^{2}+\sum_{s=2}^{T}\frac{2\gamma_{s}\eta_{s}\bar{v}_{s}\sigma^{2}}{1-\mu_{f}\eta_{s}}}{\sum_{s=1}^{T}2\gamma_{s}\eta_{s}\bar{v}_{s}\sigma^{2}}\leq1+\max_{2\leq t\leq T}\frac{1}{1-\mu_{f}\eta_{t}};\\
2w_{1}\gamma_{1}^{2}v_{0}\sigma^{2} & \overset{(\ref{eq:core-hp-w})}{=}\frac{2\gamma_{1}^{2}v_{0}\sigma^{2}}{\sum_{s=1}^{T}2\gamma_{s}\eta_{s}\bar{v}_{s}\sigma^{2}}\overset{\gamma_{1}=\eta_{1},v_{0}=v_{1}}{=}\frac{2\gamma_{1}\eta_{1}v_{1}\sigma^{2}}{\sum_{s=1}^{T}2\gamma_{s}\eta_{s}\bar{v}_{s}\sigma^{2}}\overset{(\ref{eq:core-hp-v-v-bar})}{\leq}1;\\
\frac{v_{0}}{\gamma_{T}v_{T}} & \overset{(\ref{eq:core-hp-v-v-bar})}{\leq}\frac{\bar{v}_{0}}{\gamma_{T}v_{T}}\overset{(\ref{eq:core-hp-v-bar}),v_{T}=1}{=}\frac{1}{\sum_{t=1}^{T}\gamma_{t}};\\
\frac{2}{w_{1}\gamma_{T}v_{T}}\log\frac{2}{\delta} & \overset{(\ref{eq:core-hp-w})}{=}4\sigma^{2}\log\frac{2}{\delta}\sum_{t=1}^{T}\frac{\gamma_{t}\eta_{t}\bar{v}_{t}}{\gamma_{T}v_{T}}\overset{(\ref{eq:core-hp-v-bar}),v_{T}=1}{=}4\sigma^{2}\log\frac{2}{\delta}\sum_{t=1}^{T}\frac{\gamma_{t}\eta_{t}}{\sum_{s=t}^{T}\gamma_{s}};\\
2\sum_{t=1}^{T}\frac{\gamma_{t}\eta_{t}v_{t}}{\gamma_{T}v_{T}}(M^{2}+\sigma^{2}) & \overset{(\ref{eq:core-hp-v-v-bar})}{\leq}2(M^{2}+\sigma^{2})\sum_{t=1}^{T}\frac{\gamma_{t}\eta_{t}\bar{v}_{t}}{\gamma_{T}v_{T}}\overset{(\ref{eq:core-hp-v-bar}),v_{T}=1}{=}2(M^{2}+\sigma^{2})\sum_{t=1}^{T}\frac{\gamma_{t}\eta_{t}}{\sum_{s=t}^{T}\gamma_{s}}.
\end{align*}
Hence, the proof is completed.
\end{proof}

\section{General Convex Functions\label{sec:cvx}}

In this section, we present the full version of the theorems for general
convex functions (i.e., $\mu_{f}=\mu_{h}=0$) with their proofs.
\begin{thm}[Full version of Theorem \ref{thm:main-cvx-exp}]
\label{thm:cvx-exp}Under Assumptions \ref{enu:A2}-\ref{enu:A4}
and \ref{enu:A5A} with $\mu_{f}=\mu_{h}=0$, for any $x\in\dom$:

If $T$ is unknown, by taking $\eta_{t\in\left[T\right]}=\frac{1}{2L}\land\frac{\eta}{\sqrt{t}}$,
there is
\[
\E\left[F(x^{T+1})-F(x)\right]\leq O\left(\frac{LD_{\psi}(x,x^{1})}{T}+\frac{1}{\sqrt{T}}\left[\frac{D_{\psi}(x,x^{1})}{\eta}+\eta(M^{2}+\sigma^{2})\log T\right]\right).
\]
In particular, by choosing $\eta=\Theta\left(\sqrt{\frac{D_{\psi}(x,x^{1})}{M^{2}+\sigma^{2}}}\right)$,
there is
\[
\E\left[F(x^{T+1})-F(x)\right]\leq O\left(\frac{LD_{\psi}(x,x^{1})}{T}+\frac{(M+\sigma)\sqrt{D_{\psi}(x,x^{1})}\log T}{\sqrt{T}}\right).
\]

If $T$ is known, by taking $\eta_{t\in\left[T\right]}=\frac{1}{2L}\land\frac{\eta}{\sqrt{T}}$,
there is
\[
\E\left[F(x^{T+1})-F(x)\right]\leq O\left(\frac{LD_{\psi}(x,x^{1})}{T}+\frac{1}{\sqrt{T}}\left[\frac{D_{\psi}(x,x^{1})}{\eta}+\eta(M^{2}+\sigma^{2})\log T\right]\right).
\]
In particular, by choosing $\eta=\Theta\left(\sqrt{\frac{D_{\psi}(x,x^{1})}{(M^{2}+\sigma^{2})\log T}}\right)$,
there is
\[
\E\left[F(x^{T+1})-F(x)\right]\leq O\left(\frac{LD_{\psi}(x,x^{1})}{T}+\frac{(M+\sigma)\sqrt{D_{\psi}(x,x^{1})\log T}}{\sqrt{T}}\right).
\]
\end{thm}

\begin{proof}
From Lemma \ref{lem:core-exp}, if $\eta_{t\in\left[T\right]}\leq\frac{1}{2L\lor\mu_{f}}$,
there is
\begin{equation}
\E\left[F(x^{T+1})-F(x)\right]\leq\frac{(1-\mu_{f}\eta_{1})D_{\psi}(x,x^{1})}{\sum_{t=1}^{T}\gamma_{t}}+2(M^{2}+\sigma^{2})\sum_{t=1}^{T}\frac{\gamma_{t}\eta_{t}}{\sum_{s=t}^{T}\gamma_{s}},\label{eq:cvx-exp-1}
\end{equation}
where $\gamma_{t\in\left[T\right]}=\eta_{t}\prod_{s=2}^{t}\frac{1+\mu_{h}\eta_{s-1}}{1-\mu_{f}\eta_{s}}$.
Note that $\mu_{f}=\mu_{h}=0$ now, so both $\eta_{t\in\left[T\right]}=\frac{1}{2L}\land\frac{\eta}{\sqrt{t}}$
and $\eta_{t\in\left[T\right]}=\frac{1}{2L}\land\frac{\eta}{\sqrt{T}}$
satisfy $\eta_{t\in\left[T\right]}\leq\frac{1}{2L\lor\mu_{f}}=\frac{1}{2L}$.
Besides, $\gamma_{t\in\left[T\right]}$ will degenerate to $\eta_{t\in\left[T\right]}$.
Therefore, (\ref{eq:cvx-exp-1}) can be simplified into
\begin{equation}
\E\left[F(x^{T+1})-F(x)\right]\leq\frac{D_{\psi}(x,x^{1})}{\sum_{t=1}^{T}\eta_{t}}+2(M^{2}+\sigma^{2})\sum_{t=1}^{T}\frac{\eta_{t}^{2}}{\sum_{s=t}^{T}\eta_{s}}.\label{eq:cvx-exp-2}
\end{equation}
Before proving convergence rates for these two different step sizes,
we first recall some standard results.
\begin{align}
\sum_{t=1}^{T}\frac{1}{\sqrt{t}} & =\sum_{t=1}^{T}\sqrt{t}-\frac{t-1}{\sqrt{t}}=\sqrt{T}+\sum_{t=1}^{T-1}\sqrt{t}-\frac{t}{\sqrt{t+1}}\geq\sqrt{T};\label{eq:cvx-exp-ineq-1}\\
\sum_{s=t}^{T}\frac{1}{\sqrt{s}} & \geq\int_{t}^{T+1}\frac{1}{\sqrt{s}}\mathrm{d}s=2(\sqrt{T+1}-\sqrt{t}),\forall t\in\left[T\right];\label{eq:cvx-exp-ineq-2}\\
\sum_{t=1}^{T}\frac{1}{t} & \leq1+\int_{1}^{T}\frac{1}{t}\mathrm{d}t=1+\log T.\label{eq:cvx-exp-ineq-3}
\end{align}

If $\eta_{t\in\left[T\right]}=\frac{1}{2L}\land\frac{\eta}{\sqrt{t}}$,
we consider the following three cases:
\begin{itemize}
\item $\eta<\frac{1}{2L}$: In this case, we have $\eta_{t\in\left[T\right]}=\frac{\eta}{\sqrt{t}}$
and
\begin{align}
\E\left[F(x^{T+1})-F(x)\right] & \leq\frac{D_{\psi}(x,x^{1})}{\eta\sum_{t=1}^{T}1/\sqrt{t}}+2\eta(M^{2}+\sigma^{2})\sum_{t=1}^{T}\frac{1}{t\sum_{s=t}^{T}1/\sqrt{s}}\nonumber \\
 & \overset{(\ref{eq:cvx-exp-ineq-1}),(\ref{eq:cvx-exp-ineq-2})}{\leq}\frac{D_{\psi}(x,x^{1})}{\eta\sqrt{T}}+\eta(M^{2}+\sigma^{2})\sum_{t=1}^{T}\frac{1}{t(\sqrt{T+1}-\sqrt{t})}\nonumber \\
 & \overset{(a)}{\leq}\frac{D_{\psi}(x,x^{1})}{\eta\sqrt{T}}+\frac{4\eta(M^{2}+\sigma^{2})(1+\log T)}{\sqrt{T}}\nonumber \\
 & =\frac{1}{\sqrt{T}}\left[\frac{D_{\psi}(x,x^{1})}{\eta}+4\eta(M^{2}+\sigma^{2})(1+\log T)\right],\label{eq:cvx-exp-3}
\end{align}
where $(a)$ is by
\begin{align*}
\sum_{t=1}^{T}\frac{1}{t(\sqrt{T+1}-\sqrt{t})} & =\sum_{t=1}^{T}\frac{\sqrt{T+1}+\sqrt{t}}{t(T+1-t)}\leq\sum_{t=1}^{T}\frac{2\sqrt{T+1}}{t(T+1-t)}\\
 & =\sum_{t=1}^{T}\frac{2}{\sqrt{T+1}}\left(\frac{1}{t}+\frac{1}{T+1-t}\right)=\frac{4}{\sqrt{T+1}}\sum_{t=1}^{T}\frac{1}{t}\overset{(\ref{eq:cvx-exp-ineq-3})}{\leq}\frac{4(1+\log T)}{\sqrt{T}}.
\end{align*}
\item $\eta\geq\frac{\sqrt{T}}{2L}$: In this case, we have $\eta_{t\in\left[T\right]}=\frac{1}{2L}$
and 
\begin{align}
\E\left[F(x^{T+1})-F(x)\right] & \leq\frac{D_{\psi}(x,x^{1})}{T/2L}+\frac{M^{2}+\sigma^{2}}{L}\sum_{t=1}^{T}\frac{1}{T-t+1}=\frac{2LD_{\psi}(x,x^{1})}{T}+\frac{M^{2}+\sigma^{2}}{L}\sum_{t=1}^{T}\frac{1}{t}\nonumber \\
 & \overset{(\ref{eq:cvx-exp-ineq-3})}{\leq}\frac{2LD_{\psi}(x,x^{1})}{T}+\frac{(M^{2}+\sigma^{2})(1+\log T)}{L}\nonumber \\
 & \overset{(b)}{\leq}\frac{2LD_{\psi}(x,x^{1})}{T}+\frac{2\eta(M^{2}+\sigma^{2})(1+\log T)}{\sqrt{T}},\label{eq:cvx-exp-4}
\end{align}
where $(b)$ is by $\frac{1}{L}\leq\frac{2\eta}{\sqrt{T}}$.
\item $\eta\in[\frac{1}{2L},\frac{\sqrt{T}}{2L})$: In this case, we define
$\tau=\lfloor4\eta^{2}L^{2}\rfloor$ where $\lfloor\cdot\rfloor$
is the floor function. Note that 
\[
4\eta^{2}L^{2}\in[1,T)\Rightarrow\tau=\lfloor4\eta^{2}L^{2}\rfloor\in\left[T-1\right].
\]
By observing $\frac{\eta}{\sqrt{t}}\geq\frac{1}{2L}\Leftrightarrow t\in\left[1,\tau\right]$,
we can calculate
\begin{align*}
\E\left[F(x^{T+1})-F(x)\right] & \leq\frac{D_{\psi}(x,x^{1})}{\sum_{t=1}^{T}\eta_{t}}+2(M^{2}+\sigma^{2})\sum_{t=1}^{T}\frac{\eta_{t}^{2}}{\sum_{s=t}^{T}\eta_{s}}\\
 & \overset{(c)}{\leq}\frac{D_{\psi}(x,x^{1})}{T^{2}}\underbrace{\sum_{t=1}^{T}\frac{1}{\eta_{t}}}_{\text{I}}+2(M^{2}+\sigma^{2})\left(\underbrace{\sum_{t=1}^{\tau}\frac{\eta_{t}^{2}}{\sum_{s=t}^{T}\eta_{s}}}_{\text{II}}+\underbrace{\sum_{t=\tau+1}^{T}\frac{\eta_{t}^{2}}{\sum_{s=t}^{T}\eta_{s}}}_{\text{III}}\right),
\end{align*}
where $(c)$ is by $T^{2}\leq\left(\sum_{t=1}^{T}\eta_{t}\right)\left(\sum_{t=1}^{T}\frac{1}{\eta_{t}}\right)$.
Now we bound terms I, II and III as follows
\[
\text{I}=\sum_{t=1}^{T}2L\lor\frac{\sqrt{t}}{\eta}\leq\sum_{t=1}^{T}2L+\frac{\sqrt{t}}{\eta}\leq2LT+\frac{\sqrt{T}+\int_{1}^{T}\sqrt{t}\mathrm{d}t}{\eta}=2LT+\frac{\sqrt{T}+\frac{2}{3}(T^{\frac{3}{2}}-1)}{\eta}\leq2LT+\frac{5T^{\frac{3}{2}}}{3\eta}.
\]
\begin{align*}
\text{II} & =\sum_{t=1}^{\tau}\frac{\eta_{t}^{2}}{\sum_{s=t}^{\tau}\eta_{s}+\sum_{s=\tau+1}^{T}\eta_{s}}=\sum_{t=1}^{\tau}\frac{1/(4L^{2})}{(\tau-t+1)/2L+\sum_{s=\tau+1}^{T}\eta/\sqrt{s}}\\
 & =\frac{1}{2L}\sum_{t=1}^{\tau}\frac{1}{\tau-t+1+\sum_{s=\tau+1}^{T}2\eta L/\sqrt{s}}\overset{(\ref{eq:cvx-exp-ineq-2})}{\leq}\frac{1}{2L}\sum_{t=1}^{\tau}\frac{1}{\tau-t+1+4\eta L(\sqrt{T+1}-\sqrt{\tau+1})}\\
 & \leq\begin{cases}
\frac{1}{2L}\sum_{t=1}^{\tau}\frac{1}{\tau-t+1}\leq\frac{1}{2L}\left(1+\int_{1}^{\tau}\frac{1}{t}\mathrm{d}t\right)=\frac{1+\log\tau}{2L}\overset{(d)}{\leq}\frac{\eta(1+\log T)}{\sqrt{\tau}}\\
\frac{1}{2L}\sum_{t=1}^{\tau}\frac{1}{4\eta L(\sqrt{T+1}-\sqrt{\tau+1})}=\frac{\tau}{8\eta L^{2}(\sqrt{T+1}-\sqrt{\tau+1})}\overset{(e)}{\leq}\frac{\eta}{2(\sqrt{T+1}-\sqrt{\tau+1})}
\end{cases}\\
\Rightarrow\text{II} & \leq\eta(1+\log T)\left(\frac{1}{\sqrt{\tau}}\land\frac{1}{2(\sqrt{T+1}-\sqrt{\tau+1})}\right)\leq\frac{2\eta(1+\log T)}{\sqrt{\tau}+2(\sqrt{T+1}-\sqrt{\tau+1})}\overset{(f)}{\leq}\frac{2\eta(1+\log T)}{\sqrt{T}},
\end{align*}
where $(d)$ is due to $\tau\leq T$ and $\sqrt{\tau}\leq2\eta L$,
$(e)$ holds by $\tau\leq4\eta^{2}L^{2}$ and $(f)$ is by, for $\tau\in\left[T-1\right]$
and $T\geq2$,
\[
\sqrt{\tau}+2(\sqrt{T+1}-\sqrt{\tau+1})\geq\sqrt{T-1}+2\sqrt{T+1}-2\sqrt{T}\geq\sqrt{T}.
\]
\begin{align*}
\text{III} & =\eta\sum_{t=\tau+1}^{T}\frac{1}{t\sum_{s=t}^{T}1/\sqrt{s}}\overset{(\ref{eq:cvx-exp-ineq-2})}{\leq}\eta\sum_{t=\tau+1}^{T}\frac{1}{2t(\sqrt{T+1}-\sqrt{t})}=\eta\sum_{t=\tau+1}^{T}\frac{\sqrt{T+1}+\sqrt{t}}{2t(T+1-t)}\\
 & \leq\eta\sum_{t=\tau+1}^{T}\frac{\sqrt{T+1}}{t(T+1-t)}=\eta\sum_{t=\tau+1}^{T}\frac{1}{\sqrt{T+1}}\left(\frac{1}{t}+\frac{1}{T+1-t}\right)\leq\frac{2\eta}{\sqrt{T+1}}\sum_{t=1}^{T}\frac{1}{t}\overset{(\ref{eq:cvx-exp-ineq-3})}{\leq}\frac{2\eta(1+\log T)}{\sqrt{T}}.
\end{align*}
Thus, we have
\begin{align}
\E\left[F(x^{T+1})-F(x)\right] & \leq\frac{D_{\psi}(x,x^{1})}{T^{2}}\left(2LT+\frac{5T^{\frac{3}{2}}}{3\eta}\right)+2(M^{2}+\sigma^{2})\left[\frac{2\eta(1+\log T)}{\sqrt{T}}+\frac{2\eta(1+\log T)}{\sqrt{T}}\right]\nonumber \\
 & \leq\frac{2LD_{\psi}(x,x^{1})}{T}+\frac{1}{\sqrt{T}}\left(\frac{5D_{\psi}(x,x^{1})}{3\eta}+8\eta(M^{2}+\sigma^{2})(1+\log T)\right).\label{eq:cvx-exp-5}
\end{align}
\end{itemize}
Combining (\ref{eq:cvx-exp-3}), (\ref{eq:cvx-exp-4}) and (\ref{eq:cvx-exp-5}),
we know
\begin{align}
\E\left[F(x^{T+1})-F(x)\right]\leq & \frac{1}{\sqrt{T}}\left[\frac{D_{\psi}(x,x^{1})}{\eta}+4\eta(M^{2}+\sigma^{2})(1+\log T)\right]\lor\left[\frac{2LD_{\psi}(x,x^{1})}{T}+\frac{2\eta(M^{2}+\sigma^{2})(1+\log T)}{\sqrt{T}}\right]\nonumber \\
 & \lor\left[\frac{2LD_{\psi}(x,x^{1})}{T}+\frac{1}{\sqrt{T}}\left(\frac{5D_{\psi}(x,x^{1})}{3\eta}+8\eta(M^{2}+\sigma^{2})(1+\log T)\right)\right]\nonumber \\
\leq & \frac{2LD_{\psi}(x,x^{1})}{T}+\frac{1}{\sqrt{T}}\left(\frac{5D_{\psi}(x,x^{1})}{3\eta}+8\eta(M^{2}+\sigma^{2})(1+\log T)\right)\nonumber \\
= & O\left(\frac{LD_{\psi}(x,x^{1})}{T}+\frac{1}{\sqrt{T}}\left[\frac{D_{\psi}(x,x^{1})}{\eta}+\eta(M^{2}+\sigma^{2})\log T\right]\right).\label{eq:cvx-exp-6}
\end{align}
By plugging in $\eta=\Theta\left(\sqrt{\frac{D_{\psi}(x,x^{1})}{M^{2}+\sigma^{2}}}\right)$,
we get the desired bound.

If $\eta_{t\in\left[T\right]}=\frac{1}{2L}\land\frac{\eta}{\sqrt{T}}$,
we will obtain
\begin{align}
\E\left[F(x^{T+1})-F(x)\right] & \leq\frac{D_{\psi}(x,x^{1})}{T}\left(2L\lor\frac{\sqrt{T}}{\eta}\right)+2\left(\frac{1}{2L}\land\frac{\eta}{\sqrt{T}}\right)(M^{2}+\sigma^{2})\sum_{t=1}^{T}\frac{1}{T-t+1}\nonumber \\
 & =\frac{D_{\psi}(x,x^{1})}{T}\left(2L\lor\frac{\sqrt{T}}{\eta}\right)+2\left(\frac{1}{2L}\land\frac{\eta}{\sqrt{T}}\right)(M^{2}+\sigma^{2})\sum_{t=1}^{T}\frac{1}{t}\nonumber \\
 & \overset{(\ref{eq:cvx-exp-ineq-3})}{\leq}\frac{D_{\psi}(x,x^{1})}{T}\left(2L\lor\frac{\sqrt{T}}{\eta}\right)+2\left(\frac{1}{2L}\land\frac{\eta}{\sqrt{T}}\right)(M^{2}+\sigma^{2})(1+\log T)\nonumber \\
 & \leq\frac{2LD_{\psi}(x,x^{1})}{T}+\frac{1}{\sqrt{T}}\left[\frac{D_{\psi}(x,x^{1})}{\eta}+2\eta(M^{2}+\sigma^{2})(1+\log T)\right]\nonumber \\
 & =O\left(\frac{LD_{\psi}(x,x^{1})}{T}+\frac{1}{\sqrt{T}}\left[\frac{D_{\psi}(x,x^{1})}{\eta}+\eta(M^{2}+\sigma^{2})\log T\right]\right).\label{eq:cvx-exp-7}
\end{align}
By plugging in $\eta=\Theta\left(\sqrt{\frac{D_{\psi}(x,x^{1})}{(M^{2}+\sigma^{2})\log T}}\right)$,
we get the desired bound.
\end{proof}

\begin{thm}[Full version of Theorem \ref{thm:main-cvx-hp}]
\label{thm:cvx-hp}Under Assumptions \ref{enu:A2}-\ref{enu:A4}
and \ref{enu:A5B} with $\mu_{f}=\mu_{h}=0$ and let $\delta\in(0,1)$,
for any $x\in\dom$:

If $T$ is unknown, by taking $\eta_{t\in\left[T\right]}=\frac{1}{2L}\land\frac{\eta}{\sqrt{t}}$,
then with probability at least $1-\delta$, there is
\[
F(x^{T+1})-F(x)\leq O\left(\frac{LD_{\psi}(x,x^{1})}{T}+\frac{1}{\sqrt{T}}\left[\frac{D_{\psi}(x,x^{1})}{\eta}+\eta\left(M^{2}+\sigma^{2}\log\frac{1}{\delta}\right)\log T\right]\right).
\]
In particular, by choosing $\eta=\Theta\left(\sqrt{\frac{D_{\psi}(x,x^{1})}{M^{2}+\sigma^{2}\log\frac{1}{\delta}}}\right)$,
there is
\[
F(x^{T+1})-F(x)\leq O\left(\frac{LD_{\psi}(x,x^{1})}{T}+\frac{(M+\sigma\sqrt{\log\frac{1}{\delta}})\sqrt{D_{\psi}(x,x^{1})}\log T}{\sqrt{T}}\right).
\]

If $T$ is known, by taking $\eta_{t\in\left[T\right]}=\frac{1}{2L}\land\frac{\eta}{\sqrt{T}}$,
then with probability at least $1-\delta$, there is
\[
F(x^{T+1})-F(x)\leq O\left(\frac{LD_{\psi}(x,x^{1})}{T}+\frac{1}{\sqrt{T}}\left[\frac{D_{\psi}(x,x^{1})}{\eta}+\eta\left(M^{2}+\sigma^{2}\log\frac{1}{\delta}\right)\log T\right]\right).
\]
In particular, by choosing $\eta=\Theta\left(\sqrt{\frac{D_{\psi}(x,x^{1})}{(M^{2}+\sigma^{2}\log\frac{1}{\delta})\log T}}\right)$,
there is
\[
F(x^{T+1})-F(x)\leq O\left(\frac{LD_{\psi}(x,x^{1})}{T}+\frac{(M+\sigma\sqrt{\log\frac{1}{\delta}})\sqrt{D_{\psi}(x,x^{1})\log T}}{\sqrt{T}}\right).
\]
\end{thm}

\begin{proof}
From Lemma \ref{lem:core-hp}, if $\eta_{t\in\left[T\right]}\leq\frac{1}{2L\lor\mu_{f}}$,
with probability at least $1-\delta$, there is
\begin{equation}
F(x^{T+1})-F(x)\leq2\left(1+\max_{2\leq t\leq T}\frac{1}{1-\mu_{f}\eta_{t}}\right)\left[\frac{D_{\psi}(x,x^{1})}{\sum_{t=1}^{T}\gamma_{t}}+\left(M^{2}+\sigma^{2}\left(1+2\log\frac{2}{\delta}\right)\right)\sum_{t=1}^{T}\frac{\gamma_{t}\eta_{t}}{\sum_{s=t}^{T}\gamma_{s}}\right],\label{eq:cvx-hp-1}
\end{equation}
where $\gamma_{t\in\left[T\right]}=\eta_{t}\prod_{s=2}^{t}\frac{1+\mu_{h}\eta_{s-1}}{1-\mu_{f}\eta_{s}}$.
Note that $\mu_{f}=\mu_{h}=0$ now, so both $\eta_{t\in\left[T\right]}=\frac{1}{2L}\land\frac{\eta}{\sqrt{t}}$
and $\eta_{t\in\left[T\right]}=\frac{1}{2L}\land\frac{\eta}{\sqrt{T}}$
satisfy $\eta_{t\in\left[T\right]}\leq\frac{1}{2L\lor\mu_{f}}=\frac{1}{2L}$.
Besides, $\gamma_{t\in\left[T\right]}$ will degenerate to $\eta_{t\in\left[T\right]}$.
Then we can simplify (\ref{eq:cvx-hp-1}) into
\begin{equation}
F(x^{T+1})-F(x)\leq\frac{4D_{\psi}(x,x^{1})}{\sum_{t=1}^{T}\eta_{t}}+4\left(M^{2}+\sigma^{2}\left(1+2\log\frac{2}{\delta}\right)\right)\sum_{t=1}^{T}\frac{\eta_{t}^{2}}{\sum_{s=t}^{T}\eta_{s}}.\label{eq:cvx-hp-2}
\end{equation}

If $\eta_{t\in\left[T\right]}=\frac{1}{2L}\land\frac{\eta}{\sqrt{t}}$,
similar to (\ref{eq:cvx-exp-6}), we will have
\[
F(x^{T+1})-F(x)\leq O\left(\frac{LD_{\psi}(x,x^{1})}{T}+\frac{1}{\sqrt{T}}\left[\frac{D_{\psi}(x,x^{1})}{\eta}+\eta\left(M^{2}+\sigma^{2}\log\frac{1}{\delta}\right)\log T\right]\right).
\]
By plugging in $\eta=\Theta\left(\sqrt{\frac{D_{\psi}(x,x^{1})}{M^{2}+\sigma^{2}\log\frac{1}{\delta}}}\right)$,
we get the desired bound.

If $\eta_{t\in\left[T\right]}=\frac{1}{2L}\land\frac{\eta}{\sqrt{T}}$,
similar to (\ref{eq:cvx-exp-7}), we will get
\[
F(x^{T+1})-F(x)\leq O\left(\frac{LD_{\psi}(x,x^{1})}{T}+\frac{1}{\sqrt{T}}\left[\frac{D_{\psi}(x,x^{1})}{\eta}+\eta\left(M^{2}+\sigma^{2}\log\frac{1}{\delta}\right)\log T\right]\right).
\]
By plugging in $\eta=\Theta\left(\sqrt{\frac{D_{\psi}(x,x^{1})}{(M^{2}+\sigma^{2}\log\frac{1}{\delta})\log T}}\right)$,
we get the desired bound.
\end{proof}

\subsection{Optimal Rates via the Linearly Decaying Step Size\label{subsec:cvx-optimal}}

In this section, we present the full version of Theorems \ref{thm:main-cvx-exp-optimal}
and \ref{thm:main-cvx-hp-optimal}.
\begin{thm}[Full version of Theorem \ref{thm:main-cvx-exp-optimal}]
\label{thm:cvx-exp-optimal}Under Assumptions \ref{enu:A2}-\ref{enu:A4}
and \ref{enu:A5A} with $\mu_{f}=\mu_{h}=0$, for any $x\in\dom$,
if $T$ is known, by taking $\eta_{t\in\left[T\right]}=\frac{T-t+1}{2LT}\land\frac{\eta(T-t+1)}{T^{\frac{3}{2}}}$,
there is
\[
\E\left[F(x^{T+1})-F(x)\right]\leq O\left(\frac{LD_{\psi}(x,x^{1})}{T}+\frac{1}{\sqrt{T}}\left[\frac{D_{\psi}(x,x^{1})}{\eta}+\eta(M^{2}+\sigma^{2})\right]\right).
\]
In particular, by choosing $\eta=\Theta\left(\sqrt{\frac{D_{\psi}(x,x^{1})}{M^{2}+\sigma^{2}}}\right)$,
there is
\[
\E\left[F(x^{T+1})-F(x)\right]\leq O\left(\frac{LD_{\psi}(x,x^{1})}{T}+\frac{(M+\sigma)\sqrt{D_{\psi}(x,x^{1})}}{\sqrt{T}}\right).
\]
\end{thm}

\begin{proof}
From Lemma \ref{lem:core-exp}, if $\eta_{t\in\left[T\right]}\leq\frac{1}{2L\lor\mu_{f}}$,
there is
\begin{equation}
\E\left[F(x^{T+1})-F(x)\right]\leq\frac{(1-\mu_{f}\eta_{1})D_{\psi}(x,x^{1})}{\sum_{t=1}^{T}\gamma_{t}}+2(M^{2}+\sigma^{2})\sum_{t=1}^{T}\frac{\gamma_{t}\eta_{t}}{\sum_{s=t}^{T}\gamma_{s}},\label{eq:cvx-exp-optimal-1}
\end{equation}
where $\gamma_{t\in\left[T\right]}=\eta_{t}\prod_{s=2}^{t}\frac{1+\mu_{h}\eta_{s-1}}{1-\mu_{f}\eta_{s}}$.
Note that $\mu_{f}=\mu_{h}=0$ now, so $\eta_{t\in\left[T\right]}=\frac{T-t+1}{2LT}\land\frac{\eta(T-t+1)}{T^{\frac{3}{2}}}$
satisfies $\eta_{t\in\left[T\right]}\leq\frac{1}{2L\lor\mu_{f}}=\frac{1}{2L}$.
Besides, $\gamma_{t\in\left[T\right]}$ will degenerate to $\eta_{t\in\left[T\right]}$.
Therefore, (\ref{eq:cvx-exp-optimal-1}) can be simplified as follows
\begin{equation}
\E\left[F(x^{T+1})-F(x)\right]\leq\frac{D_{\psi}(x,x^{1})}{\sum_{t=1}^{T}\eta_{t}}+2(M^{2}+\sigma^{2})\sum_{t=1}^{T}\frac{\eta_{t}^{2}}{\sum_{s=t}^{T}\eta_{s}}.\label{eq:cvx-exp-optimal-2}
\end{equation}
Plugging $\eta_{t\in\left[T\right]}=\frac{T-t+1}{2LT}\land\frac{\eta(T-t+1)}{T^{\frac{3}{2}}}$
into (\ref{eq:cvx-exp-optimal-2}), we have
\begin{align}
\E\left[F(x^{T+1})-F(x)\right] & \leq\frac{D_{\psi}(x,x^{1})}{\sum_{t=1}^{T}T-t+1}\left(2LT\lor\frac{T^{\frac{3}{2}}}{\eta}\right)+\frac{2\eta(M^{2}+\sigma^{2})}{T^{\frac{3}{2}}}\sum_{t=1}^{T}\frac{(T-t+1)^{2}}{\sum_{s=t}^{T}T-s+1}\nonumber \\
 & \leq\frac{4LD_{\psi}(x,x^{1})}{T}+\frac{2D_{\psi}(x,x^{1})}{\eta\sqrt{T}}+\frac{4\eta(M^{2}+\sigma^{2})}{\sqrt{T}}\nonumber \\
 & =O\left(\frac{LD_{\psi}(x,x^{1})}{T}+\frac{1}{\sqrt{T}}\left[\frac{D_{\psi}(x,x^{1})}{\eta}+\eta(M^{2}+\sigma^{2})\right]\right).\label{eq:cvx-exp-optimal-3}
\end{align}
By plugging in $\eta=\Theta\left(\sqrt{\frac{D_{\psi}(x,x^{1})}{M^{2}+\sigma^{2}}}\right)$,
we get the desired bound.
\end{proof}

\begin{thm}[Full version of Theorem \ref{thm:main-cvx-hp-optimal}]
\label{thm:cvx-hp-optimal}Under Assumptions \ref{enu:A2}-\ref{enu:A4}
and \ref{enu:A5B} with $\mu_{f}=\mu_{h}=0$ and let $\delta\in(0,1)$,
for any $x\in\dom$, if $T$ is known, by taking $\eta_{t\in\left[T\right]}=\frac{T-t+1}{2LT}\land\frac{\eta(T-t+1)}{T^{\frac{3}{2}}}$,
then with probability at least $1-\delta$, there is
\[
F(x^{T+1})-F(x)\leq O\left(\frac{LD_{\psi}(x,x^{1})}{T}+\frac{1}{\sqrt{T}}\left[\frac{D_{\psi}(x,x^{1})}{\eta}+\eta(M^{2}+\sigma^{2}\log\frac{1}{\delta})\right]\right).
\]
In particular, by choosing $\eta=\Theta\left(\sqrt{\frac{D_{\psi}(x,x^{1})}{M^{2}+\sigma^{2}\log\frac{1}{\delta}}}\right)$,
there is
\[
F(x^{T+1})-F(x)\leq O\left(\frac{LD_{\psi}(x,x^{1})}{T}+\frac{(M+\sigma\sqrt{\log\frac{1}{\delta}})\sqrt{D_{\psi}(x,x^{1})}}{\sqrt{T}}\right).
\]
\end{thm}

\begin{proof}
From Lemma \ref{lem:core-hp}, if $\eta_{t\in\left[T\right]}\leq\frac{1}{2L\lor\mu_{f}}$,
with probability at least $1-\delta$, there is
\begin{equation}
F(x^{T+1})-F(x)\leq2\left(1+\max_{2\leq t\leq T}\frac{1}{1-\mu_{f}\eta_{t}}\right)\left[\frac{D_{\psi}(x,x^{1})}{\sum_{t=1}^{T}\gamma_{t}}+\left(M^{2}+\sigma^{2}\left(1+2\log\frac{2}{\delta}\right)\right)\sum_{t=1}^{T}\frac{\gamma_{t}\eta_{t}}{\sum_{s=t}^{T}\gamma_{s}}\right],\label{eq:cvx-hp-optimal-1}
\end{equation}
where $\gamma_{t\in\left[T\right]}=\eta_{t}\prod_{s=2}^{t}\frac{1+\mu_{h}\eta_{s-1}}{1-\mu_{f}\eta_{s}}$.
Note that $\mu_{f}=\mu_{h}=0$ now, hence, $\eta_{t\in\left[T\right]}=\frac{T-t+1}{2LT}\land\frac{\eta(T-t+1)}{T^{\frac{3}{2}}}$
satisfies $\eta_{t\in\left[T\right]}\leq\frac{1}{2L\lor\mu_{f}}=\frac{1}{2L}$.
Besides, $\gamma_{t\in\left[T\right]}$ will degenerate to $\eta_{t\in\left[T\right]}$.
Then we can simplify (\ref{eq:cvx-hp-optimal-1}) into
\[
F(x^{T+1})-F(x)\leq\frac{4D_{\psi}(x,x^{1})}{\sum_{t=1}^{T}\eta_{t}}+4\left(M^{2}+\sigma^{2}\left(1+2\log\frac{2}{\delta}\right)\right)\sum_{t=1}^{T}\frac{\eta_{t}^{2}}{\sum_{s=t}^{T}\eta_{s}}.
\]
Similar to (\ref{eq:cvx-exp-optimal-3}), we will have
\[
F(x^{T+1})-F(x)\leq O\left(\frac{LD_{\psi}(x,x^{1})}{T}+\frac{1}{\sqrt{T}}\left[\frac{D_{\psi}(x,x^{1})}{\eta}+\eta\left(M^{2}+\sigma^{2}\log\frac{1}{\delta}\right)\right]\right).
\]
By plugging in $\eta=\Theta\left(\sqrt{\frac{D_{\psi}(x,x^{1})}{M^{2}+\sigma^{2}\log\frac{1}{\delta}}}\right)$,
we get the desired bound.
\end{proof}

\section{Strongly Convex Functions\label{sec:str}}

In this section, we present the full version of the theorems for strongly
convex functions (i.e., $\mu_{f}+\mu_{h}>0$) with their proofs.

\subsection{The Case of $\mu_{f}>0$\label{subsec:str-f}}
\begin{thm}[Full version of Theorem \ref{thm:main-str-exp}]
\label{thm:str-f-exp}Under Assumptions \ref{enu:A2}-\ref{enu:A4}
and \ref{enu:A5A} with $\mu_{f}>0$ and $\mu_{h}=0$, let $\kappa_{f}\coloneqq\frac{L}{\mu_{f}}\geq0$,
for any $x\in\dom$:

If $T$ is unknown, by taking either $\eta_{t\in\left[T\right]}=\frac{1}{\mu_{f}(t+2\kappa_{f})}$
or $\eta_{t\in\left[T\right]}=\frac{2}{\mu_{f}(t+1+4\kappa_{f})}$,
there is
\[
\E\left[F(x^{T+1})-F(x)\right]\leq\begin{cases}
O\left(\frac{LD_{\psi}(x,x^{1})}{T}+\frac{(M^{2}+\sigma^{2})\log T}{\mu_{f}(T+\kappa_{f})}\right) & \eta_{t\in\left[T\right]}=\frac{1}{\mu_{f}(t+2\kappa_{f})}\\
O\left(\frac{L(1+\kappa_{f})D_{\psi}(x,x^{1})}{T(T+\kappa_{f})}+\frac{(M^{2}+\sigma^{2})\log T}{\mu_{f}(T+\kappa_{f})}\right) & \eta_{t\in\left[T\right]}=\frac{2}{\mu_{f}(t+1+4\kappa_{f})}
\end{cases}.
\]

If $T$ is known, by taking $\eta_{t\in\left[T\right]}=\begin{cases}
\frac{1}{\mu_{f}(1+2\kappa_{f})} & t\leq\tau\\
\frac{2}{\mu_{f}(t-\tau+2+4\kappa_{f})} & t\geq\tau+1
\end{cases}$ with $\tau\coloneqq\left\lceil \frac{T}{2}\right\rceil $, there
is
\[
\E\left[F(x^{T+1})-F(x)\right]\leq O\left(\frac{LD_{\psi}(x,x^{1})}{\exp\left(\frac{T}{2+4\kappa_{f}}\right)}+\frac{(M^{2}+\sigma^{2})\log T}{\mu_{f}(T+\kappa_{f})}\right).
\]
\end{thm}

\begin{proof}
When $\mu_{f}>0$ and $\mu_{h}=0$, suppose the condition of $\eta_{t\in\left[T\right]}\leq\frac{1}{2L\lor\mu_{f}}$
in Lemma \ref{lem:core-exp} holds, we have
\begin{equation}
\E\left[F(x^{T+1})-F(x)\right]\leq\frac{(1-\mu_{f}\eta_{1})D_{\psi}(x,x^{1})}{\sum_{t=1}^{T}\gamma_{t}}+2(M^{2}+\sigma^{2})\sum_{t=1}^{T}\frac{\gamma_{t}\eta_{t}}{\sum_{s=t}^{T}\gamma_{s}},\label{eq:str-f-exp-general}
\end{equation}
where $\gamma_{t\in\left[T\right]}=\eta_{t}\prod_{s=2}^{t}\frac{1+\mu_{h}\eta_{s-1}}{1-\mu_{f}\eta_{s}}=\frac{\eta_{t}}{\prod_{s=2}^{t}(1-\mu_{f}\eta_{s})}$.
Now, let us check the condition of $\eta_{t\in\left[T\right]}\leq\frac{1}{2L\lor\mu_{f}}$
for the three choices respectively:
\begin{align*}
\eta_{t\in\left[T\right]} & =\frac{1}{\mu_{f}(t+2\kappa_{f})}\leq\frac{1}{\mu_{f}+2L}\leq\frac{1}{2L\lor\mu_{f}};\\
\eta_{t\in\left[T\right]} & =\frac{2}{\mu_{f}(t+1+4\kappa_{f})}\leq\frac{1}{\mu_{f}+2L}\leq\frac{1}{2L\lor\mu_{f}};\\
\eta_{t\in\left[T\right]} & =\begin{cases}
\frac{1}{\mu_{f}(1+2\kappa_{f})}=\frac{1}{\mu_{f}+2L}\leq\frac{1}{2L\lor\mu_{f}} & t\leq\tau\\
\frac{2}{\mu_{f}(t-\tau+2+4\kappa_{f})}\leq\frac{1}{\mu_{f}+2L}\leq\frac{1}{2L\lor\mu_{f}} & t\geq\tau+1
\end{cases}.
\end{align*}
Therefore, (\ref{eq:str-f-exp-general}) holds for all cases.

First, we consider $\eta_{t\in\left[T\right]}=\frac{1}{\mu_{f}(t+2\kappa_{f})}$.
We can find $\eta_{1}=\frac{1}{\mu_{f}(1+2\kappa_{f})}$ and
\begin{equation}
\gamma_{t\in\left[T\right]}=\frac{\eta_{t}}{\prod_{s=2}^{t}(1-\mu_{f}\eta_{s})}=\frac{1}{\mu_{f}(t+2\kappa_{f})\prod_{s=2}^{t}(1-\frac{1}{s+2\kappa_{f}})}=\frac{1}{\mu_{f}(1+2\kappa_{f})}=\eta_{1}.\label{eq:str-f-exp-gamma-1}
\end{equation}
Hence, using (\ref{eq:str-f-exp-general}) and (\ref{eq:str-f-exp-gamma-1}),
we have
\begin{align}
\E\left[F(x^{T+1})-F(x)\right] & \leq\frac{(\eta_{1}^{-1}-\mu_{f})D_{\psi}(x,x^{1})}{T}+2(M^{2}+\sigma^{2})\sum_{t=1}^{T}\frac{\eta_{t}}{T-t+1}\nonumber \\
 & =\frac{2LD_{\psi}(x,x^{1})}{T}+\frac{2(M^{2}+\sigma^{2})}{\mu_{f}}\sum_{t=1}^{T}\frac{1}{(T-t+1)(t+2\kappa_{f})}\nonumber \\
 & =\frac{2LD_{\psi}(x,x^{1})}{T}+\frac{2(M^{2}+\sigma^{2})}{\mu_{f}(T+1+2\kappa_{f})}\sum_{t=1}^{T}\frac{1}{T-t+1}+\frac{1}{t+2\kappa_{f}}\nonumber \\
 & \leq\frac{2LD_{\psi}(x,x^{1})}{T}+\frac{2(M^{2}+\sigma^{2})}{\mu_{f}(T+1+2\kappa_{f})}\left(1+\log T+\frac{1}{1+2\kappa_{f}}+\log\frac{T+2\kappa_{f}}{1+2\kappa_{f}}\right)\nonumber \\
 & \leq\frac{2LD_{\psi}(x,x^{1})}{T}+\frac{4(M^{2}+\sigma^{2})(1+\log T)}{\mu_{f}(T+1+2\kappa_{f})}\nonumber \\
 & =O\left(\frac{LD_{\psi}(x,x^{1})}{T}+\frac{(M^{2}+\sigma^{2})\log T}{\mu_{f}(T+\kappa_{f})}\right).\label{eq:str-f-exp-1}
\end{align}

Next, for the case of $\eta_{t\in\left[T\right]}=\frac{2}{\mu_{f}(t+1+4\kappa_{f})}$,
there are $\eta_{1}=\frac{1}{\mu_{f}(1+2\kappa_{f})}$ and
\begin{align}
\gamma_{t\in\left[T\right]} & =\frac{\eta_{t}}{\prod_{s=2}^{t}(1-\mu_{f}\eta_{s})}=\frac{2}{\mu_{f}(t+1+4\kappa_{f})\prod_{s=2}^{t}(1-\frac{2}{s+1+4\kappa_{f}})}\nonumber \\
 & =\frac{t+4\kappa_{f}}{\mu_{f}(1+2\kappa_{f})(1+4\kappa_{f})}=\frac{t+4\kappa_{f}}{1+4\kappa_{f}}\eta_{1}.\label{eq:str-f-exp-gamma-2}
\end{align}
Hence, using (\ref{eq:str-f-exp-general}) and (\ref{eq:str-f-exp-gamma-2}),
we have
\begin{align}
\E\left[F(x^{T+1})-F(x)\right] & \leq\frac{(1+4\kappa_{f})(\eta_{1}^{-1}-\mu_{f})D_{\psi}(x,x^{1})}{\sum_{t=1}^{T}t+4\kappa_{f}}+2(M^{2}+\sigma^{2})\sum_{t=1}^{T}\frac{(t+4\kappa_{f})\eta_{t}}{\sum_{s=t}^{T}s+4\kappa_{f}}\nonumber \\
 & =\frac{4(1+4\kappa_{f})LD_{\psi}(x,x^{1})}{T(T+1+8\kappa_{f})}+\frac{8(M^{2}+\sigma^{2})}{\mu_{f}}\sum_{t=1}^{T}\frac{t+4\kappa_{f}}{(t+1+4\kappa_{f})(T-t+1)(T+t+8\kappa_{f})}\nonumber \\
 & \leq\frac{4(1+4\kappa_{f})LD_{\psi}(x,x^{1})}{T(T+1+8\kappa_{f})}+\frac{8(M^{2}+\sigma^{2})}{\mu_{f}}\sum_{t=1}^{T}\frac{1}{(T-t+1)(T+t+8\kappa_{f})}\nonumber \\
 & =\frac{4(1+4\kappa_{f})LD_{\psi}(x,x^{1})}{T(T+1+8\kappa_{f})}+\frac{8(M^{2}+\sigma^{2})}{\mu_{f}(2T+1+8\kappa_{f})}\sum_{t=1}^{T}\frac{1}{T-t+1}+\frac{1}{T+t+8\kappa_{f}}\nonumber \\
 & \leq\frac{4(1+4\kappa_{f})LD_{\psi}(x,x^{1})}{T(T+1+8\kappa_{f})}+\frac{8(M^{2}+\sigma^{2})}{\mu_{f}(2T+1+8\kappa_{f})}\left(1+\log T+\log\frac{2T+8\kappa_{f}}{T+8\kappa_{f}}\right)\nonumber \\
 & \leq\frac{4(1+4\kappa_{f})LD_{\psi}(x,x^{1})}{T(T+1+8\kappa_{f})}+\frac{8(M^{2}+\sigma^{2})(1+\log2T)}{\mu_{f}(2T+1+8\kappa_{f})}\nonumber \\
 & =O\left(\frac{L(1+\kappa_{f})D_{\psi}(x,x^{1})}{T(T+\kappa_{f})}+\frac{(M^{2}+\sigma^{2})\log T}{\mu_{f}(T+\kappa_{f})}\right).\label{eq:str-f-exp-2}
\end{align}

Finally, if $T$ is known, recall that we choose
\[
\eta_{t\in\left[T\right]}=\begin{cases}
\frac{1}{\mu_{f}(1+2\kappa_{f})} & t\leq\tau\\
\frac{2}{\mu_{f}(t-\tau+2+4\kappa_{f})} & t\geq\tau+1
\end{cases},
\]
where $\tau=\left\lceil \frac{T}{2}\right\rceil $. Note that we have
$\eta_{1}=\frac{1}{\mu_{f}(1+2\kappa_{f})}$ and
\begin{equation}
\gamma_{t\in\left[T\right]}=\frac{\eta_{t}}{\prod_{s=2}^{t}(1-\mu_{f}\eta_{s})}=\begin{cases}
\frac{\eta_{1}}{(1-\mu_{f}\eta_{1})^{t-1}} & t\leq\tau\\
\frac{\eta_{1}(t-\tau+1+4\kappa_{f})}{(1-\mu_{f}\eta_{1})^{\tau-1}(1+4\kappa_{f})} & t\geq\tau+1
\end{cases}.\label{eq:str-f-exp-gamma-3}
\end{equation}
So, there is
\begin{align}
\frac{1-\mu_{f}\eta_{1}}{\sum_{t=1}^{T}\gamma_{t}} & \leq\frac{1-\mu_{f}\eta_{1}}{\sum_{t=1}^{\tau}\gamma_{t}}\overset{(\ref{eq:str-f-exp-gamma-3})}{=}\frac{1-\mu_{f}\eta_{1}}{\sum_{t=1}^{\tau}\frac{\eta_{1}}{(1-\mu_{f}\eta_{1})^{t-1}}}=\frac{\mu_{f}(1-\mu_{f}\eta_{1})^{\tau}}{1-(1-\mu_{f}\eta_{1})^{\tau}}\nonumber \\
 & \leq\frac{(1-\mu_{f}\eta_{1})^{\tau}}{\eta_{1}}\leq(\eta_{1}^{-1}-\mu_{f})\exp\left(-\mu_{f}\eta_{1}(\tau-1)\right)\nonumber \\
 & =2L\exp\left(-\frac{\tau-1}{1+2\kappa_{f}}\right)\leq2eL\exp\left(-\frac{T}{2+4\kappa_{f}}\right).\label{eq:str-f-exp-3}
\end{align}
Now, we observe that
\[
\sum_{t=1}^{T}\frac{\gamma_{t}\eta_{t}}{\sum_{s=t}^{T}\gamma_{s}}=\sum_{t=1}^{\tau}\frac{\gamma_{t}\eta_{t}}{\sum_{s=t}^{T}\gamma_{s}}+\sum_{t=\tau+1}^{T}\frac{\gamma_{t}\eta_{t}}{\sum_{s=t}^{T}\gamma_{s}}=\underbrace{\sum_{t=1}^{\tau}\frac{\gamma_{t}\eta_{1}}{\sum_{s=t}^{T}\gamma_{s}}}_{\mathrm{I}}+\underbrace{\sum_{t=\tau+1}^{T}\frac{\gamma_{t}\eta_{t}}{\sum_{s=t}^{T}\gamma_{s}}}_{\mathrm{II}}.
\]
We bound terms $\mathrm{I}$ and $\mathrm{II}$ separately in the
following.
\begin{itemize}
\item Term $\mathrm{I}$: When $t\leq\tau$, we have
\[
\frac{\gamma_{t}}{\sum_{s=t}^{T}\gamma_{s}}\leq\frac{\gamma_{t}}{\sum_{s=\tau}^{T}\gamma_{s}}\overset{(\ref{eq:str-f-exp-gamma-3})}{=}\frac{(1+4\kappa_{f})(1-\mu_{f}\eta_{1})^{\tau-t}}{\sum_{s=\tau}^{T}s-\tau+1+4\kappa_{f}}\leq\frac{2(1+4\kappa_{f})}{(T-\tau+1)(T-\tau+2+8\kappa_{f})},
\]
implying that
\begin{align*}
\mathrm{I} & \leq\frac{2\tau\eta_{1}(1+4\kappa_{f})}{(T-\tau+1)(T-\tau+2+8\kappa_{f})}\overset{\tau\leq\frac{T+1}{2}}{\leq}\frac{4\eta_{1}(1+4\kappa_{f})}{T+3+16\kappa_{f}}\\
 & =\frac{4(1+4\kappa_{f})}{\mu_{f}(1+2\kappa_{f})(T+3+16\kappa_{f})}\leq\frac{8}{\mu_{f}(T+3+16\kappa_{f})}.
\end{align*}
\item Term $\mathrm{II}$: When $t\geq\tau+1$, we have
\[
\frac{\gamma_{t}}{\sum_{s=t}^{T}\gamma_{s}}\overset{(\ref{eq:str-f-exp-gamma-3})}{=}\frac{t-\tau+1+4\kappa_{f}}{\sum_{s=t}^{T}s-\tau+1+4\kappa_{f}}=\frac{2(t-\tau+1+4\kappa_{f})}{(T-t+1)(T+t-2\tau+2+8\kappa_{f})},
\]
implying that
\begin{align*}
\mathrm{II} & \leq\sum_{t=\tau+1}^{T}\frac{4}{\mu_{f}(T-t+1)(T+t-2\tau+2+8\kappa_{f})}\\
 & =\frac{4}{\mu_{f}(2T-2\tau+3+8\kappa_{f})}\sum_{t=\tau+1}^{T}\frac{1}{T-t+1}+\frac{1}{T+t-2\tau+2+8\kappa_{f}}\\
 & \leq\frac{4}{\mu_{f}(2T-2\tau+3+8\kappa_{f})}\left(1+\log(T-\tau)+\log\frac{2T-2\tau+2+8\kappa_{f}}{T-\tau+2+8\kappa_{f}}\right)\\
 & \overset{\frac{T}{2}\leq\tau\leq\frac{T+1}{2}}{\leq}\frac{4(1+\log T)}{\mu_{f}(T+2+8\kappa_{f})}.
\end{align*}
\end{itemize}
Therefore, we obtain
\begin{equation}
\sum_{t=1}^{T}\frac{\gamma_{t}\eta_{t}}{\sum_{s=t}^{T}\gamma_{s}}\leq\frac{8}{\mu_{f}(T+3+16\kappa_{f})}+\frac{4(1+\log T)}{\mu_{f}(T+2+8\kappa_{f})}\leq\frac{12(1+\log T)}{\mu_{f}(T+2+8\kappa_{f})}.\label{eq:str-f-exp-4}
\end{equation}
We combine (\ref{eq:str-f-exp-general}), (\ref{eq:str-f-exp-3}),
and (\ref{eq:str-f-exp-4}) to conclude
\begin{equation}
\E\left[F(x^{T+1})-F(x)\right]\leq\frac{2eLD_{\psi}(x,x^{1})}{\exp\left(\frac{T}{2+4\kappa_{f}}\right)}+\frac{24(M^{2}+\sigma^{2})(1+\log T)}{\mu_{f}(T+2+8\kappa_{f})}=O\left(\frac{LD_{\psi}(x,x^{1})}{\exp\left(\frac{T}{2+4\kappa_{f}}\right)}+\frac{(M^{2}+\sigma^{2})\log T}{\mu_{f}(T+\kappa_{f})}\right).\label{eq:str-f-exp-5}
\end{equation}
\end{proof}

\begin{thm}[Full version of Theorem \ref{thm:main-str-hp}]
\label{thm:str-f-hp}Under Assumptions \ref{enu:A2}-\ref{enu:A4}
and \ref{enu:A5B} with $\mu_{f}>0$ and $\mu_{h}=0$, let $\kappa_{f}\coloneqq\frac{L}{\mu_{f}}\geq0$
and $\delta\in(0,1)$, for any $x\in\dom$:

If $T$ is unknown, by taking either $\eta_{t\in\left[T\right]}=\frac{1}{\mu_{f}(t+2\kappa_{f})}$
or $\eta_{t\in\left[T\right]}=\frac{2}{\mu_{f}(t+1+4\kappa_{f})}$,
then with probability at least $1-\delta$, there is
\[
F(x^{T+1})-F(x)\leq\begin{cases}
O\left(\frac{\mu_{f}(1+\kappa_{f})D_{\psi}(x,x^{1})}{T}+\frac{(M^{2}+\sigma^{2}\log\frac{1}{\delta})\log T}{\mu_{f}(T+\kappa_{f})}\right) & \eta_{t\in\left[T\right]}=\frac{1}{\mu_{f}(t+2\kappa_{f})}\\
O\left(\frac{\mu_{f}(1+\kappa_{f})^{2}D_{\psi}(x,x^{1})}{T(T+\kappa_{f})}+\frac{(M^{2}+\sigma^{2}\log\frac{1}{\delta})\log T}{\mu_{f}(T+\kappa_{f})}\right) & \eta_{t\in\left[T\right]}=\frac{2}{\mu_{f}(t+1+4\kappa_{f})}
\end{cases}.
\]

If $T$ is known, by taking $\eta_{t\in\left[T\right]}=\begin{cases}
\frac{1}{\mu_{f}(1+2\kappa_{f})} & t=1\\
\frac{1}{\mu_{f}(\eta+2\kappa_{f})} & 2\leq t\leq\tau\\
\frac{2}{\mu_{f}(t-\tau+2+4\kappa_{f})} & t\geq\tau+1
\end{cases}$ where $\eta\geq0$ can be any number satisfying $\eta+\kappa_{f}>1$
and $\tau\coloneqq\left\lceil \frac{T}{2}\right\rceil $, then with
probability at least $1-\delta$, there is
\[
F(x^{T+1})-F(x)\leq O\left(\left(1\lor\frac{1}{\eta+2\kappa_{f}-1}\right)\left[\frac{\mu_{f}(1+\kappa_{f})D_{\psi}(x,x^{1})}{\exp\left(\frac{T}{2\eta+4\kappa_{f}}\right)}+\frac{(M^{2}+\sigma^{2}\log\frac{1}{\delta})\log T}{\mu_{f}(T+\kappa_{f})}\right]\right).
\]
\end{thm}

\begin{proof}
When $\mu_{f}>0$ and $\mu_{h}=0$, suppose the condition of $\eta_{t\in\left[T\right]}\leq\frac{1}{2L\lor\mu_{f}}$
in Lemma \ref{lem:core-hp} holds, we have, with probability at least
$1-\delta$,
\begin{equation}
F(x^{T+1})-F(x)\leq2\left(1+\max_{2\leq t\leq T}\frac{1}{1-\mu_{f}\eta_{t}}\right)\left[\frac{D_{\psi}(x,x^{1})}{\sum_{t=1}^{T}\gamma_{t}}+\left(M^{2}+\sigma^{2}\left(1+2\log\frac{2}{\delta}\right)\right)\sum_{t=1}^{T}\frac{\gamma_{t}\eta_{t}}{\sum_{s=t}^{T}\gamma_{s}}\right],\label{eq:str-f-hp-general}
\end{equation}
where $\gamma_{t\in\left[T\right]}=\eta_{t}\prod_{s=2}^{t}\frac{1+\mu_{h}\eta_{s-1}}{1-\mu_{f}\eta_{s}}=\frac{\eta_{t}}{\prod_{s=2}^{t}(1-\mu_{f}\eta_{s})}$.
Now, let us check the condition of $\eta_{t\in\left[T\right]}\leq\frac{1}{2L\lor\mu_{f}}$
for the three choices respectively:
\begin{align*}
\eta_{t\in\left[T\right]} & =\frac{1}{\mu_{f}(t+2\kappa_{f})}\leq\frac{1}{\mu_{f}+2L}\leq\frac{1}{2L\lor\mu_{f}};\\
\eta_{t\in\left[T\right]} & =\frac{2}{\mu_{f}(t+1+4\kappa_{f})}\leq\frac{1}{\mu_{f}+2L}\leq\frac{1}{2L\lor\mu_{f}};\\
\eta_{t\in\left[T\right]} & =\begin{cases}
\frac{1}{\mu_{f}(1+2\kappa_{f})}=\frac{1}{\mu_{f}+2L}\leq\frac{1}{2L\lor\mu_{f}} & t=1\\
\frac{1}{\mu_{f}(\eta+2\kappa_{f})}\leq\frac{1}{\mu_{f}(2\kappa_{f}\lor(\eta+\kappa_{f}))}\leq\frac{1}{2L\lor\mu_{f}} & 2\leq t\leq\tau\\
\frac{2}{\mu_{f}(t-\tau+2+4\kappa_{f})}\leq\frac{1}{\mu_{f}+2L}\leq\frac{1}{2L\lor\mu_{f}} & t\geq\tau+1
\end{cases}.
\end{align*}
Therefore, (\ref{eq:str-f-hp-general}) holds for all cases.

First, we consider $\eta_{t\in\left[T\right]}=\frac{1}{\mu_{f}(t+2\kappa_{f})}$.
We can find $1+\max_{2\leq t\leq T}\frac{1}{1-\mu_{f}\eta_{t}}=1+\frac{1}{1-\mu_{f}\eta_{2}}\leq3$.
Hence, using (\ref{eq:str-f-hp-general}), we have
\[
F(x^{T+1})-F(x)\leq6\left(\frac{D_{\psi}(x,x^{1})}{\sum_{t=1}^{T}\gamma_{t}}+\left(M^{2}+\sigma^{2}\left(1+2\log\frac{2}{\delta}\right)\right)\sum_{t=1}^{T}\frac{\gamma_{t}\eta_{t}}{\sum_{s=t}^{T}\gamma_{s}}\right).
\]
Following similar steps in the proof of (\ref{eq:str-f-exp-1}), we
can get
\[
F(x^{T+1})-F(x)\leq O\left(\frac{\mu_{f}(1+\kappa_{f})D_{\psi}(x,x^{1})}{T}+\frac{(M^{2}+\sigma^{2}\log\frac{1}{\delta})\log T}{\mu_{f}(T+\kappa_{f})}\right).
\]

Next, for the case of $\eta_{t\in\left[T\right]}=\frac{2}{\mu_{f}(t+1+4\kappa_{f})}$,
there is $1+\max_{2\leq t\leq T}\frac{1}{1-\mu_{f}\eta_{t}}=1+\frac{1}{1-\mu_{f}\eta_{2}}\leq4$.
Thus, using (\ref{eq:str-f-hp-general}), we have
\[
F(x^{T+1})-F(x)\leq8\left(\frac{D_{\psi}(x,x^{1})}{\sum_{t=1}^{T}\gamma_{t}}+\left(M^{2}+\sigma^{2}\left(1+2\log\frac{2}{\delta}\right)\right)\sum_{t=1}^{T}\frac{\gamma_{t}\eta_{t}}{\sum_{s=t}^{T}\gamma_{s}}\right).
\]
Following similar steps in the proof of (\ref{eq:str-f-exp-2}), we
can get
\[
F(x^{T+1})-F(x)\leq O\left(\frac{\mu_{f}(1+\kappa_{f})^{2}D_{\psi}(x,x^{1})}{T(T+\kappa_{f})}+\frac{(M^{2}+\sigma^{2}\log\frac{1}{\delta})\log T}{\mu_{f}(T+\kappa_{f})}\right).
\]

Finally, if $T$ is known, we recall the current choice is
\[
\eta_{t\in\left[T\right]}=\begin{cases}
\frac{1}{\mu_{f}(1+2\kappa_{f})} & t=1\\
\frac{1}{\mu_{f}(\eta+2\kappa_{f})} & 2\leq t\leq\tau\\
\frac{2}{\mu_{f}(t-\tau+2+4\kappa_{f})} & t\geq\tau+1
\end{cases},
\]
where $\eta>1$ and $\tau=\left\lceil \frac{T}{2}\right\rceil $.
Note that we have
\begin{align*}
1+\max_{2\leq t\leq T}\frac{1}{1-\mu_{f}\eta_{t}} & =1+\frac{1}{1-\mu_{f}(\eta_{2}\lor\eta_{\tau+1})}=2+\frac{1}{\eta\land1.5+2\kappa_{f}-1}\\
 & =2+\frac{1}{\eta+2\kappa_{f}-1}\lor\frac{1}{0.5+2\kappa_{f}}\leq4+\frac{1}{\eta+2\kappa_{f}-1}.
\end{align*}
Thus, using (\ref{eq:str-f-hp-general}), we obtain
\[
F(x^{T+1})-F(x)\leq\left(8+\frac{2}{\eta+2\kappa_{f}-1}\right)\left(\frac{D_{\psi}(x,x^{1})}{\sum_{t=1}^{T}\gamma_{t}}+\left(M^{2}+\sigma^{2}\left(1+2\log\frac{2}{\delta}\right)\right)\sum_{t=1}^{T}\frac{\gamma_{t}\eta_{t}}{\sum_{s=t}^{T}\gamma_{s}}\right).
\]
Following similar steps in the proof of (\ref{eq:str-f-exp-5}), we
can conclude
\[
F(x^{T+1})-F(x)\leq O\left(\left(1\lor\frac{1}{\eta+2\kappa_{f}-1}\right)\left[\frac{\mu_{f}(1+\kappa_{f})D_{\psi}(x,x^{1})}{\exp\left(\frac{T}{2\eta+4\kappa_{f}}\right)}+\frac{(M^{2}+\sigma^{2}\log\frac{1}{\delta})\log T}{\mu_{f}(T+\kappa_{f})}\right]\right).
\]
\end{proof}

\subsubsection{Optimal Rates via a New Step Size\label{subsec:str-f-optimal}}

Now, we present the full version of Theorems \ref{thm:main-str-exp-optimal}
and \ref{thm:main-str-hp-optimal}.
\begin{thm}[Full version of Theorem \ref{thm:main-str-exp-optimal}]
\label{thm:str-f-exp-optimal}Under Assumptions \ref{enu:A2}-\ref{enu:A4}
and \ref{enu:A5A} with $\mu_{f}>0$ and $\mu_{h}=0$, let $\kappa_{f}\coloneqq\frac{L}{\mu_{f}}\geq0$,
for any $x\in\dom$, if $T$ is known, by taking $\eta_{t\in\left[T\right]}=\begin{cases}
\frac{1}{\mu_{f}(1+2\kappa_{f})} & t\leq\tau_{1}\\
\frac{2}{\mu_{f}(t-\tau_{1}+2+4\kappa_{f})} & \tau_{1}+1\leq t\leq\tau_{2}\\
\frac{T-t+1}{2L(T-\tau_{2})}\land\frac{T-t+1}{\mu_{f}(T-\tau_{2})(T+\kappa_{f})} & t\geq\tau_{2}+1
\end{cases}$ with $\tau_{1}\coloneqq\left\lceil \frac{T}{4}\right\rceil $ and
$\tau_{2}\coloneqq\left\lceil \frac{T}{2}\right\rceil $, there is
\[
\E\left[F(x^{T+1})-F(x)\right]\leq O\left(\frac{LD_{\psi}(x,x^{1})}{\exp\left(\frac{T}{4+8\kappa_{f}}\right)}+\frac{M^{2}+\sigma^{2}}{\mu_{f}(T+\kappa_{f})}\right).
\]
\end{thm}

\begin{proof}
When $\mu_{f}>0$ and $\mu_{h}=0$, suppose the condition of $\eta_{t\in\left[T\right]}\leq\frac{1}{2L\lor\mu_{f}}$
in Lemma \ref{lem:core-exp} holds, we have
\begin{equation}
\E\left[F(x^{T+1})-F(x)\right]\leq\frac{(1-\mu_{f}\eta_{1})D_{\psi}(x,x^{1})}{\sum_{t=1}^{T}\gamma_{t}}+2(M^{2}+\sigma^{2})\sum_{t=1}^{T}\frac{\gamma_{t}\eta_{t}}{\sum_{s=t}^{T}\gamma_{s}},\label{eq:str-f-exp-optimal-general}
\end{equation}
where $\gamma_{t\in\left[T\right]}=\eta_{t}\prod_{s=2}^{t}\frac{1+\mu_{h}\eta_{s-1}}{1-\mu_{f}\eta_{s}}=\frac{\eta_{t}}{\prod_{s=2}^{t}(1-\mu_{f}\eta_{s})}$.
Now, let us check the condition of $\eta_{t\in\left[T\right]}\leq\frac{1}{2L\lor\mu_{f}}$:
\[
\eta_{t\in\left[T\right]}=\begin{cases}
\frac{1}{\mu_{f}(1+2\kappa_{f})}=\frac{1}{\mu_{f}+2L}\leq\frac{1}{2L\lor\mu_{f}} & t\leq\tau_{1}\\
\frac{2}{\mu_{f}(t-\tau_{1}+2+4\kappa_{f})}\leq\frac{1}{\mu_{f}+2L}\leq\frac{1}{2L\lor\mu_{f}} & \tau_{1}+1\leq t\leq\tau_{2}\\
\frac{T-t+1}{2L(T-\tau_{2})}\land\frac{T-t+1}{\mu_{f}(T-\tau_{2})(T+\kappa_{f})}\leq\frac{1}{2L}\land\frac{1}{\mu(T+\kappa_{f})}\leq\frac{1}{2L\lor\mu_{f}} & t\geq\tau_{2}+1
\end{cases}.
\]
Therefore, (\ref{eq:str-f-exp-optimal-general}) holds.

In the following proof, we write
\begin{eqnarray}
\eta=1 & \text{and} & \eta_{*}=\frac{1}{2L(T-\tau_{2})}\land\frac{1}{\mu_{f}(T-\tau_{2})(T+\kappa_{f})}.\label{eq:str-f-exp-optimal-eta-star}
\end{eqnarray}
Note that we have $\eta_{1}=\frac{1}{\mu_{f}(\eta+2\kappa_{f})}$
and
\begin{equation}
\gamma_{t\in\left[T\right]}=\frac{\eta_{t}}{\prod_{s=2}^{t}(1-\mu_{f}\eta_{s})}=\begin{cases}
\frac{\eta_{1}}{(1-\mu_{f}\eta_{1})^{t-1}} & t\leq\tau_{1}\\
\frac{(t-\tau_{1}+1+4\kappa_{f})}{(1-\mu_{f}\eta_{1})^{\tau_{1}-1}\mu_{f}(1+2\kappa_{f})(1+4\kappa_{f})} & \tau_{1}+1\leq t\leq\tau_{2}\\
\frac{(\tau_{2}-\tau_{1}+1+4\kappa_{f})(\tau_{2}-\tau_{1}+2+4\kappa_{f})\eta_{*}(T-t+1)}{(1-\mu_{f}\eta_{1})^{\tau_{1}-1}(1+4\kappa_{f})(2+4\kappa_{f})\prod_{s=\tau_{2}+1}^{t}(1-\mu_{f}\eta_{s})} & t\geq\tau_{2}+1
\end{cases}.\label{eq:str-f-exp-optimal-gamma}
\end{equation}
So, there is
\begin{align}
\frac{1-\mu_{f}\eta_{1}}{\sum_{t=1}^{T}\gamma_{t}} & \leq\frac{1-\mu_{f}\eta_{1}}{\sum_{t=1}^{\tau_{1}}\gamma_{t}}\overset{(\ref{eq:str-f-exp-optimal-gamma})}{=}\frac{1-\mu_{f}\eta_{1}}{\sum_{t=1}^{\tau_{1}}\frac{\eta_{1}}{(1-\mu_{f}\eta_{1})^{t-1}}}=\frac{\mu_{f}(1-\mu_{f}\eta_{1})^{\tau_{1}}}{1-(1-\mu_{f}\eta_{1})^{\tau_{1}}}\nonumber \\
 & \leq\frac{(1-\mu_{f}\eta_{1})^{\tau_{1}}}{\eta_{1}}\leq(\eta_{1}^{-1}-\mu_{f})\exp\left(-\mu_{f}\eta_{1}(\tau_{1}-1)\right)\nonumber \\
 & \overset{\eta_{1}\mu_{f}\leq1,\tau_{1}\geq\frac{T}{4}}{\leq}e(\eta_{1}^{-1}-\mu_{f})\exp\left(-\frac{\mu_{f}\eta_{1}T}{4}\right)=2eL\exp\left(-\frac{T}{4+8\kappa_{f}}\right).\label{eq:str-f-exp-optimal-1}
\end{align}
Now, we observe that
\begin{align*}
\sum_{t=1}^{T}\frac{\gamma_{t}\eta_{t}}{\sum_{s=t}^{T}\gamma_{s}} & =\sum_{t=1}^{\tau_{1}}\frac{\gamma_{t}\eta_{t}}{\sum_{s=t}^{T}\gamma_{s}}+\sum_{t=\tau_{1}+1}^{\tau_{2}}\frac{\gamma_{t}\eta_{t}}{\sum_{s=t}^{T}\gamma_{s}}+\sum_{t=\tau_{2}+1}^{T}\frac{\gamma_{t}\eta_{t}}{\sum_{s=t}^{T}\gamma_{s}}\\
 & =\underbrace{\sum_{t=1}^{\tau_{1}}\frac{\gamma_{t}\eta_{1}}{\sum_{s=t}^{T}\gamma_{s}}}_{\mathrm{I}}+\underbrace{\sum_{t=\tau_{1}+1}^{\tau_{2}}\frac{\gamma_{t}\eta_{t}}{\sum_{s=t}^{T}\gamma_{s}}}_{\mathrm{II}}+\underbrace{\sum_{t=\tau_{2}+1}^{T}\frac{\gamma_{t}\eta_{t}}{\sum_{s=t}^{T}\gamma_{s}}}_{\mathrm{III}}.
\end{align*}
We bound terms $\mathrm{I}$, $\mathrm{II}$, and $\mathrm{III}$
separately in the following.
\begin{itemize}
\item Term $\mathrm{I}$: When $t\leq\tau_{1}$, we note that
\begin{align*}
\frac{\gamma_{t}}{\sum_{s=t}^{T}\gamma_{s}} & \leq\frac{\gamma_{t}}{\sum_{s=\tau_{1}+1}^{\tau_{2}}\gamma_{s}}\overset{(\ref{eq:str-f-exp-optimal-gamma})}{=}\frac{\mu_{f}(1+2\kappa_{f})(1+4\kappa_{f})\eta_{1}(1-\mu_{f}\eta_{1})^{\tau_{1}-t}}{\sum_{s=\tau_{1}+1}^{\tau_{2}}s-\tau_{1}+1+4\kappa_{f}}\\
 & =\frac{2\mu_{f}(1+2\kappa_{f})(1+4\kappa_{f})\eta_{1}(1-\mu_{f}\eta_{1})^{\tau_{1}-t}}{(\tau_{2}-\tau_{1})(\tau_{2}-\tau_{1}+3+8\kappa_{f})},
\end{align*}
implying that
\begin{align*}
\mathrm{I} & \leq\frac{2\mu_{f}(1+2\kappa_{f})(1+4\kappa_{f})\eta_{1}^{2}\sum_{t=1}^{\tau_{1}}(1-\mu_{f}\eta_{1})^{\tau_{1}-t}}{(\tau_{2}-\tau_{1})(\tau_{2}-\tau_{1}+3+8\kappa_{f})}\leq\frac{2\mu_{f}(1+2\kappa_{f})(1+4\kappa_{f})\eta_{1}^{2}(\tau_{1}\land\frac{1}{\mu_{f}\eta_{1}})}{(\tau_{2}-\tau_{1})(\tau_{2}-\tau_{1}+3+8\kappa_{f})}\\
 & =\frac{2(1+2\kappa_{f})(1+4\kappa_{f})(\tau_{1}\land(\eta+2\kappa_{f}))}{\mu_{f}(\eta+2\kappa_{f})^{2}(\tau_{2}-\tau_{1})(\tau_{2}-\tau_{1}+3+8\kappa_{f})}\overset{(\ref{eq:str-f-exp-optimal-eta-star})}{\leq}\frac{4\tau_{1}}{\mu_{f}(\tau_{2}-\tau_{1})(\tau_{2}-\tau_{1}+3+8\kappa_{f})}\\
 & \overset{\tau_{1}\leq\frac{T+3}{4},\tau_{2}\geq\frac{T}{2},T\geq4}{\leq}\frac{112}{\mu_{f}(T+9+32\kappa_{f})}\leq\frac{112}{\mu_{f}(T+\kappa_{f})}.
\end{align*}
\item Term $\mathrm{II}$: When $\tau_{1}+1\leq t\leq\tau_{2}$, we note
that
\begin{align*}
\frac{\gamma_{t}}{\sum_{s=t}^{T}\gamma_{s}} & \leq\frac{\gamma_{t}}{\sum_{s=\tau_{2}+1}^{T}\gamma_{s}}\overset{(\ref{eq:str-f-exp-optimal-gamma})}{=}\frac{2(t-\tau_{1}+1+4\kappa_{f})}{\mu_{f}\sum_{s=\tau_{2}+1}^{T}\frac{(\tau_{2}-\tau_{1}+1+4\kappa_{f})(\tau_{2}-\tau_{1}+2+4\kappa_{f})\eta_{*}(T-s+1)}{\prod_{\ell=\tau_{2}+1}^{s}(1-\mu_{f}\eta_{\ell})}}\\
 & \leq\frac{2(t-\tau_{1}+1+4\kappa_{f})}{\mu_{f}(\tau_{2}-\tau_{1}+1+4\kappa_{f})(\tau_{2}-\tau_{1}+2+4\kappa_{f})\eta_{*}\sum_{s=\tau_{2}+1}^{T}T-s+1}\\
 & =\frac{2(t-\tau_{1}+1+4\kappa_{f})}{\mu_{f}(\tau_{2}-\tau_{1}+1+4\kappa_{f})(\tau_{2}-\tau_{1}+2+4\kappa_{f})\eta_{*}(T-\tau_{2})(T-\tau_{2}+1)},
\end{align*}
implying that
\begin{align*}
\mathrm{II} & \leq\frac{4(\tau_{2}-\tau_{1})}{\mu_{f}^{2}(\tau_{2}-\tau_{1}+1+4\kappa_{f})(\tau_{2}-\tau_{1}+2+4\kappa_{f})\eta_{*}(T-\tau_{2})(T-\tau_{2}+1)}\\
 & \overset{(\ref{eq:str-f-exp-optimal-eta-star})}{=}\frac{4(\tau_{2}-\tau_{1})(2\kappa_{f}\lor(T+\kappa_{f}))}{\mu_{f}(\tau_{2}-\tau_{1}+1+4\kappa_{f})(\tau_{2}-\tau_{1}+2+4\kappa_{f})(T-\tau_{2}+1)}\\
 & \overset{(a)}{\leq}\frac{3(2\kappa_{f}\lor(T+\kappa_{f}))}{\mu_{f}(\tau_{2}-\tau_{1}+1+4\kappa_{f})(\tau_{2}-\tau_{1}+2+4\kappa_{f})}\overset{(b)}{\leq}\frac{48(2\kappa_{f}\lor(T+\kappa_{f}))}{\mu_{f}(T+1+16\kappa_{f})(T+5+16\kappa_{f})}\\
 & \leq\frac{54}{\mu_{f}(T+\kappa_{f})},
\end{align*}
where $(a)$ is due to $\frac{\tau_{2}-\tau_{1}}{T-\tau_{2}+1}\leq\frac{\frac{T+1}{2}-\frac{T}{4}}{\frac{T+1}{2}}=\frac{T+2}{2(T+1)}\leq\frac{3}{4}$,
and $(b)$ is by $\tau_{2}-\tau_{1}\geq\frac{T}{2}-\frac{T+3}{4}=\frac{T-3}{4}$.
\item Term $\mathrm{III}$: When $t\geq\tau_{2}+1$, we note that $\frac{\gamma_{t}}{\eta_{t}}=\frac{1}{\prod_{s=2}^{t}(1-\mu_{f}\eta_{s})}$
is always non-decreasing. Hence, we have
\[
\frac{\gamma_{t}}{\sum_{s=t}^{T}\gamma_{s}}=\frac{\frac{\gamma_{t}}{\eta_{t}}\eta_{t}}{\sum_{s=t}^{T}\frac{\gamma_{s}}{\eta_{s}}\eta_{s}}\leq\frac{\eta_{t}}{\sum_{s=t}^{T}\eta_{s}}=\frac{T-t+1}{\sum_{s=t}^{T}T-s+1}=\frac{2}{T-t+2},
\]
implying that
\[
\mathrm{III}\leq2\eta_{*}\sum_{t=\tau_{2}+1}^{T}\frac{T-t+1}{T-t+2}\leq2\eta_{*}(T-\tau_{2})\overset{(\ref{eq:str-f-exp-optimal-eta-star})}{\leq}\frac{2}{\mu_{f}(T+\kappa_{f})}.
\]
\end{itemize}
Therefore, we obtain
\begin{equation}
\sum_{t=1}^{T}\frac{\gamma_{t}\eta_{t}}{\sum_{s=t}^{T}\gamma_{s}}\leq\frac{112}{\mu_{f}(T+\kappa_{f})}+\frac{54}{\mu_{f}(T+\kappa_{f})}+\frac{2}{\mu_{f}(T+\kappa_{f})}=\frac{168}{\mu_{f}(T+\kappa_{f})}.\label{eq:str-f-exp-optimal-2}
\end{equation}
We combine (\ref{eq:str-f-exp-optimal-general}), (\ref{eq:str-f-exp-optimal-1}),
and (\ref{eq:str-f-exp-optimal-2}) to conclude
\begin{equation}
\E\left[F(x^{T+1})-F(x)\right]\leq\frac{2eLD_{\psi}(x,x^{1})}{\exp\left(\frac{T}{4+8\kappa_{f}}\right)}+\frac{336(M^{2}+\sigma^{2})}{\mu_{f}(T+\kappa_{f})}=O\left(\frac{LD_{\psi}(x,x^{1})}{\exp\left(\frac{T}{4+8\kappa_{f}}\right)}+\frac{M^{2}+\sigma^{2}}{\mu_{f}(T+\kappa_{f})}\right).\label{eq:str-f-exp-optimal-3}
\end{equation}
\end{proof}

\begin{thm}[Full version of Theorem \ref{thm:main-str-hp-optimal}]
\label{thm:str-f-hp-optimal}Under Assumptions \ref{enu:A2}-\ref{enu:A4}
and \ref{enu:A5B} with $\mu_{f}>0$ and $\mu_{h}=0$, let $\kappa_{f}\coloneqq\frac{L}{\mu_{f}}\geq0$
and $\delta\in(0,1)$, if $T$ is known, by taking $\eta_{t\in\left[T\right]}=\begin{cases}
\frac{1}{\mu_{f}(\eta+2\kappa_{f})} & t\leq\tau_{1}\\
\frac{2}{\mu_{f}(t-\tau_{1}+2+4\kappa_{f})} & \tau_{1}+1\leq t\leq\tau_{2}\\
\frac{T-t+1}{2L(T-\tau_{2})}\land\frac{T-t+1}{\mu_{f}(T-\tau_{2})(T+2\kappa_{f})} & t\geq\tau_{2}+1
\end{cases}$ where $\eta\geq0$ can be any number satisfying $\eta+\kappa_{f}>1$,
$\tau_{1}\coloneqq\left\lceil \frac{T}{4}\right\rceil $ and $\tau_{2}\coloneqq\left\lceil \frac{T}{2}\right\rceil $,
then with probability at least $1-\delta$, there is
\[
F(x^{T+1})-F(x)\leq O\left(\left(1\lor\frac{1}{\eta+2\kappa_{f}-1}\right)\left[\frac{\mu_{f}(\eta+\kappa_{f})D_{\psi}(x,x^{1})}{\exp\left(\frac{T}{4\eta+8\kappa_{f}}\right)}+\frac{M^{2}+\sigma^{2}\log\frac{1}{\delta}}{\mu_{f}(T+\kappa_{f})}\right]\right).
\]
\end{thm}

\begin{proof}
When $\mu_{f}>0$ and $\mu_{h}=0$, suppose the condition of $\eta_{t\in\left[T\right]}\leq\frac{1}{2L\lor\mu_{f}}$
in Lemma \ref{lem:core-hp} holds, we have, with probability at least
$1-\delta$,
\begin{equation}
F(x^{T+1})-F(x)\leq2\left(1+\max_{2\leq t\leq T}\frac{1}{1-\mu_{f}\eta_{t}}\right)\left[\frac{D_{\psi}(x,x^{1})}{\sum_{t=1}^{T}\gamma_{t}}+\left(M^{2}+\sigma^{2}\left(1+2\log\frac{2}{\delta}\right)\right)\sum_{t=1}^{T}\frac{\gamma_{t}\eta_{t}}{\sum_{s=t}^{T}\gamma_{s}}\right],\label{eq:str-f-hp-optimal-general}
\end{equation}
where $\gamma_{t\in\left[T\right]}=\eta_{t}\prod_{s=2}^{t}\frac{1+\mu_{h}\eta_{s-1}}{1-\mu_{f}\eta_{s}}=\frac{\eta_{t}}{\prod_{s=2}^{t}(1-\mu_{f}\eta_{s})}$.
Now, let us check the condition of $\eta_{t\in\left[T\right]}\leq\frac{1}{2L\lor\mu_{f}}$:
\[
\eta_{t\in\left[T\right]}=\begin{cases}
\frac{1}{\mu_{f}(\eta+2\kappa_{f})}\leq\frac{1}{\mu_{f}(2\kappa_{f}\lor(\eta+\kappa_{f}))}\leq\frac{1}{2L\lor\mu_{f}} & t\leq\tau_{1}\\
\frac{2}{\mu_{f}(t-\tau_{1}+2+4\kappa_{f})}\leq\frac{1}{\mu_{f}+2L}\leq\frac{1}{2L\lor\mu_{f}} & \tau_{1}+1\leq t\leq\tau_{2}\\
\frac{T-t+1}{2L(T-\tau_{2})}\land\frac{T-t+1}{\mu_{f}(T-\tau_{2})(T+2\kappa_{f})}\leq\frac{1}{2L}\land\frac{1}{\mu(T+2\kappa_{f})}\leq\frac{1}{2L\lor\mu_{f}} & t\geq\tau_{2}+1
\end{cases}.
\]
Therefore, (\ref{eq:str-f-hp-optimal-general}) holds.

Note that we have
\[
\eta_{\tau_{1}+1}=\frac{1}{\mu_{f}(1.5+2\kappa_{f})}\geq\frac{1}{\mu_{f}(T+2\kappa_{f})}\geq\eta_{\tau_{2}+1},
\]
which implies that
\begin{align*}
1+\max_{2\leq t\leq T}\frac{1}{1-\mu_{f}\eta_{t}} & =1+\frac{1}{1-\mu_{f}(\eta_{2}\lor\eta_{\tau_{1}+1}\lor\eta_{\tau_{2}+1})}=1+\frac{1}{1-\mu_{f}(\eta_{2}\lor\eta_{\tau_{1}+1})}\\
 & =2+\frac{1}{\eta\land1.5+2\kappa_{f}-1}=2+\frac{1}{\eta+2\kappa_{f}-1}\lor\frac{1}{0.5+2\kappa_{f}}\\
 & \leq4+\frac{1}{\eta+2\kappa_{f}-1}.
\end{align*}
Thus, using (\ref{eq:str-f-hp-optimal-general}), we obtain
\[
F(x^{T+1})-F(x)\leq\left(8+\frac{2}{\eta+2\kappa_{f}-1}\right)\left(\frac{D_{\psi}(x,x^{1})}{\sum_{t=1}^{T}\gamma_{t}}+\left(M^{2}+\sigma^{2}\left(1+2\log\frac{2}{\delta}\right)\right)\sum_{t=1}^{T}\frac{\gamma_{t}\eta_{t}}{\sum_{s=t}^{T}\gamma_{s}}\right).
\]
Following similar steps in the proof of (\ref{eq:str-f-exp-optimal-3}),
we can conclude
\[
F(x^{T+1})-F(x)\leq O\left(\left(1\lor\frac{1}{\eta+2\kappa_{f}-1}\right)\left[\frac{\mu_{f}(\eta+\kappa_{f})D_{\psi}(x,x^{1})}{\exp\left(\frac{T}{4\eta+8\kappa_{f}}\right)}+\frac{M^{2}+\sigma^{2}\log\frac{1}{\delta}}{\mu_{f}(T+\kappa_{f})}\right]\right).
\]
\end{proof}

\subsection{The Case of $\mu_{h}>0$\label{subsec:str-h}}
\begin{thm}
\label{thm:str-h-exp}Under Assumptions \ref{enu:A2}-\ref{enu:A4}
and \ref{enu:A5A} with $\mu_{f}=0$ and $\mu_{h}>0$, let $\kappa_{h}\coloneqq\frac{L}{\mu_{h}}\geq0$,
for any $x\in\dom$:

If $T$ is unknown, by taking $\eta_{t\in\left[T\right]}=\frac{2}{\mu_{h}(t+4\kappa_{h})}$,
there is
\[
\E\left[F(x^{T+1})-F(x)\right]\leq O\left(\frac{\mu_{h}(1+\kappa_{h})^{2}D_{\psi}(x,x^{1})}{T(T+\kappa_{h})}+\frac{(M^{2}+\sigma^{2})\log T}{\mu_{h}(T+\kappa_{h})}\right).
\]

If $T$ is known, by taking $\eta_{t\in\left[T\right]}=\begin{cases}
\frac{1}{\mu_{h}(\eta+2\kappa_{h})} & t\leq\tau\\
\frac{2}{\mu_{h}(t-\tau+4\kappa_{h})} & t\geq\tau+1
\end{cases}$ where $\eta\geq0$ can be any number satisfying $\eta+\kappa_{h}>0$
and $\tau\coloneqq\left\lceil \frac{T}{2}\right\rceil $, there is
\[
\E\left[F(x^{T+1})-F(x)\right]\leq O\left(\frac{\mu_{h}D_{\psi}(x,x^{1})}{\exp\left(\frac{T}{2(1+\eta+2\kappa_{h})}\right)-1}+\left(\frac{1}{\eta+2\kappa_{h}}\lor\log T\right)\frac{M^{2}+\sigma^{2}}{\mu_{h}(T+\kappa_{h})}\right).
\]
\end{thm}

\begin{proof}
When $\mu_{f}=0$ and $\mu_{h}>0$, suppose the condition of $\eta_{t\in\left[T\right]}\leq\frac{1}{2L\lor\mu_{f}}=\frac{1}{2L}$
in Lemma \ref{lem:core-exp} holds, we have
\begin{align}
\E\left[F(x^{T+1})-F(x)\right] & \leq\frac{(1-\mu_{f}\eta_{1})D_{\psi}(x,x^{1})}{\sum_{t=1}^{T}\gamma_{t}}+2(M^{2}+\sigma^{2})\sum_{t=1}^{T}\frac{\gamma_{t}\eta_{t}}{\sum_{s=t}^{T}\gamma_{s}}\nonumber \\
 & =\frac{D_{\psi}(x,x^{1})}{\sum_{t=1}^{T}\gamma_{t}}+2(M^{2}+\sigma^{2})\sum_{t=1}^{T}\frac{\gamma_{t}\eta_{t}}{\sum_{s=t}^{T}\gamma_{s}},\label{eq:str-h-exp-general}
\end{align}
where $\gamma_{t\in\left[T\right]}=\eta_{t}\prod_{s=2}^{t}\frac{1+\mu_{h}\eta_{s-1}}{1-\mu_{f}\eta_{s}}=\eta_{t}\prod_{s=2}^{t}(1+\mu_{h}\eta_{s-1})$.
Now, let us check the condition of $\eta_{t\in\left[T\right]}\leq\frac{1}{2L}$
for the two choices respectively:
\begin{align*}
\eta_{t\in\left[T\right]} & =\frac{2}{\mu_{h}(t+4\kappa_{h})}\leq\frac{1}{2\kappa_{h}\mu_{h}}=\frac{1}{2L};\\
\eta_{t\in\left[T\right]} & =\begin{cases}
\frac{1}{\mu_{h}(\eta+2\kappa_{h})}\leq\frac{1}{2\kappa_{h}\mu_{h}}=\frac{1}{2L} & t\leq\tau\\
\frac{2}{\mu_{h}(t-\tau+4\kappa_{h})}\leq\frac{1}{2\kappa_{h}\mu_{h}}=\frac{1}{2L} & t\geq\tau+1
\end{cases}.
\end{align*}
Therefore, (\ref{eq:str-h-exp-general}) holds for both cases.

If $\eta_{t\in\left[T\right]}=\frac{2}{\mu_{h}(t+4\kappa_{h})}$,
we have $\eta_{1}=\frac{2}{\mu_{h}(1+4\kappa_{h})}$ and
\begin{align}
\gamma_{t\in\left[T\right]} & =\eta_{t}\prod_{s=2}^{t}(1+\mu_{h}\eta_{s-1})=\frac{2}{\mu_{h}(t+4\kappa_{h})}\prod_{s=2}^{t}\left(1+\frac{2}{s-1+4\kappa_{h}}\right)\nonumber \\
 & =\frac{2(t+1+4\kappa_{h})}{\mu_{h}(1+4\kappa_{h})(2+4\kappa_{h})}=\frac{t+1+4\kappa_{h}}{2+4\kappa_{h}}\eta_{1}.\label{eq:str-h-exp-gamma-1}
\end{align}
Hence, by (\ref{eq:str-h-exp-general}) and (\ref{eq:str-h-exp-gamma-1}),
\begin{align}
\E\left[F(x^{T+1})-F(x)\right] & \leq\frac{(2+4\kappa_{h})\eta_{1}^{-1}D_{\psi}(x,x^{1})}{\sum_{t=1}^{T}t+1+4\kappa_{h}}+2(M^{2}+\sigma^{2})\sum_{t=1}^{T}\frac{(t+1+4\kappa_{h})\eta_{t}}{\sum_{s=t}^{T}s+1+4\kappa_{h}}\nonumber \\
 & =\frac{\mu_{h}(2+4\kappa_{h})(1+4\kappa_{h})D_{\psi}(x,x^{1})}{T(T+3+8\kappa_{h})}+\frac{8(M^{2}+\sigma^{2})}{\mu_{h}}\sum_{t=1}^{T}\frac{t+1+4\kappa_{h}}{(t+4\kappa_{h})(T-t+1)(T+t+2+8\kappa_{h})}\nonumber \\
 & \leq\frac{\mu_{h}(2+4\kappa_{h})(1+4\kappa_{h})D_{\psi}(x,x^{1})}{T(T+3+8\kappa_{h})}+\frac{16(M^{2}+\sigma^{2})}{\mu_{h}}\sum_{t=1}^{T}\frac{1}{(T-t+1)(T+t+2+8\kappa_{h})}\nonumber \\
 & =\frac{\mu_{h}(2+4\kappa_{h})(1+4\kappa_{h})D_{\psi}(x,x^{1})}{T(T+3+8\kappa_{h})}+\frac{16(M^{2}+\sigma^{2})}{\mu_{h}(2T+3+8\kappa_{h})}\sum_{t=1}^{T}\frac{1}{T-t+1}+\frac{1}{T+t+2+8\kappa_{h}}\nonumber \\
 & \leq\frac{\mu_{h}(2+4\kappa_{h})(1+4\kappa_{h})D_{\psi}(x,x^{1})}{T(T+3+8\kappa_{h})}+\frac{16(M^{2}+\sigma^{2})}{\mu_{h}(2T+3+8\kappa_{h})}\left(1+\log T+\log\frac{2T+2+8\kappa_{h}}{T+2+8\kappa_{h}}\right)\nonumber \\
 & \leq\frac{\mu_{h}(2+4\kappa_{h})(1+4\kappa_{h})D_{\psi}(x,x^{1})}{T(T+3+8\kappa_{h})}+\frac{16(M^{2}+\sigma^{2})(1+\log2T)}{\mu_{h}(2T+3+8\kappa_{h})}\nonumber \\
 & =O\left(\frac{\mu_{h}(1+\kappa_{h})^{2}D_{\psi}(x,x^{1})}{T(T+\kappa_{h})}+\frac{(M^{2}+\sigma^{2})\log T}{\mu_{h}(T+\kappa_{h})}\right).\label{eq:str-h-exp-1}
\end{align}

If $\eta_{t\in\left[T\right]}=\begin{cases}
\frac{1}{\mu_{h}(\eta+2\kappa_{h})} & t\leq\tau\\
\frac{2}{\mu_{h}(t-\tau+4\kappa_{h})} & t\geq\tau+1
\end{cases}$, where $\eta\geq0$ satisfies $\eta+\kappa_{h}>0$ and $\tau=\left\lceil \frac{T}{2}\right\rceil $,
we know $\eta_{1}=\frac{1}{\mu_{h}(\eta+2\kappa_{h})}$ and
\begin{equation}
\gamma_{t\in\left[T\right]}=\eta_{t}\prod_{s=2}^{t}(1+\mu_{h}\eta_{s-1})=\begin{cases}
\eta_{1}(1+\mu_{h}\eta_{1})^{t-1} & t\leq\tau\\
\frac{2(t-\tau+1+4\kappa_{h})}{\mu_{h}(1+4\kappa_{h})(2+4\kappa_{h})}(1+\mu_{h}\eta_{1})^{\tau} & t\geq\tau+1
\end{cases}.\label{eq:str-h-exp-gamma-2}
\end{equation}
So, there is
\begin{align}
\frac{1}{\sum_{t=1}^{T}\gamma_{t}} & \leq\frac{1}{\sum_{t=1}^{\tau}\gamma_{t}}\overset{(\ref{eq:str-h-exp-gamma-2})}{=}\frac{1}{\sum_{t=1}^{\tau}\eta_{1}(1+\mu_{h}\eta_{1})^{t-1}}=\frac{\mu_{h}}{(1+\mu_{h}\eta_{1})^{\tau}-1}\nonumber \\
 & \overset{(a)}{\leq}\frac{\mu_{h}}{\exp\left(\frac{\mu_{h}\eta_{1}\tau}{1+\mu_{h}\eta_{1}}\right)-1}=\frac{\mu_{h}}{\exp\left(\frac{\tau}{1+\eta+2\kappa_{h}}\right)-1}\overset{\tau\geq\frac{T}{2}}{\leq}\frac{\mu_{h}}{\exp\left(\frac{T}{2(1+\eta+2\kappa_{h})}\right)-1},\label{eq:str-h-exp-2}
\end{align}
where $(a)$ is by
\[
(1+\mu_{h}\eta_{1})^{\tau}=\exp\left(\tau\log\left(1+\mu_{h}\eta_{1}\right)\right)\geq\exp\left(\frac{\mu_{h}\eta_{1}\tau}{1+\mu_{h}\eta_{1}}\right).
\]
Now, we observe that
\[
\sum_{t=1}^{T}\frac{\gamma_{t}\eta_{t}}{\sum_{s=t}^{T}\gamma_{s}}=\sum_{t=1}^{\tau}\frac{\gamma_{t}\eta_{t}}{\sum_{s=t}^{T}\gamma_{s}}+\sum_{t=\tau+1}^{T}\frac{\gamma_{t}\eta_{t}}{\sum_{s=t}^{T}\gamma_{s}}=\underbrace{\sum_{t=1}^{\tau}\frac{\gamma_{t}\eta_{1}}{\sum_{s=t}^{T}\gamma_{s}}}_{\mathrm{I}}+\underbrace{\sum_{t=\tau+1}^{T}\frac{\gamma_{t}\eta_{t}}{\sum_{s=t}^{T}\gamma_{s}}}_{\mathrm{II}}.
\]
We bound terms $\mathrm{I}$ and $\mathrm{II}$ separately in the
following.
\begin{itemize}
\item Term $\mathrm{I}$: When $t\leq\tau$, we have
\begin{align*}
\frac{\gamma_{t}}{\sum_{s=t}^{T}\gamma_{s}} & \leq\frac{\gamma_{t}}{\sum_{s=\tau+1}^{T}\gamma_{s}}\overset{(\ref{eq:str-h-exp-gamma-2})}{=}\frac{\eta_{1}(1+\mu_{h}\eta_{1})^{t-1}}{\sum_{s=\tau+1}^{T}\frac{2(s-\tau+1+4\kappa_{h})}{\mu_{h}(1+4\kappa_{h})(2+4\kappa_{h})}(1+\mu_{h}\eta_{1})^{\tau}}\\
 & \leq\frac{\eta_{1}}{(1+\mu_{h}\eta_{1})\sum_{s=\tau+1}^{T}\frac{2(s-\tau+1+4\kappa_{h})}{\mu_{h}(1+4\kappa_{h})(2+4\kappa_{h})}}=\frac{\mu_{h}(1+4\kappa_{h})(2+4\kappa_{h})\eta_{1}}{(1+\mu_{h}\eta_{1})(T-\tau)(T-\tau+3+8\kappa_{h})},
\end{align*}
implying that
\begin{align*}
\mathrm{I} & \leq\frac{\mu_{h}(1+4\kappa_{h})(2+4\kappa_{h})\eta_{1}^{2}\tau}{(1+\mu_{h}\eta_{1})(T-\tau)(T-\tau+3+8\kappa_{h})}\overset{\tau\leq\frac{T+1}{2},T\geq2}{\leq}\frac{6\mu_{h}(1+4\kappa_{h})(2+4\kappa_{h})\eta_{1}^{2}}{(1+\mu_{h}\eta_{1})(T+5+16\kappa_{h})}\\
 & =\frac{6(1+4\kappa_{h})(2+4\kappa_{h})}{\mu_{h}(\eta+2\kappa_{h})(1+\eta+2\kappa_{h})(T+5+16\kappa_{h})}\leq\frac{12\left(2+\frac{1}{\eta+2\kappa_{h}}\right)}{\mu_{h}(T+5+16\kappa_{h})}.
\end{align*}
\item Term $\mathrm{II}$: When $t\geq\tau+1$, we have
\[
\frac{\gamma_{t}}{\sum_{s=t}^{T}\gamma_{s}}\overset{(\ref{eq:str-h-exp-gamma-2})}{=}\frac{t-\tau+1+4\kappa_{h}}{\sum_{s=t}^{T}s-\tau+1+4\kappa_{h}}=\frac{2(t-\tau+1+4\kappa_{h})}{(T-t+1)(T+t-2\tau+2+8\kappa_{h})},
\]
implying that
\begin{align*}
\mathrm{II} & \leq\sum_{t=\tau+1}^{T}\frac{8}{\mu_{h}(T-t+1)(T+t-2\tau+2+8\kappa_{h})}\\
 & =\frac{8}{\mu_{h}(2T-2\tau+3+8\kappa_{h})}\sum_{t=\tau+1}^{T}\frac{1}{T-t+1}+\frac{1}{T+t-2\tau+2+8\kappa_{h}}\\
 & \leq\frac{8}{\mu_{h}(2T-2\tau+3+8\kappa_{h})}\left(1+\log(T-\tau)+\log\frac{2T-2\tau+2+8\kappa_{h}}{T-\tau+2+8\kappa_{h}}\right)\\
 & \overset{\frac{T}{2}\leq\tau\leq\frac{T+1}{2}}{\leq}\frac{8(1+\log T)}{\mu_{h}(T+2+8\kappa_{h})}.
\end{align*}
\end{itemize}
Therefore, we obtain
\begin{equation}
\sum_{t=1}^{T}\frac{\gamma_{t}\eta_{t}}{\sum_{s=t}^{T}\gamma_{s}}\leq\frac{12\left(2+\frac{1}{\eta+2\kappa_{h}}\right)}{\mu_{h}(T+5+16\kappa_{h})}+\frac{8(1+\log T)}{\mu_{h}(T+2+8\kappa_{h})}\leq\frac{12\left(2+\frac{1}{\eta+2\kappa_{h}}\right)+8(1+\log T)}{\mu_{h}(T+2+8\kappa_{h})}.\label{eq:str-h-exp-3}
\end{equation}
We combine (\ref{eq:str-h-exp-general}), (\ref{eq:str-h-exp-2}),
and (\ref{eq:str-h-exp-3}) to conclude
\begin{align}
\E\left[F(x^{T+1})-F(x)\right] & \leq\frac{\mu_{h}D_{\psi}(x,x^{1})}{\exp\left(\frac{T}{2(1+\eta+2\kappa_{h})}\right)-1}+\frac{8(M^{2}+\sigma^{2})\left[3\left(2+\frac{1}{\eta+2\kappa_{h}}\right)+2(1+\log T)\right]}{\mu_{h}(T+2+8\kappa_{h})}\nonumber \\
 & =O\left(\frac{\mu_{h}D_{\psi}(x,x^{1})}{\exp\left(\frac{T}{2(1+\eta+2\kappa_{h})}\right)-1}+\left(\frac{1}{\eta+2\kappa_{h}}\lor\log T\right)\frac{(M^{2}+\sigma^{2})}{\mu_{h}(T+\kappa_{h})}\right).\label{eq:str-h-exp-4}
\end{align}
\end{proof}

\begin{thm}
\label{thm:str-h-hp}Under Assumptions \ref{enu:A2}-\ref{enu:A4}
and \ref{enu:A5B} with $\mu_{f}=0$ and $\mu_{h}>0$, let $\kappa_{h}\coloneqq\frac{L}{\mu_{h}}\geq0$
and $\delta\in(0,1)$, for any $x\in\dom$:

If $T$ is unknown, by taking $\eta_{t\in\left[T\right]}=\frac{2}{\mu_{h}(t+4\kappa_{h})}$,
then with probability at least $1-\delta$, there is
\[
F(x^{T+1})-F(x)\leq O\left(\frac{\mu_{h}(1+\kappa_{h})^{2}D_{\psi}(x,x^{1})}{T(T+\kappa_{h})}+\frac{(M^{2}+\sigma^{2}\log\frac{1}{\delta})\log T}{\mu_{h}(T+\kappa_{h})}\right).
\]

If $T$ is known, by taking $\eta_{t\in\left[T\right]}=\begin{cases}
\frac{1}{\mu_{h}(\eta+2\kappa_{h})} & t\leq\tau\\
\frac{2}{\mu_{h}(t-\tau+4\kappa_{h})} & t\geq\tau+1
\end{cases}$ where $\eta\geq0$ can be any number satisfying $\eta+\kappa_{h}>0$
and $\tau\coloneqq\left\lceil \frac{T}{2}\right\rceil $, then with
probability at least $1-\delta$, there is
\[
F(x^{T+1})-F(x)\leq O\left(\frac{\mu_{h}D_{\psi}(x,x^{1})}{\exp\left(\frac{T}{2(1+\eta+2\kappa_{h})}\right)-1}+\left(\frac{1}{\eta+2\kappa_{h}}\lor\log T\right)\frac{M^{2}+\sigma^{2}\log\frac{1}{\delta}}{\mu_{h}(T+\kappa_{h})}\right).
\]
\end{thm}

\begin{proof}
When $\mu_{f}=0$ and $\mu_{h}>0$, suppose the condition of $\eta_{t\in\left[T\right]}\leq\frac{1}{2L\lor\mu_{f}}=\frac{1}{2L}$
in Lemma \ref{lem:core-hp} holds, we have, with probability at least
$1-\delta$,
\begin{align}
F(x^{T+1})-F(x) & \leq2\left(1+\max_{2\leq t\leq T}\frac{1}{1-\mu_{f}\eta_{t}}\right)\left[\frac{D_{\psi}(x,x^{1})}{\sum_{t=1}^{T}\gamma_{t}}+\left(M^{2}+\sigma^{2}\left(1+2\log\frac{2}{\delta}\right)\right)\sum_{t=1}^{T}\frac{\gamma_{t}\eta_{t}}{\sum_{s=t}^{T}\gamma_{s}}\right]\nonumber \\
 & =\frac{4D_{\psi}(x,x^{1})}{\sum_{t=1}^{T}\gamma_{t}}+4\left(M^{2}+\sigma^{2}\left(1+2\log\frac{2}{\delta}\right)\right)\sum_{t=1}^{T}\frac{\gamma_{t}\eta_{t}}{\sum_{s=t}^{T}\gamma_{s}}.\label{eq:str-h-hp-general}
\end{align}
where $\gamma_{t\in\left[T\right]}=\eta_{t}\prod_{s=2}^{t}\frac{1+\mu_{h}\eta_{s-1}}{1-\mu_{f}\eta_{s}}=\eta_{t}\prod_{s=2}^{t}(1+\mu_{h}\eta_{s-1})$.
Now let us check the condition of $\eta_{t\in\left[T\right]}\leq\frac{1}{2L}$
for the two choices respectively:
\begin{align*}
\eta_{t\in\left[T\right]} & =\frac{2}{\mu_{h}(t+4\kappa_{h})}\leq\frac{1}{2\kappa_{h}\mu_{h}}=\frac{1}{2L};\\
\eta_{t\in\left[T\right]} & =\begin{cases}
\frac{1}{\mu_{h}(\eta+2\kappa_{h})}\leq\frac{1}{2\kappa_{h}\mu_{h}}=\frac{1}{2L} & t\leq\tau\\
\frac{2}{\mu_{h}(t-\tau+4\kappa_{h})}\leq\frac{1}{2\kappa_{h}\mu_{h}}=\frac{1}{2L} & t\geq\tau+1
\end{cases}.
\end{align*}
Therefore, (\ref{eq:str-h-hp-general}) holds for both cases.

If $\eta_{t\in\left[T\right]}=\frac{2}{\mu_{h}(t+4\kappa_{h})}$,
following the similar steps in the proof of (\ref{eq:str-h-exp-1}),
we can finally get
\[
F(x^{T+1})-F(x)\leq O\left(\frac{\mu_{h}(1+\kappa_{h})^{2}D_{\psi}(x,x^{1})}{T(T+\kappa_{h})}+\frac{(M^{2}+\sigma^{2}\log\frac{1}{\delta})\log T}{\mu_{h}(T+\kappa_{h})}\right).
\]

If $\eta_{t\in\left[T\right]}=\begin{cases}
\frac{1}{\mu_{h}(\eta+2\kappa_{h})} & t\leq\tau\\
\frac{2}{\mu_{h}(t-\tau+4\kappa_{h})} & t\geq\tau+1
\end{cases}$, following the similar steps in the proof of (\ref{eq:str-h-exp-4}),
we can finally get
\[
F(x^{T+1})-F(x)\leq O\left(\frac{\mu_{h}D_{\psi}(x,x^{1})}{\exp\left(\frac{T}{2(1+\eta+2\kappa_{h})}\right)-1}+\left(\frac{1}{\eta+2\kappa_{h}}\lor\log T\right)\frac{M^{2}+\sigma^{2}\log\frac{1}{\delta}}{\mu_{h}(T+\kappa_{h})}\right).
\]
\end{proof}

\begin{rem}
The $O(\log T)$ factor in Theorems \ref{thm:str-h-exp} and \ref{thm:str-h-hp}
could be removed with more sophisticated step sizes, similar to those
in Theorems \ref{thm:str-f-exp-optimal} and \ref{thm:str-f-hp-optimal}.
However, we choose not to include them to keep the paper concise.
\end{rem}

\section{Unified Theoretical Analysis under Heavy-Tailed Noise\label{sec:heavy-analysis}}

In this section, we present the analysis under heavy-tailed noise.
We recall that the mirror map $\psi$ is assumed to be $(1,\frac{p}{p-1})$
uniformly convex now, i.e.,
\[
D_{\psi}(x,y)=\psi(x)-\psi(y)-\left\langle \na\psi(y),x-y\right\rangle \geq\frac{p-1}{p}\left\Vert x-y\right\Vert ^{\frac{p}{p-1}},\forall x,y\in\dom.
\]

We first prove the following general lemma. Unlike Lemma \ref{lem:core-general},
the requirement of $\eta_{t\in\left[T\right]}\leq\frac{1}{2L}$ is
removed.
\begin{lem}
\label{lem:heavy-core-general}Under Assumptions \ref{enu:A2} and
\ref{enu:A3} with $\mu_{f}=\mu_{h}=0$, let $v_{t\in\left\{ 0\right\} \cup\left[T\right]}>0$
be defined as $v_{t}\coloneqq\frac{\eta_{T}}{\sum_{s=t\lor1}^{T}\eta_{s}}$,
then for any $x\in\dom$, we have
\begin{align*}
\eta_{T}v_{T}\left(F(x^{T+1})-F(x)\right)\leq & v_{0}D_{\psi}(x,x^{1})+\sum_{t=1}^{T}\eta_{t}v_{t-1}\left\langle \xi^{t},z^{t-1}-x^{t}\right\rangle \\
 & +\sum_{t=1}^{T}\frac{\eta_{t}^{\frac{p}{2-p}}v_{t}(2-p)2^{\frac{3p-4}{2-p}}L^{\frac{p}{2-p}}+\eta_{t}^{p}v_{t}4^{p-1}(M^{p}+\left\Vert \xi^{t}\right\Vert _{*}^{p})}{p},
\end{align*}
where $\xi^{t}\coloneqq\hg^{t}-\E\left[\hg^{t}\mid\F^{t-1}\right],\forall t\in\left[T\right]$
and $z^{t}\coloneqq\frac{v_{0}}{v_{t}}x+\sum_{s=1}^{t}\frac{v_{s}-v_{s-1}}{v_{t}}x^{s},\forall t\in\left\{ 0\right\} \cup\left[T\right]$.
\end{lem}

\begin{proof}
The proof is similar to the proof of Lemma \ref{lem:core-general}.
Again, we use the sequence
\[
z^{t}=\begin{cases}
\left(1-\frac{v_{t-1}}{v_{t}}\right)x^{t}+\frac{v_{t-1}}{v_{t}}z^{t-1} & t\in\left[T\right]\\
x & t=0
\end{cases}\Leftrightarrow z^{t}=\frac{v_{0}}{v_{t}}x+\sum_{s=1}^{t}\frac{v_{s}-v_{s-1}}{v_{t}}x^{s},\forall t\in\left\{ 0\right\} \cup\left[T\right].
\]

As before, we start the proof from the $(L,M)$-smoothness of $f$:
\begin{align}
f(x^{t+1})-f(x^{t})\leq & \left\langle g^{t},x^{t+1}-x^{t}\right\rangle +\frac{L}{2}\left\Vert x^{t+1}-x^{t}\right\Vert ^{2}+M\left\Vert x^{t+1}-x^{t}\right\Vert \nonumber \\
= & \left\langle \xi^{t},z^{t}-x^{t}\right\rangle +\underbrace{\left\langle \xi^{t},x^{t}-x^{t+1}\right\rangle }_{\text{I}}+\underbrace{\left\langle \hg^{t},x^{t+1}-z^{t}\right\rangle }_{\text{II}}\nonumber \\
 & +\underbrace{\left\langle g^{t},z^{t}-x^{t}\right\rangle }_{\text{III}}+\underbrace{\frac{L}{2}\left\Vert x^{t+1}-x^{t}\right\Vert ^{2}+M\left\Vert x^{t+1}-x^{t}\right\Vert }_{\text{IV}},\label{eq:heavy-core-general-1}
\end{align}
where $g^{t}\coloneqq\E\left[\hg^{t}\vert\F^{t-1}\right]\in\partial f(x^{t})$
and $\xi^{t}=\hg^{t}-g^{t}$. Next we bound these four terms respectively.
\begin{itemize}
\item For term I, by applying Cauchy-Schwarz inequality, Young's inequality
and the $(1,\frac{p}{p-1})$ uniform convexity of $\psi$ (see (\ref{eq:uniformly-cvx})),
we can get the following upper bound
\begin{align}
\text{I} & \leq\left\Vert \xi^{t}\right\Vert _{*}\left\Vert x^{t}-x^{t+1}\right\Vert =(4\eta_{t})^{\frac{p-1}{p}}\left\Vert \xi^{t}\right\Vert _{*}\cdot\frac{\left\Vert x^{t}-x^{t+1}\right\Vert }{(4\eta_{t})^{\frac{p-1}{p}}}\nonumber \\
 & \leq\frac{(4\eta_{t})^{p-1}\left\Vert \xi^{t}\right\Vert _{*}^{p}}{p}+\frac{\left\Vert x^{t}-x^{t+1}\right\Vert ^{\frac{p}{p-1}}}{\frac{p}{p-1}\cdot4\eta_{t}}\leq\frac{\eta_{t}^{p-1}4^{p-1}\left\Vert \xi^{t}\right\Vert _{*}^{p}}{p}+\frac{D_{\psi}(x^{t+1},x^{t})}{4\eta_{t}}.\label{eq:heavy-core-general-I}
\end{align}
\item For term II, the same as (\ref{eq:core-general-II}) with $\mu_{h}=0$,
we can obtain
\begin{equation}
\text{II}\leq\frac{D_{\psi}(z^{t},x^{t})-D_{\psi}(z^{t},x^{t+1})-D_{\psi}(x^{t+1},x^{t})}{\eta_{t}}+h(z_{t})-h(x_{t+1}).\label{eq:heavy-core-general-II}
\end{equation}
\item For term III, we simply use the convexity of $f$ to get
\begin{equation}
\text{III}\leq f(z^{t})-f(x^{t}).\label{eq:heavy-core-general-III}
\end{equation}
\item For term IV, we have
\begin{align}
\text{IV} & =\frac{L}{2}(4\eta_{t})^{\frac{2(p-1)}{p}}\cdot\frac{\left\Vert x^{t+1}-x^{t}\right\Vert ^{2}}{(4\eta_{t})^{\frac{2(p-1)}{p}}}+(4\eta_{t})^{\frac{p-1}{p}}M\cdot\frac{\left\Vert x^{t+1}-x^{t}\right\Vert }{(4\eta_{t})^{\frac{p-1}{p}}}\nonumber \\
 & \leq\frac{\left[\frac{L}{2}(4\eta_{t})^{\frac{2(p-1)}{p}}\right]^{\frac{p}{2-p}}}{\frac{p}{2-p}}+\frac{\left\Vert x^{t+1}-x^{t}\right\Vert ^{\frac{p}{p-1}}}{\frac{p}{2(p-1)}\cdot4\eta_{t}}+\frac{(4\eta_{t})^{p-1}M^{p}}{p}+\frac{\left\Vert x^{t+1}-x^{t}\right\Vert ^{\frac{p}{p-1}}}{\frac{p}{p-1}\cdot4\eta_{t}}\nonumber \\
 & \leq\frac{\eta_{t}^{\frac{2(p-1)}{2-p}}(2-p)2^{\frac{3p-4}{2-p}}L^{\frac{p}{2-p}}}{p}+\frac{\eta_{t}^{p-1}4^{p-1}M^{p}}{p}+\frac{3D_{\psi}(x^{t+1},x^{t})}{4\eta_{t}},\label{eq:heavy-core-general-IV}
\end{align}
where the second line is due to Young's inequality and the last inequality
holds by the $(1,\frac{p}{p-1})$ uniform convexity of $\psi$.
\end{itemize}
By plugging the bounds (\ref{eq:heavy-core-general-I}), (\ref{eq:heavy-core-general-II}),
(\ref{eq:heavy-core-general-III}) and (\ref{eq:heavy-core-general-IV})
into (\ref{eq:heavy-core-general-1}), we obtain
\begin{align*}
f(x^{t+1})-f(x^{t})\leq & \left\langle \xi^{t},z^{t}-x^{t}\right\rangle +\frac{\eta_{t}^{p-1}4^{p-1}\left\Vert \xi^{t}\right\Vert _{*}^{p}}{p}+\frac{D_{\psi}(x^{t+1},x^{t})}{4\eta_{t}}\\
 & +\frac{D_{\psi}(z^{t},x^{t})-D_{\psi}(z^{t},x^{t+1})-D_{\psi}(x^{t+1},x^{t})}{\eta_{t}}+h(z_{t})-h(x_{t+1})\\
 & +f(z^{t})-f(x^{t})+\frac{\eta_{t}^{\frac{2(p-1)}{2-p}}(2-p)2^{\frac{3p-4}{2-p}}L^{\frac{p}{2-p}}}{p}+\frac{\eta_{t}^{p-1}4^{p-1}M^{p}}{p}+\frac{3D_{\psi}(x^{t+1},x^{t})}{4\eta_{t}}.
\end{align*}
Rearranging the terms to get
\begin{align}
F(x^{t+1})-F(z^{t})\leq & \left\langle \xi^{t},z^{t}-x^{t}\right\rangle +\frac{D_{\psi}(z^{t},x^{t})-D_{\psi}(z^{t},x^{t+1})}{\eta_{t}}\nonumber \\
 & +\frac{\eta_{t}^{\frac{2(p-1)}{2-p}}(2-p)L^{\frac{p}{2-p}}2^{\frac{3p-4}{2-p}}+\eta_{t}^{p-1}4^{p-1}(M^{p}+\left\Vert \xi^{t}\right\Vert _{*}^{p})}{p}\nonumber \\
\leq & \frac{v_{t-1}}{v_{t}}\left\langle \xi^{t},z^{t-1}-x^{t}\right\rangle +\frac{\frac{v_{t-1}}{v_{t}}D_{\psi}(z^{t-1},x^{t})-D_{\psi}(z^{t},x^{t+1})}{\eta_{t}}\nonumber \\
 & +\frac{\eta_{t}^{\frac{2(p-1)}{2-p}}(2-p)2^{\frac{3p-4}{2-p}}L^{\frac{p}{2-p}}+\eta_{t}^{p-1}4^{p-1}(M^{p}+\left\Vert \xi^{t}\right\Vert _{*}^{p})}{p},\label{eq:heavy-core-general-2}
\end{align}
where the last line is again by $z^{t}-x^{t}=\frac{v_{t-1}}{v_{t}}\left(z^{t-1}-x^{t}\right)$
and $D_{\psi}(z^{t},x^{t})\leq\frac{v_{t-1}}{v_{t}}D_{\psi}(z^{t-1},x^{t})$.

Multiplying both sides of (\ref{eq:heavy-core-general-2}) by $\eta_{t}v_{t}$
(these two terms are non-negative) and summing up from $t=1$ to $T$,
we obtain
\begin{align}
\sum_{t=1}^{T}\eta_{t}v_{t}\left(F(x^{t+1})-F(z^{t})\right)\leq & \sum_{t=1}^{T}\eta_{t}v_{t-1}\left\langle \xi^{t},z^{t-1}-x^{t}\right\rangle +\sum_{t=1}^{T}v_{t-1}D_{\psi}(z^{t-1},x^{t})-v_{t}D_{\psi}(z^{t},x^{t+1})\nonumber \\
 & +\sum_{t=1}^{T}\frac{\eta_{t}^{\frac{p}{2-p}}v_{t}(2-p)2^{\frac{3p-4}{2-p}}L^{\frac{p}{2-p}}+\eta_{t}^{p}v_{t}4^{p-1}(M^{p}+\left\Vert \xi^{t}\right\Vert _{*}^{p})}{p}\nonumber \\
= & v_{0}D_{\psi}(x,x^{1})-v_{T}D_{\psi}(z^{T},x^{T+1})+\sum_{t=1}^{T}\eta_{t}v_{t-1}\left\langle \xi^{t},z^{t-1}-x^{t}\right\rangle \nonumber \\
 & +\sum_{t=1}^{T}\frac{\eta_{t}^{\frac{p}{2-p}}v_{t}(2-p)2^{\frac{3p-4}{2-p}}L^{\frac{p}{2-p}}+\eta_{t}^{p}v_{t}4^{p-1}(M^{p}+\left\Vert \xi^{t}\right\Vert _{*}^{p})}{p}.\label{eq:heavy-core-general-3}
\end{align}
Similar to the proof of Lemma \ref{lem:core-general}, we can prove
\begin{equation}
\sum_{t=1}^{T}\eta_{t}v_{t}\left(F(x^{t+1})-F(z^{t})\right)\geq\eta_{T}v_{T}\left(F(x^{T+1})-F(x)\right).\label{eq:heavy-core-general-4}
\end{equation}
Combining (\ref{eq:heavy-core-general-3}) and (\ref{eq:heavy-core-general-4}),
there is
\begin{align*}
\eta_{T}v_{T}\left(F(x^{T+1})-F(x)\right)\leq & v_{0}D_{\psi}(x,x^{1})+\sum_{t=1}^{T}\eta_{t}v_{t-1}\left\langle \xi^{t},z^{t-1}-x^{t}\right\rangle \\
 & +\sum_{t=1}^{T}\frac{\eta_{t}^{\frac{p}{2-p}}v_{t}(2-p)2^{\frac{3p-4}{2-p}}L^{\frac{p}{2-p}}+\eta_{t}^{p}v_{t}4^{p-1}(M^{p}+\left\Vert \xi^{t}\right\Vert _{*}^{p})}{p}.
\end{align*}
\end{proof}

Now we are ready to state the general in-expectation convergence result
in Lemma \ref{lem:heavy-core-exp}.
\begin{lem}
\label{lem:heavy-core-exp}Under Assumptions \ref{enu:A2}-\ref{enu:A4}
and \ref{enu:A5C} with $\mu_{f}=\mu_{h}=0$, then for any $x\in\dom$,
we have
\[
\E\left[F(x^{T+1})-F(x)\right]\leq\frac{D_{\psi}(x,x^{1})}{\sum_{t=1}^{T}\eta_{t}}+\frac{(2-p)2^{\frac{3p-4}{2-p}}L^{\frac{p}{2-p}}}{p}\sum_{t=1}^{T}\frac{\eta_{t}^{\frac{p}{2-p}}}{\sum_{s=t}^{T}\eta_{s}}+\frac{4^{p-1}(M^{p}+\sigma^{p})}{p}\sum_{t=1}^{T}\frac{\eta_{t}^{p}}{\sum_{s=t}^{T}\eta_{s}}.
\]
\end{lem}

\begin{proof}
We invoke Lemma \ref{lem:heavy-core-general} to get
\begin{align*}
\eta_{T}v_{T}\left(F(x^{T+1})-F(x)\right)\leq & v_{0}D_{\psi}(x,x^{1})+\sum_{t=1}^{T}\eta_{t}v_{t-1}\left\langle \xi^{t},z^{t-1}-x^{t}\right\rangle \\
 & +\sum_{t=1}^{T}\frac{\eta_{t}^{\frac{p}{2-p}}v_{t}(2-p)2^{\frac{3p-4}{2-p}}L^{\frac{p}{2-p}}+\eta_{t}^{p}v_{t}4^{p-1}(M^{p}+\left\Vert \xi^{t}\right\Vert _{*}^{p})}{p}.
\end{align*}
Take expectations on both sides to obtain
\begin{align*}
\eta_{T}v_{T}\E\left[F(x^{T+1})-F(x)\right]\leq & v_{0}D_{\psi}(x,x^{1})+\sum_{t=1}^{T}\eta_{t}v_{t-1}\E\left[\left\langle \xi^{t},z^{t-1}-x^{t}\right\rangle \right]\\
 & +\sum_{t=1}^{T}\frac{\eta_{t}^{\frac{p}{2-p}}v_{t}(2-p)2^{\frac{3p-4}{2-p}}L^{\frac{p}{2-p}}+\eta_{t}^{p}v_{t}4^{p-1}(M^{p}+\E\left[\left\Vert \xi^{t}\right\Vert _{*}^{p}\right])}{p}\\
\leq & v_{0}D_{\psi}(x,x^{1})+\sum_{t=1}^{T}\frac{\eta_{t}^{\frac{p}{2-p}}v_{t}(2-p)2^{\frac{3p-4}{2-p}}L^{\frac{p}{2-p}}+\eta_{t}^{p}v_{t}4^{p-1}(M^{p}+\sigma^{p})}{p},
\end{align*}
where the last line is due to $\E\left[\left\Vert \xi^{t}\right\Vert _{*}^{p}\right]\leq\sigma^{p}$
(Assumption \ref{enu:A5C}) and $\E\left[\left\langle \xi^{t},z^{t-1}-x^{t}\right\rangle \right]=\E\left[\left\langle \E\left[\xi^{t}\vert\F^{t-1}\right],z^{t-1}-x^{t}\right\rangle \right]=0$
($z^{t-1}-x^{t}\in\F^{t-1}$ and Assumption \ref{enu:A4}). Finally,
we divide both sides by $\eta_{T}v_{T}$ and plug in $v_{t\in\left\{ 0\right\} \cup\left[T\right]}=\frac{\eta_{T}}{\sum_{s=t\lor1}^{T}\eta_{s}}$
to finish the proof.
\end{proof}

\section{General Convex Functions under Heavy-Tailed Noise\label{sec:heavy-cvx}}

In this section, we prove the full version of Theorem \ref{thm:main-heavy-cvx-exp}
by using Lemma \ref{lem:heavy-core-exp},
\begin{thm}[Full version of Theorem \ref{thm:main-heavy-cvx-exp}]
\label{thm:heavy-cvx-exp}Under Assumptions \ref{enu:A2}-\ref{enu:A4}
and \ref{enu:A5C} with $\mu_{f}=\mu_{h}=0$, for any $x\in\dom$:

If $T$ is unknown, by taking $\eta_{t\in\left[T\right]}=\frac{\eta_{*}}{t^{\frac{2-p}{p}}}\land\frac{\eta}{t^{\frac{1}{p}}}$,
there is
\begin{align*}
\E\left[F(x^{T+1})-F(x)\right]\leq & O\left(\frac{1}{T^{2-\frac{2}{p}}}\left(\frac{D_{\psi}(x,x^{1})}{\eta_{*}}+\eta_{*}^{\frac{2p-2}{2-p}}L^{\frac{p}{2-p}}\log T+\frac{\eta^{p}(M^{p}+\sigma^{p})}{\eta_{*}}\log T\right)\right.\\
 & +\left.\frac{1}{T^{1-\frac{1}{p}}}\left(\frac{D_{\psi}(x,x^{1})}{\eta}+\eta^{p-1}(M^{p}+\sigma^{p})\log T+\frac{\eta_{*}^{\frac{p}{2-p}}L^{\frac{p}{2-p}}}{\eta}\log T\right)\right).
\end{align*}
In particular, by choosing $\eta_{*}=\Theta\left(\frac{\left[D_{\psi}(x,x^{1})\right]^{\frac{2-p}{p}}}{L}\right)$
and $\eta=\Theta\left(\left[\frac{D_{\psi}(x,x^{1})}{M^{p}+\sigma^{p}}\right]^{\frac{1}{p}}\right)$,
there is
\[
\E\left[F(x^{T+1})-F(x)\right]\leq O\left(\frac{L\left[D_{\psi}(x,x^{1})\right]^{2-\frac{2}{p}}\log T}{T^{2-\frac{2}{p}}}+\frac{(M+\sigma)\left[D_{\psi}(x,x^{1})\right]^{1-\frac{1}{p}}\log T}{T^{1-\frac{1}{p}}}\right).
\]

If $T$ is known, by taking $\eta_{t\in\left[T\right]}=\frac{\eta_{*}}{T^{\frac{2-p}{p}}}\land\frac{\eta}{T^{\frac{1}{p}}}$,
there is
\begin{align*}
\E\left[F(x^{T+1})-F(x)\right]\leq & O\left(\frac{1}{T^{2-\frac{2}{p}}}\left(\frac{D_{\psi}(x,x^{1})}{\eta_{*}}+\eta_{*}^{\frac{2p-2}{2-p}}L^{\frac{p}{2-p}}\log T\right)\right.\\
 & +\left.\frac{1}{T^{1-\frac{1}{p}}}\left(\frac{D_{\psi}(x,x^{1})}{\eta}+\eta^{p-1}(M^{p}+\sigma^{p})\log T\right)\right).
\end{align*}
In particular, by choosing $\eta_{*}=\Theta\left(\frac{1}{L}\left[\frac{D_{\psi}(x,x^{1})}{\log T}\right]^{\frac{2-p}{p}}\right)$
and $\eta=\Theta\left(\left[\frac{D_{\psi}(x,x^{1})}{(M^{p}+\sigma^{p})\log T}\right]^{\frac{1}{p}}\right)$,
there is
\[
\E\left[F(x^{T+1})-F(x)\right]\leq O\left(\frac{L\left[D_{\psi}(x,x^{1})\right]^{2-\frac{2}{p}}(\log T)^{\frac{2}{p}-1}}{T^{2-\frac{2}{p}}}+\frac{(M+\sigma)\left[D_{\psi}(x,x^{1})\right]^{1-\frac{1}{p}}(\log T)^{\frac{1}{p}}}{T^{1-\frac{1}{p}}}\right).
\]
\end{thm}

\begin{proof}
By Lemma \ref{lem:heavy-core-exp}, there is
\[
\E\left[F(x^{T+1})-F(x)\right]\leq\frac{D_{\psi}(x,x^{1})}{\sum_{t=1}^{T}\eta_{t}}+\frac{(2-p)2^{\frac{3p-4}{2-p}}L^{\frac{p}{2-p}}}{p}\sum_{t=1}^{T}\frac{\eta_{t}^{\frac{p}{2-p}}}{\sum_{s=t}^{T}\eta_{s}}+\frac{4^{p-1}(M^{p}+\sigma^{p})}{p}\sum_{t=1}^{T}\frac{\eta_{t}^{p}}{\sum_{s=t}^{T}\eta_{s}}.
\]
Hence, we only need to plug in different step sizes. Before doing
so, we first prove the following simple fact for any $a\in(0,1)$,
\begin{equation}
\sum_{t=1}^{T}\frac{\sum_{s=t}^{T}s^{a}}{t(T-t+1)^{2}}\leq\sum_{t=1}^{T}\frac{T^{a}}{t(T-t+1)}=\frac{1}{T^{1-a}}\sum_{t=1}^{T}\frac{1}{t}+\frac{1}{T-t+1}\leq\frac{2(1+\log T)}{T^{1-a}}.\label{eq:heavy-cvx-exp-ineq}
\end{equation}

If $\eta_{t\in\left[T\right]}=\frac{\eta_{*}}{t^{\frac{2-p}{p}}}\land\frac{\eta}{t^{\frac{1}{p}}}$,
we have
\begin{align*}
\E\left[F(x^{T+1})-F(x)\right]\leq & \frac{D_{\psi}(x,x^{1})}{T^{2}}\left(\sum_{t=1}^{T}\frac{1}{\eta_{t}}\right)+\frac{(2-p)2^{\frac{3p-4}{2-p}}L^{\frac{p}{2-p}}}{p}\sum_{t=1}^{T}\frac{\eta_{t}^{\frac{p}{2-p}}}{(T-t+1)^{2}}\left(\sum_{s=t}^{T}\frac{1}{\eta_{s}}\right)\\
 & +\frac{4^{p-1}(M^{p}+\sigma^{p})}{p}\sum_{t=1}^{T}\frac{\eta_{t}^{p}}{(T-t+1)^{2}}\left(\sum_{s=t}^{T}\frac{1}{\eta_{s}}\right).
\end{align*}
We can bound
\begin{align*}
\sum_{t=1}^{T}\frac{1}{\eta_{t}} & \leq\sum_{t=1}^{T}\frac{t^{\frac{2-p}{p}}}{\eta_{*}}+\frac{t^{\frac{1}{p}}}{\eta}\leq\frac{T^{\frac{2-p}{p}}+\int_{1}^{T}t^{\frac{2-p}{p}}\mathrm{d}t}{\eta_{*}}+\frac{T^{\frac{1}{p}}+\int_{1}^{T}t^{\frac{1}{p}}\mathrm{d}t}{\eta}\\
 & \leq\frac{T^{\frac{2-p}{p}}+\frac{p}{2}T^{\frac{2}{p}}}{\eta_{*}}+\frac{T^{\frac{1}{p}}+\frac{p}{p+1}T^{\frac{1}{p}+1}}{\eta}\overset{(a)}{\leq}\frac{T^{\frac{2}{p}}}{\eta_{*}}+\frac{5T^{\frac{1}{p}+1}}{3\eta},
\end{align*}
where $(a)$ is by $p\leq2$ and $T\geq1$. For the other two terms,
we have
\begin{align*}
\sum_{t=1}^{T}\frac{\eta_{t}^{\frac{p}{2-p}}}{(T-t+1)^{2}}\left(\sum_{s=t}^{T}\frac{1}{\eta_{s}}\right) & \leq\sum_{t=1}^{T}\frac{\eta_{t}^{\frac{p}{2-p}}}{(T-t+1)^{2}}\left(\sum_{s=t}^{T}\frac{s^{\frac{2-p}{p}}}{\eta_{*}}+\frac{s^{\frac{1}{p}}}{\eta}\right)\\
 & \leq\eta_{*}^{\frac{2p-2}{2-p}}\sum_{t=1}^{T}\frac{\sum_{s=t}^{T}s^{\frac{2-p}{p}}}{t(T-t+1)^{2}}+\frac{\eta_{*}^{\frac{p}{2-p}}}{\eta}\sum_{t=1}^{T}\frac{\sum_{s=t}^{T}s^{\frac{1}{p}}}{t(T-t+1)^{2}}\\
 & \overset{(\ref{eq:heavy-cvx-exp-ineq})}{\leq}\eta_{*}^{\frac{2p-2}{2-p}}\frac{2(1+\log T)}{T^{2-\frac{2}{p}}}+\frac{\eta_{*}^{\frac{p}{2-p}}}{\eta}\cdot\frac{2(1+\log T)}{T^{1-\frac{1}{p}}},
\end{align*}
\begin{align*}
\sum_{t=1}^{T}\frac{\eta_{t}^{p}}{(T-t+1)^{2}}\left(\sum_{s=t}^{T}\frac{1}{\eta_{s}}\right) & \leq\sum_{t=1}^{T}\frac{\eta_{t}^{p}}{(T-t+1)^{2}}\left(\sum_{s=t}^{T}\frac{s^{\frac{2-p}{p}}}{\eta_{*}}+\frac{s^{\frac{1}{p}}}{\eta}\right)\\
 & \leq\frac{\eta^{p}}{\eta_{*}}\sum_{t=1}^{T}\frac{\sum_{s=t}^{T}s^{\frac{2-p}{p}}}{t(T-t+1)^{2}}+\eta^{p-1}\sum_{t=1}^{T}\frac{\sum_{s=t}^{T}s^{\frac{1}{p}}}{t(T-t+1)^{2}}\\
 & \overset{(\ref{eq:heavy-cvx-exp-ineq})}{\leq}\frac{\eta^{p}}{\eta_{*}}\cdot\frac{2(1+\log T)}{T^{2-\frac{2}{p}}}+\eta^{p-1}\frac{2(1+\log T)}{T^{1-\frac{1}{p}}}.
\end{align*}
Hence, we have
\begin{align*}
\E\left[F(x^{T+1})-F(x)\right]\leq & \frac{D_{\psi}(x,x^{1})}{\eta_{*}T^{2-\frac{2}{p}}}+\frac{5D_{\psi}(x,x^{1})}{3\eta T^{1-\frac{1}{p}}}\\
 & +\frac{(2-p)2^{\frac{3p-4}{2-p}}L^{\frac{p}{2-p}}}{p}\left[\eta_{*}^{\frac{2p-2}{2-p}}\frac{2(1+\log T)}{T^{2-\frac{2}{p}}}+\frac{\eta_{*}^{\frac{p}{2-p}}}{\eta}\cdot\frac{2(1+\log T)}{T^{1-\frac{1}{p}}}\right]\\
 & +\frac{4^{p-1}(M^{p}+\sigma^{p})}{p}\left[\frac{\eta^{p}}{\eta_{*}}\cdot\frac{2(1+\log T)}{T^{2-\frac{2}{p}}}+\eta^{p-1}\frac{2(1+\log T)}{T^{1-\frac{1}{p}}}\right]\\
= & O\left(\frac{1}{T^{2-\frac{2}{p}}}\left(\frac{D_{\psi}(x,x^{1})}{\eta_{*}}+\eta_{*}^{\frac{2p-2}{2-p}}L^{\frac{p}{2-p}}\log T+\frac{\eta^{p}(M^{p}+\sigma^{p})}{\eta_{*}}\log T\right)\right.\\
 & +\left.\frac{1}{T^{1-\frac{1}{p}}}\left(\frac{D_{\psi}(x,x^{1})}{\eta}+\eta^{p-1}(M^{p}+\sigma^{p})\log T+\frac{\eta_{*}^{\frac{p}{2-p}}L^{\frac{p}{2-p}}}{\eta}\log T\right)\right).
\end{align*}
By plugging in $\eta_{*}=\Theta\left(\frac{\left[D_{\psi}(x,x^{1})\right]^{\frac{2-p}{p}}}{L}\right)$
and $\eta=\Theta\left(\left[\frac{D_{\psi}(x,x^{1})}{M^{p}+\sigma^{p}}\right]^{\frac{1}{p}}\right)$,
we get the desired bound.

If $\eta_{t\in\left[T\right]}=\frac{\eta_{*}}{T^{\frac{2-p}{p}}}\land\frac{\eta}{T^{\frac{1}{p}}}$,
we will obtain 
\begin{align*}
\E\left[F(x^{T+1})-F(x)\right]\leq & \frac{D_{\psi}(x,x^{1})}{T}\left(\frac{T^{\frac{2-p}{p}}}{\eta_{*}}\lor\frac{T^{\frac{1}{p}}}{\eta}\right)+\frac{(2-p)2^{\frac{3p-4}{2-p}}L^{\frac{p}{2-p}}}{p}\left(\frac{\eta_{*}}{T^{\frac{2-p}{p}}}\land\frac{\eta}{T^{\frac{1}{p}}}\right)^{\frac{2p-2}{2-p}}\sum_{t=1}^{T}\frac{1}{T-t+1}\\
 & +\frac{4^{p-1}(M^{p}+\sigma^{p})}{p}\left(\frac{\eta_{*}}{T^{\frac{2-p}{p}}}\land\frac{\eta}{T^{\frac{1}{p}}}\right)^{p-1}\sum_{t=1}^{T}\frac{1}{T-t+1}\\
\leq & \frac{1}{T^{\frac{2p-2}{p}}}\left[\frac{D_{\psi}(x,x^{1})}{\eta_{*}}+\frac{\eta_{*}^{\frac{2p-2}{2-p}}(2-p)2^{\frac{3p-4}{2-p}}L^{\frac{p}{2-p}}\log T}{p}\right]\\
 & +\frac{1}{T^{2-\frac{2}{p}}}\left[\frac{D_{\psi}(x,x^{1})}{\eta}+\frac{\eta^{p-1}4^{p-1}(M^{p}+\sigma^{p})\log T}{p}\right]\\
= & O\left(\frac{1}{T^{2-\frac{2}{p}}}\left(\frac{D_{\psi}(x,x^{1})}{\eta_{*}}+\eta_{*}^{\frac{2p-2}{2-p}}L^{\frac{p}{2-p}}\log T\right)\right.\\
 & +\left.\frac{1}{T^{1-\frac{1}{p}}}\left(\frac{D_{\psi}(x,x^{1})}{\eta}+\eta^{p-1}(M^{p}+\sigma^{p})\log T\right)\right).
\end{align*}
By plugging in $\eta_{*}=\Theta\left(\frac{1}{L}\left[\frac{D_{\psi}(x,x^{1})}{\log T}\right]^{\frac{2-p}{p}}\right)$
and $\eta=\Theta\left(\left[\frac{D_{\psi}(x,x^{1})}{(M^{p}+\sigma^{p})\log T}\right]^{\frac{1}{p}}\right)$,
we get the desired bound.
\end{proof}

\subsection{Optimal Rate under Heavy-Tailed Noise\label{subsec:heavy-optimal}}

The full version of Theorem \ref{thm:main-heavy-cvx-exp-optimal}
is stated below.
\begin{thm}[Full version of Theorem \ref{thm:main-heavy-cvx-exp-optimal}]
\label{thm:heavy-cvx-exp-optimal}Under Assumptions \ref{enu:A2}-\ref{enu:A4}
and \ref{enu:A5C} with $\mu_{f}=\mu_{h}=0$, for any $x\in\dom$,
if $T$ is known, by taking $\eta_{t\in\left[T\right]}=\frac{\eta_{*}(T-t+1)}{T^{\frac{2}{p}}}\land\frac{\eta(T-t+1)}{T^{\frac{1}{p}+1}}$,
there is
\[
\E\left[F(x^{T+1})-F(x)\right]\leq O\left(\frac{1}{T^{2-\frac{2}{p}}}\left(\frac{D_{\psi}(x,x^{1})}{\eta_{*}}+\eta_{*}^{\frac{2p-2}{2-p}}L^{\frac{p}{2-p}}\right)+\frac{1}{T^{1-\frac{1}{p}}}\left(\frac{D_{\psi}(x,x^{1})}{\eta}+\eta^{p-1}(M^{p}+\sigma^{p})\right)\right).
\]
In particular, by choosing $\eta_{*}=\Theta\left(\frac{\left[D_{\psi}(x,x^{1})\right]^{\frac{2-p}{p}}}{L}\right)$
and $\eta=\Theta\left(\left[\frac{D_{\psi}(x,x^{1})}{M^{p}+\sigma^{p}}\right]^{\frac{1}{p}}\right)$,
there is
\[
\E\left[F(x^{T+1})-F(x)\right]\leq O\left(\frac{L\left[D_{\psi}(x,x^{1})\right]^{2-\frac{2}{p}}}{T^{2-\frac{2}{p}}}+\frac{(M+\sigma)\left[D_{\psi}(x,x^{1})\right]^{1-\frac{1}{p}}}{T^{1-\frac{1}{p}}}\right).
\]
\end{thm}

\begin{proof}
By Lemma \ref{lem:heavy-core-exp}, there is
\[
\E\left[F(x^{T+1})-F(x)\right]\leq\frac{D_{\psi}(x,x^{1})}{\sum_{t=1}^{T}\eta_{t}}+\frac{(2-p)2^{\frac{3p-4}{2-p}}L^{\frac{p}{2-p}}}{p}\sum_{t=1}^{T}\frac{\eta_{t}^{\frac{p}{2-p}}}{\sum_{s=t}^{T}\eta_{s}}+\frac{4^{p-1}(M^{p}+\sigma^{p})}{p}\sum_{t=1}^{T}\frac{\eta_{t}^{p}}{\sum_{s=t}^{T}\eta_{s}}.
\]
For $\eta_{t\in\left[T\right]}=\frac{\eta_{*}(T-t+1)}{T^{\frac{2}{p}}}\land\frac{\eta(T-t+1)}{T^{\frac{1}{p}+1}}$,
we can compute
\[
\sum_{t=1}^{T}\eta_{t}=\left(\frac{\eta_{*}}{T^{\frac{2}{p}}}\land\frac{\eta}{T^{\frac{1}{p}+1}}\right)\sum_{t=1}^{T}T-t+1\geq\frac{\eta_{*}}{2}T^{2-\frac{2}{p}}\land\frac{\eta}{2}T^{1-\frac{1}{p}};
\]
\begin{align*}
\sum_{t=1}^{T}\frac{\eta_{t}^{\frac{p}{2-p}}}{\sum_{s=t}^{T}\eta_{s}} & =\left(\frac{\eta_{*}}{T^{\frac{2}{p}}}\land\frac{\eta}{T^{\frac{1}{p}+1}}\right)^{\frac{2p-2}{2-p}}\sum_{t=1}^{T}\frac{(T-t+1)^{\frac{p}{2-p}}}{\sum_{s=t}^{T}T-s+1}\leq\frac{2\eta_{*}^{\frac{2p-2}{2-p}}}{T^{\frac{4p-4}{p(2-p)}}}\sum_{t=1}^{T}(T-t+1)^{\frac{3p-4}{2-p}}\\
 & \leq\frac{2\eta_{*}^{\frac{2p-2}{2-p}}}{T^{\frac{4p-4}{p(2-p)}}}\begin{cases}
\int_{1}^{T+1}t^{\frac{3p-4}{2-p}}\mathrm{d}t & p\in\left[\frac{4}{3},2\right)\\
\int_{0}^{T}t^{\frac{3p-4}{2-p}}\mathrm{d}t & p\in\left(1,\frac{4}{3}\right)
\end{cases}\leq\frac{2^{\frac{2p-2}{2-p}}(2-p)\eta_{*}^{\frac{2p-2}{2-p}}}{(p-1)T^{2-\frac{2}{p}}};
\end{align*}
\begin{align*}
\sum_{t=1}^{T}\frac{\eta_{t}^{p}}{\sum_{s=t}^{T}\eta_{s}} & =\left(\frac{\eta_{*}}{T^{\frac{2}{p}}}\land\frac{\eta}{T^{\frac{1}{p}+1}}\right)^{p-1}\sum_{t=1}^{T}\frac{(T-t+1)^{p}}{\sum_{s=t}^{T}T-s+1}\leq\frac{2\eta^{p-1}}{T^{\frac{p^{2}-1}{p}}}\sum_{t=1}^{T}(T-t+1)^{p-2}\\
 & \leq\frac{2\eta^{p-1}}{T^{\frac{p^{2}-1}{p}}}\int_{0}^{T}t^{p-2}\mathrm{d}t=\frac{2\eta^{p-1}}{(p-1)T^{1-\frac{1}{p}}}.
\end{align*}
Hence, there is
\begin{align*}
\E\left[F(x^{T+1})-F(x)\right] & \leq\frac{2D_{\psi}(x,x^{1})}{\eta_{*}T^{2-\frac{2}{p}}}+\frac{2D_{\psi}(x,x^{1})}{\eta T^{1-\frac{1}{p}}}+\frac{(2-p)^{2}2^{\frac{5p-6}{2-p}}L^{\frac{p}{2-p}}\eta_{*}^{\frac{2p-2}{2-p}}}{p(p-1)T^{2-\frac{2}{p}}}+\frac{4^{p-\frac{1}{2}}(M^{p}+\sigma^{p})\eta^{p-1}}{p(p-1)T^{1-\frac{1}{p}}}\\
 & =O\left(\frac{1}{T^{2-\frac{2}{p}}}\left(\frac{D_{\psi}(x,x^{1})}{\eta_{*}}+\eta_{*}^{\frac{2p-2}{2-p}}L^{\frac{p}{2-p}}\right)+\frac{1}{T^{1-\frac{1}{p}}}\left(\frac{D_{\psi}(x,x^{1})}{\eta}+\eta^{p-1}(M^{p}+\sigma^{p})\right)\right).
\end{align*}
By plugging in $\eta_{*}=\Theta\left(\frac{\left[D_{\psi}(x,x^{1})\right]^{\frac{2-p}{p}}}{L}\right)$
and $\eta=\Theta\left(\left[\frac{D_{\psi}(x,x^{1})}{M^{p}+\sigma^{p}}\right]^{\frac{1}{p}}\right)$,
we get the desired bound.
\end{proof}

\section{Unified Theoretical Analysis under Sub-Weibull Noise\label{sec:weibull-analysis}}

Lemma \ref{lem:core-general-anytime} can be viewed as an any-time
version of Lemma \ref{lem:core-general}. However, this intriguing
any-time property only holds for the minimizer $x^{*}$ unlike Lemma
\ref{lem:core-general} being true for any $x\in\dom$. Besides, the
weight sequence $w_{t\in\left[T\right]}$ used in Lemma \ref{lem:core-general}
is dropped. The reason is that we need an almost surely bounded random
vector to use Lemma \ref{lem:weibull} (for the case $p=1$) to bound
$\gamma_{t}v_{t-1}\left\langle \xi^{t},z^{t-1}-x^{t}\right\rangle $.
However, it is unclear whether $\left\Vert \gamma_{t}v_{t-1}(z^{t-1}-x^{t})\right\Vert $
is bounded. Hence, the weight sequence technique of \citet{pmlr-v202-liu23aa}
may fail.
\begin{lem}
\label{lem:core-general-anytime}Under Assumptions \ref{enu:A1}-\ref{enu:A3},
suppose $\eta_{t\in\left[T\right]}\leq\frac{1}{2L\lor\mu_{f}}$ and
let $\gamma_{t\in\left[T\right]}\coloneqq\eta_{t}\prod_{s=2}^{t}\frac{1+\mu_{h}\eta_{s-1}}{1-\mu_{f}\eta_{s}}$
and $v_{t\in\left\{ 0\right\} \cup\left[T\right]}>0$ be defined as
$v_{t}\coloneqq\frac{\gamma_{T}}{\sum_{s=t\lor1}^{T}\gamma_{s}}$,
then for any $s\in\left[T\right]$, we have
\begin{align*}
 & \gamma_{s}v_{s}\left(F(x^{s+1})-F(x^{*})\right)+\gamma_{s+1}\eta_{s+1}^{-1}(1-\mu_{f}\eta_{s+1})v_{s}D_{\psi}(z^{s},x^{s+1})\\
\leq & (1-\mu_{f}\eta_{1})v_{0}D_{\psi}(x^{*},x^{1})+\sum_{t=1}^{s}2\gamma_{t}\eta_{t}v_{t}(M^{2}+\left\Vert \xi^{t}\right\Vert _{*}^{2})+\sum_{t=1}^{s}\gamma_{t}v_{t-1}\left\langle \xi^{t},z^{t-1}-x^{t}\right\rangle ,
\end{align*}
where $\xi^{t}\coloneqq\hg^{t}-\E\left[\hg^{t}\mid\F^{t-1}\right],\forall t\in\left[T\right]$
and $z^{t}\coloneqq\frac{v_{0}}{v_{t}}x^{*}+\sum_{s=1}^{t}\frac{v_{s}-v_{s-1}}{v_{t}}x^{s},\forall t\in\left\{ 0\right\} \cup\left[T\right]$.
\end{lem}

\begin{proof}
The proof is similar to the proof of Lemma \ref{lem:core-general}.
Again, we use the sequence
\[
z^{t}=\begin{cases}
\left(1-\frac{v_{t-1}}{v_{t}}\right)x^{t}+\frac{v_{t-1}}{v_{t}}z^{t-1} & t\in\left[T\right]\\
x^{*} & t=0
\end{cases}\Leftrightarrow z^{t}=\frac{v_{0}}{v_{t}}x^{*}+\sum_{s=1}^{t}\frac{v_{s}-v_{s-1}}{v_{t}}x^{s},\forall t\in\left\{ 0\right\} \cup\left[T\right].
\]

Following the same step when proving Lemma \ref{lem:core-general},
we have for any $t\geq1$
\begin{align}
F(x^{t+1})-F(z^{t})\leq & (\eta_{t}^{-1}-\mu_{f})\frac{v_{t-1}}{v_{t}}D_{\psi}(z^{t-1},x^{t})-(\eta_{t}^{-1}+\mu_{h})D_{\psi}(z^{t},x^{t+1})\nonumber \\
 & +2\eta_{t}(M^{2}+\left\Vert \xi^{t}\right\Vert _{*}^{2})+\frac{v_{t-1}}{v_{t}}\left\langle \xi^{t},z^{t-1}-x^{t}\right\rangle .\label{eq:core-general-anytime-1}
\end{align}
Multiplying both sides by $\gamma_{t}v_{t}$ (these two terms are
non-negative) and summing up from $t=1$ to $s$ where $s\in\left[T\right]$,
we obtain
\begin{align}
\sum_{t=1}^{s}\gamma_{t}v_{t}\left(F(x^{t+1})-F(z^{t})\right)\leq & (1-\mu_{f}\eta_{1})v_{0}D_{\psi}(x^{*},x^{1})-\gamma_{s}(\eta_{s}^{-1}+\mu_{h})v_{s}D_{\psi}(z^{s},x^{s+1})\nonumber \\
 & +\sum_{t=1}^{s}2\gamma_{t}\eta_{t}v_{t}(M^{2}+\left\Vert \xi^{t}\right\Vert _{*}^{2})+\sum_{t=1}^{s}\gamma_{t}v_{t-1}\left\langle \xi^{t},z^{t-1}-x^{t}\right\rangle .\label{eq:core-general-anytime-2}
\end{align}

Next, by the convexity of $F$ and the definition of $z^{t}$, we
can find
\begin{align*}
\sum_{t=1}^{s}\gamma_{t}v_{t}\left(F(x^{t+1})-F(z^{t})\right)\geq & \gamma_{s}v_{s}\left(F(x^{s+1})-F(x^{*})\right)-\left(\sum_{\ell=1}^{s}\gamma_{\ell}\right)\left(v_{1}-v_{0}\right)\left(F(x^{1})-F(x^{*})\right)\\
 & +\sum_{t=2}^{s}\left[\gamma_{t-1}v_{t-1}-\left(\sum_{\ell=t}^{s}\gamma_{\ell}\right)\left(v_{t}-v_{t-1}\right)\right]\left(F(x^{t})-F(x^{*})\right).
\end{align*}
Note that $v_{1}=v_{0}$ and for $2\leq t\leq s$ there is
\begin{align*}
\gamma_{t-1}v_{t-1}-\left(\sum_{\ell=t}^{s}\gamma_{\ell}\right)\left(v_{t}-v_{t-1}\right) & =\left(\sum_{\ell=t-1}^{s}\gamma_{\ell}\right)v_{t-1}-\left(\sum_{\ell=t}^{s}\gamma_{\ell}\right)v_{t}=\gamma_{T}\left(\frac{\sum_{\ell=t-1}^{s}\gamma_{\ell}}{\sum_{\ell=t-1}^{T}\gamma_{\ell}}-\frac{\sum_{\ell=t}^{s}\gamma_{\ell}}{\sum_{\ell=t}^{T}\gamma_{\ell}}\right)\\
 & =\frac{\gamma_{T}\gamma_{t-1}\left(\sum_{\ell=s+1}^{T}\gamma_{\ell}\right)}{\left(\sum_{\ell=t-1}^{T}\gamma_{\ell}\right)\left(\sum_{\ell=t}^{T}\gamma_{\ell}\right)}\geq0.
\end{align*}
Hence, by using $F(x^{t})-F(x^{*})\geq0$, we still have
\begin{equation}
\sum_{t=1}^{s}\gamma_{t}v_{t}\left(F(x^{t+1})-F(z^{t})\right)\geq\gamma_{s}v_{s}\left(F(x^{s+1})-F(x^{*})\right).\label{eq:core-general-anytime-3}
\end{equation}
Combining (\ref{eq:core-general-anytime-2}) and (\ref{eq:core-general-anytime-3}),
we conclude for any $s\in\left[T\right]$
\begin{align*}
 & \gamma_{s}v_{s}\left(F(x^{s+1})-F(x^{*})\right)+\gamma_{s}(\eta_{s}^{-1}+\mu_{h})v_{s}D_{\psi}(z^{s},x^{s+1})\\
\leq & (1-\mu_{f}\eta_{1})v_{0}D_{\psi}(x^{*},x^{1})+\sum_{t=1}^{s}2\gamma_{t}\eta_{t}v_{t}(M^{2}+\left\Vert \xi^{t}\right\Vert _{*}^{2})+\sum_{t=1}^{s}\gamma_{t}v_{t-1}\left\langle \xi^{t},z^{t-1}-x^{t}\right\rangle .
\end{align*}
The proof is finished by noticing $\gamma_{s}(\eta_{s}^{-1}+\mu_{h})=\gamma_{s+1}\eta_{s+1}^{-1}(1-\mu_{f}\eta_{s+1}),\forall s\in\left[T\right]$.
\end{proof}

As described above, our key problem now is to deal with the term $\sum_{t=1}^{s}\gamma_{t}v_{t-1}\left\langle \xi^{t},z^{t-1}-x^{t}\right\rangle $
(the term $\sum_{t=1}^{s}\gamma_{t}\eta_{t}v_{t}\left\Vert \xi^{t}\right\Vert _{*}^{2}$
can be bounded by Lemma \ref{lem:weibull-hp-sigma}). Inspired by
\citet{pmlr-v202-ivgi23a} and \citet{liu2023stochastic}, we will
introduce another sequence $y_{t\in\left[T\right]}$ (see (\ref{eq:weibull-core-hp-y})
for its definition). By a transformation, we only need to bound $\sqrt{y_{t}}\max_{\ell\in\left[s\right]}\left|\sum_{t=1}^{\ell}\gamma_{t}v_{t-1}\left\langle \xi^{t},\frac{z^{t-1}-x^{t}}{\sqrt{y_{t}}}\right\rangle \right|$
(this is done by Lemma C.2 in \citet{pmlr-v202-ivgi23a}). With the
carefully designed $y_{t}$, we will make sure that $\gamma_{t}v_{t-1}\frac{z^{t-1}-x^{t}}{\sqrt{y_{t}}}$
is bounded almost surely (the choice of $y_{t}$ is inspired by \citet{liu2023stochastic},
especially from their Theorem 6). Hence, the term $\left|\sum_{t=1}^{\ell}\gamma_{t}v_{t-1}\left\langle \xi^{t},\frac{z^{t-1}-x^{t}}{\sqrt{y_{t}}}\right\rangle \right|$
can be controlled in a high-probability way (see Lemma \ref{lem:weibull-hp-inner}).
Moreover, we use an induction-based proof like \citet{liu2023stochastic}
to obtain the following general lemma. We refer the interested reader
to our proof for details.

Note that Lemma \ref{lem:weibull-core-hp} can also be used to get
the high-probability convergence results for strongly convex objectives.
We omit that case for simplicity.
\begin{lem}
\label{lem:weibull-core-hp}Under Assumptions \ref{enu:A1}-\ref{enu:A4}
and \ref{enu:A5D}, suppose $\eta_{t\in\left[T\right]}\leq\frac{1}{2L\lor\mu_{f}}$
and let $\gamma_{t\in\left[T\right]}\coloneqq\eta_{t}\prod_{s=2}^{t}\frac{1+\mu_{h}\eta_{s-1}}{1-\mu_{f}\eta_{s}}$,
then for any $\delta\in(0,1)$, with probability at least $1-\delta$,
we have
\[
F(x^{T+1})-F(x^{*})\leq\max_{2\leq t\leq T}\frac{2}{1-\mu_{f}\eta_{t}}\left[\frac{D_{\psi}(x^{*},x^{1})}{\sum_{t=1}^{T}\gamma_{t}}+2\left(M^{2}+\sigma^{2}C(\delta,p)\right)\sum_{t=1}^{T}\frac{\gamma_{t}\eta_{t}}{\sum_{s=t}^{T}\gamma_{s}}\right],
\]
where 
\begin{equation}
C(\delta,p)\coloneqq\begin{cases}
\left(\frac{2e}{p}\lor e\log\frac{2e}{\delta}\right)^{\frac{2}{p}}+\left(4\sqrt{3}+3\sqrt{2}\right)^{2}\log^{\frac{2}{p}}\frac{4}{\delta}=O\left(1+\log^{\frac{2}{p}}\frac{1}{\delta}\right) & p\in\left[1,2\right)\\
\left(\frac{2e}{p}\lor e\log\frac{2e}{\delta}\right)^{\frac{2}{p}}+\frac{64\left[1\lor\log^{\frac{p+2}{p}}\left(4\left[3+\left(\frac{3}{p}\right)^{\frac{2}{p}}2\right]/\delta\right)\right]}{\log^{\frac{2}{p}}2}=O\left(1+\log^{1+\frac{2}{p}}\frac{1}{\delta}\right) & p\in\left(0,1\right)
\end{cases}.\label{eq:weibull-core-hp-C}
\end{equation}
\end{lem}

\begin{proof}
We first introduce the following sequence
\begin{equation}
y_{t\in\left[T\right]}\coloneqq\max_{\ell\in\left[t\right]}\gamma_{\ell}\eta_{\ell}^{-1}\left(1-\mu_{f}\eta_{\ell}\mathds{1}\left[\ell\geq2\right]\right)v_{\ell-1}D_{\psi}(z^{\ell-1},x^{\ell}).\label{eq:weibull-core-hp-y}
\end{equation}
Now, fix a time $s\in\left[T\right]$, we invoke Lemma \ref{lem:core-general-anytime}
to get
\begin{align*}
 & \gamma_{s}v_{s}\left(F(x^{s+1})-F(x^{*})\right)+\gamma_{s+1}\eta_{s+1}^{-1}(1-\mu_{f}\eta_{s+1})v_{s}D_{\psi}(z^{s},x^{s+1})\\
\leq & (1-\mu_{f}\eta_{1})v_{0}D_{\psi}(x^{*},x^{1})+\sum_{t=1}^{s}2\gamma_{t}\eta_{t}v_{t}(M^{2}+\left\Vert \xi^{t}\right\Vert _{*}^{2})+\sum_{t=1}^{s}\gamma_{t}v_{t-1}\left\langle \xi^{t},z^{t-1}-x^{t}\right\rangle .
\end{align*}
We can bound
\begin{align*}
 & \sum_{t=1}^{s}\gamma_{t}v_{t-1}\left\langle \xi^{t},z^{t-1}-x^{t}\right\rangle =\sum_{t=1}^{s}\gamma_{t}v_{t-1}\left\langle \xi^{t},\frac{z^{t-1}-x^{t}}{\sqrt{y_{t}}}\right\rangle \sqrt{y_{t}}\\
\overset{(a)}{\leq} & 2\sqrt{y_{s}}\max_{\ell\in\left[s\right]}\left|\sum_{t=1}^{\ell}\gamma_{t}v_{t-1}\left\langle \xi^{t},\frac{z^{t-1}-x^{t}}{\sqrt{y_{t}}}\right\rangle \right|\overset{(b)}{\leq}\frac{y_{s}}{2}+2\left(\max_{\ell\in\left[s\right]}\left|\sum_{t=1}^{\ell}\gamma_{t}v_{t-1}\left\langle \xi^{t},\frac{z^{t-1}-x^{t}}{\sqrt{y_{t}}}\right\rangle \right|\right)^{2},
\end{align*}
where $(a)$ is by Lemma C.2 of \citet{pmlr-v202-ivgi23a}, $(b)$
is due to AM-GM inequality. Hence, for any $s\in\left[T\right]$,
the following inequality holds
\begin{align}
 & \gamma_{s}v_{s}\left(F(x^{s+1})-F(x^{*})\right)+\gamma_{s+1}\eta_{s+1}^{-1}(1-\mu_{f}\eta_{s+1})v_{s}D_{\psi}(z^{s},x^{s+1})\nonumber \\
\leq & (1-\mu_{f}\eta_{1})v_{0}D_{\psi}(x^{*},x^{1})+\sum_{t=1}^{s}2\gamma_{t}\eta_{t}v_{t}(M^{2}+\left\Vert \xi^{t}\right\Vert _{*}^{2})+\frac{y_{s}}{2}+2\left(\max_{\ell\in\left[s\right]}\left|\sum_{t=1}^{\ell}\gamma_{t}v_{t-1}\left\langle \xi^{t},\frac{z^{t-1}-x^{t}}{\sqrt{y_{t}}}\right\rangle \right|\right)^{2}\nonumber \\
\leq & (1-\mu_{f}\eta_{1})v_{0}D_{\psi}(x^{*},x^{1})+\sum_{t=1}^{T}2\gamma_{t}\eta_{t}v_{t}(M^{2}+\left\Vert \xi^{t}\right\Vert _{*}^{2})+\frac{y_{s}}{2}+2\left(\max_{\ell\in\left[T\right]}\left|\sum_{t=1}^{\ell}\gamma_{t}v_{t-1}\left\langle \xi^{t},\frac{z^{t-1}-x^{t}}{\sqrt{y_{t}}}\right\rangle \right|\right)^{2}.\label{eq:weibull-core-hp-1}
\end{align}

Next, we use induction to prove for any $s\in\left[T\right]$
\begin{equation}
y_{s}\leq2v_{0}D_{\psi}(x^{*},x^{1})+4\sum_{t=1}^{T}\gamma_{t}\eta_{t}v_{t}(M^{2}+\left\Vert \xi^{t}\right\Vert _{*}^{2})+4\left(\max_{\ell\in\left[T\right]}\left|\sum_{t=1}^{\ell}\gamma_{t}v_{t-1}\left\langle \xi^{t},\frac{z^{t-1}-x^{t}}{\sqrt{y_{t}}}\right\rangle \right|\right)^{2}.\label{eq:weibull-core-hp-hypothesis}
\end{equation}
First for $s=1$, $y_{1}\overset{(\ref{eq:weibull-core-hp-y})}{=}v_{0}D_{\psi}(x^{*},x^{1})$
is smaller than R.H.S. of (\ref{eq:weibull-core-hp-hypothesis}).
Suppose (\ref{eq:weibull-core-hp-hypothesis}) holds for time $s$
where $1\leq s\leq T-1$. Then, (\ref{eq:weibull-core-hp-1}) implies
\begin{align*}
 & \gamma_{s+1}\eta_{s+1}^{-1}(1-\mu_{f}\eta_{s+1})v_{s}D_{\psi}(z^{s},x^{s+1})\\
\leq & 2v_{0}D_{\psi}(x^{*},x^{1})+\sum_{t=1}^{T}4\gamma_{t}\eta_{t}v_{t}(M^{2}+\left\Vert \xi^{t}\right\Vert _{*}^{2})+4\left(\max_{\ell\in\left[T\right]}\left|\sum_{t=1}^{\ell}\gamma_{t}v_{t-1}\left\langle \xi^{t},\frac{z^{t-1}-x^{t}}{\sqrt{y_{t}}}\right\rangle \right|\right)^{2},
\end{align*}
which gives us
\begin{align*}
y_{s+1} & =y_{s}\lor\gamma_{s+1}\eta_{s+1}^{-1}(1-\mu_{f}\eta_{s+1})v_{s}D_{\psi}(z^{s},x^{s+1})\\
 & \leq2v_{0}D_{\psi}(x^{*},x^{1})+\sum_{t=1}^{T}4\gamma_{t}\eta_{t}v_{t}(M^{2}+\left\Vert \xi^{t}\right\Vert _{*}^{2})+4\left(\max_{\ell\in\left[T\right]}\left|\sum_{t=1}^{\ell}\gamma_{t}v_{t-1}\left\langle \xi^{t},\frac{z^{t-1}-x^{t}}{\sqrt{y_{t}}}\right\rangle \right|\right)^{2}.
\end{align*}
Hence, the induction is completed.

Applying $s=T$ to (\ref{eq:weibull-core-hp-1}) and (\ref{eq:weibull-core-hp-hypothesis}),
we have
\[
\gamma_{T}v_{T}\left(F(x^{T+1})-F(x^{*})\right)\leq(2-\mu_{f}\eta_{1})v_{0}D_{\psi}(x^{*},x^{1})+\sum_{t=1}^{T}4\gamma_{t}\eta_{t}v_{t}(M^{2}+\left\Vert \xi^{t}\right\Vert _{*}^{2})+4\left(\max_{s\in\left[T\right]}\left|\sum_{t=1}^{s}\gamma_{t}v_{t-1}\left\langle \xi^{t},\frac{z^{t-1}-x^{t}}{\sqrt{y_{t}}}\right\rangle \right|\right)^{2}.
\]
Using Lemmas \ref{lem:weibull-hp-sigma} and \ref{lem:weibull-hp-inner},
with probability at least $1-\delta$, there are
\[
\sum_{t=1}^{T}\gamma_{t}\eta_{t}v_{t}\left\Vert \xi^{t}\right\Vert _{*}^{2}\leq\sum_{t=1}^{T}\gamma_{t}\eta_{t}v_{t}\sigma^{2}\left(\frac{2e}{p}\lor e\log\frac{2e}{\delta}\right)^{\frac{2}{p}},
\]
and
\[
\left(\max_{s\in\left[T\right]}\left|\sum_{t=1}^{s}\gamma_{t}v_{t-1}\left\langle \xi^{t},\frac{z^{t-1}-x^{t}}{\sqrt{y_{t}}}\right\rangle \right|\right)^{2}\leq\max_{2\leq t\leq T}\frac{\widetilde{C}(\delta,p)}{1-\mu_{f}\eta_{t}}\sum_{t=1}^{T}\gamma_{t}\eta_{t}v_{t}\sigma^{2},
\]
where 
\[
\widetilde{C}(\delta,p)\coloneqq\begin{cases}
\left(4\sqrt{3}+3\sqrt{2}\right)^{2}\log^{\frac{2}{p}}\frac{4}{\delta} & p\in\left[1,2\right)\\
\frac{64\left[1\lor\log^{\frac{p+2}{p}}\left(4\left[3+\left(\frac{3}{p}\right)^{\frac{2}{p}}2\right]/\delta\right)\right]}{\log^{\frac{2}{p}}2} & p\in\left(0,1\right)
\end{cases}.
\]
Thus, we obtain
\begin{align*}
 & \gamma_{T}v_{T}\left(F(x^{T+1})-F(x^{*})\right)\\
\leq & (2-\mu_{f}\eta_{1})v_{0}D_{\psi}(x^{*},x^{1})+4\left[M^{2}+\sigma^{2}\left(\left(\frac{2e}{p}\lor e\log\frac{2e}{\delta}\right)^{\frac{2}{p}}+\max_{2\leq t\leq T}\frac{\widetilde{C}(\delta,p)}{1-\mu_{f}\eta_{t}}\right)\right]\sum_{t=1}^{T}\gamma_{t}\eta_{t}v_{t}\\
\leq & \max_{2\leq t\leq T}\frac{2}{1-\mu_{f}\eta_{t}}\left[v_{0}D_{\psi}(x^{*},x^{1})+2\left(M^{2}+\sigma^{2}C(\delta,p)\right)\sum_{t=1}^{T}\gamma_{t}\eta_{t}v_{t}\right].
\end{align*}
Dividing both sides by $\gamma_{T}v_{T}$ and plugging in $v_{t\in\left\{ 0\right\} \cup\left[T\right]}=\frac{\gamma_{T}}{\sum_{s=t\lor1}^{T}\gamma_{s}}$,
the proof is finished.
\end{proof}

The proof of the following lemma is inspired by \citet{vladimirova2020sub}.
\begin{lem}
\label{lem:weibull-hp-sigma}Under Assumption \ref{enu:A5D}, for
any $\delta\in(0,1)$, with probability at least $1-\delta/2$, we
have
\[
\sum_{t=1}^{T}\gamma_{t}\eta_{t}v_{t}\left\Vert \xi^{t}\right\Vert _{*}^{2}\leq\sum_{t=1}^{T}\gamma_{t}\eta_{t}v_{t}\sigma^{2}\left(\frac{2e}{p}\lor e\log\frac{2e}{\delta}\right)^{\frac{2}{p}},
\]
where $\gamma_{t}$, $\eta_{t}$, $v_{t}$ are the same as in Lemma
\ref{lem:core-general-anytime}.
\end{lem}

\begin{proof}
Let $\Xi_{t}\coloneqq\gamma_{t}\eta_{t}v_{t}\left\Vert \xi^{t}\right\Vert _{*}^{2}$
and $\Lambda_{t}\coloneqq\gamma_{t}\eta_{t}v_{t}\sigma^{2}$ for $t\in\left[T\right]$.
By Lemma \ref{lem:sub-weibull-moment}, for any $k\geq1$, there is
\[
\E\left[\left|\Xi_{t}\right|^{k}\mid\F^{t}\right]\leq e\Lambda_{t}^{k}\left(\frac{2k}{p}\right)^{\frac{2k}{p}}\Rightarrow\left\Vert \Xi_{t}\right\Vert _{k}\leq\Lambda_{t}e^{\frac{1}{k}}\left(\frac{2k}{p}\right)^{\frac{2}{p}},
\]
where $\left\Vert \cdot\right\Vert _{k}\coloneqq\left(\E\left[\left|\cdot\right|^{k}\right]\right)^{\frac{1}{k}}$.
Hence, for any $k\geq1$,
\[
\left\Vert \sum_{t=1}^{T}\Xi_{t}\right\Vert _{k}\leq\sum_{t=1}^{T}\left\Vert \Xi_{t}\right\Vert _{k}\leq\left(\sum_{t=1}^{T}\Lambda_{t}\right)e^{\frac{1}{k}}\left(\frac{2k}{p}\right)^{\frac{2}{p}}.
\]
Let $k\coloneqq1\lor\left(\frac{p}{2}\log\frac{2e}{\delta}\right)$
and $\lambda\coloneqq\left(\sum_{t=1}^{T}\Lambda_{t}\right)\left(\frac{2e}{\delta}\right)^{\frac{1}{k}}\left(\frac{2k}{p}\right)^{\frac{2}{p}}$,
then by Markov's inequality
\[
\Pr\left[\sum_{t=1}^{T}\gamma_{t}\eta_{t}v_{t}\left\Vert \xi^{t}\right\Vert _{*}^{2}>\lambda\right]=\Pr\left[\left(\sum_{t=1}^{T}\Xi_{t}\right)^{k}>\lambda^{k}\right]\leq\frac{\left\Vert \sum_{t=1}^{T}\Xi_{t}\right\Vert _{k}^{k}}{\lambda^{k}}\leq\frac{\delta}{2},
\]
which implies with probability at least $1-\delta/2$
\[
\sum_{t=1}^{T}\gamma_{t}\eta_{t}v_{t}\left\Vert \xi^{t}\right\Vert _{*}^{2}\leq\left(\sum_{t=1}^{T}\Lambda_{t}\right)\left(\frac{2e}{\delta}\right)^{\frac{1}{k}}\left(\frac{2k}{p}\right)^{\frac{2}{p}}\overset{(a)}{\leq}\left(\sum_{t=1}^{T}\Lambda_{t}\right)\left(\frac{2ek}{p}\right)^{\frac{2}{p}}=\sum_{t=1}^{T}\gamma_{t}\eta_{t}v_{t}\sigma^{2}\left(\frac{2e}{p}\lor e\log\frac{2e}{\delta}\right)^{\frac{2}{p}},
\]
where $(a)$ is due to $\frac{2e}{\delta}\leq\exp\left(\frac{2k}{p}\right)$.
\end{proof}

Lemma \ref{lem:weibull-hp-inner} is stated for two cases, $p\in\left[1,2\right)$
and $p\in\left(0,1\right)$. To prove the first case, we only need
to use Lemma \ref{lem:weibull}. The other case of $p\in\left(0,1\right)$
is tricky. To deal with it, we need Theorem 1 in \citet{li2018note}.
\begin{lem}
\label{lem:weibull-hp-inner}Under Assumptions \ref{enu:A4} and \ref{enu:A5D},
for any $\delta\in(0,1)$, with probability at least $1-\delta/2$,
for any $s\in\left[T\right]$, we have
\[
\left|\sum_{t=1}^{s}\gamma_{t}v_{t-1}\left\langle \xi^{t},\frac{z^{t-1}-x^{t}}{\sqrt{y_{t}}}\right\rangle \right|\leq\begin{cases}
4\sqrt{3\log\frac{4}{\delta}\sum_{t=1}^{T}\frac{\gamma_{t}\eta_{t}v_{t}\sigma^{2}}{1-\mu_{f}\eta_{t}\mathds{1}\left[t\geq2\right]}}+3\sqrt{2}\left(\sum_{t=1}^{T}\left(\frac{\gamma_{t}\eta_{t}v_{t}\sigma^{2}}{1-\mu_{f}\eta_{t}\mathds{1}\left[t\geq2\right]}\right)^{\frac{q}{2}}\right)^{\frac{1}{q}}\log^{\frac{1}{p}}\frac{4}{\delta} & p\in\left[1,2\right)\\
\frac{8}{\log^{\frac{1}{p}}2}\left[1\lor\log^{\frac{p+2}{2p}}\left(4\left[3+\left(\frac{3}{p}\right)^{\frac{2}{p}}2\right]/\delta\right)\right]\sqrt{\sum_{t=1}^{T}\frac{\gamma_{t}\eta_{t}v_{t}\sigma^{2}}{1-\mu_{f}\eta_{t}\mathds{1}\left[t\geq2\right]}} & p\in\left(0,1\right)
\end{cases}.
\]
where $\gamma_{t}$, $\eta_{t}$, $v_{t}$ are the same as in Lemma
\ref{lem:core-general-anytime}, $y_{t}\coloneqq\max_{\ell\in\left[t\right]}\gamma_{\ell}\eta_{\ell}^{-1}\left(1-\mu_{f}\eta_{\ell}\mathds{1}\left[\ell\geq2\right]\right)v_{\ell-1}D_{\psi}(z^{\ell-1},x^{\ell})$,
and $q\coloneqq\frac{p}{p-1}$ when $p\in\left[1,2\right)$.
\end{lem}

\begin{proof}
Let $Z^{t}\coloneqq\gamma_{t}v_{t-1}\frac{z^{t-1}-x^{t}}{\sqrt{y_{t}}},\forall t\in\left[T\right]$,
our goal is to bound $\left|\sum_{t=1}^{s}\left\langle \xi^{t},Z^{t}\right\rangle \right|$.
A useful observation is that
\begin{align}
\left\Vert Z^{t}\right\Vert  & =\gamma_{t}v_{t-1}\frac{\left\Vert z^{t-1}-x^{t}\right\Vert }{\sqrt{y_{t}}}\leq\gamma_{t}v_{t-1}\sqrt{\frac{2D_{\psi}(z^{t-1},x^{t})}{\gamma_{t}\eta_{t}^{-1}\left(1-\mu_{f}\eta_{t}\mathds{1}\left[t\geq2\right]\right)v_{t-1}D_{\psi}(z^{t-1},x^{t})}}\nonumber \\
 & =\sqrt{\frac{2\gamma_{t}\eta_{t}v_{t-1}}{1-\mu_{f}\eta_{t}\mathds{1}\left[t\geq2\right]}}\overset{v_{t-1}\leq v_{t}}{\leq}\sqrt{\frac{2\gamma_{t}\eta_{t}v_{t}}{1-\mu_{f}\eta_{t}\mathds{1}\left[t\geq2\right]}}\coloneqq m_{t}.\label{eq:weibull-hp-inner-m}
\end{align}
In the following, we split the proof into three cases, $p\in\left(1,2\right)$,
$p=1$, and $p\in\left(0,1\right)$.

First, we consider the case $p\in\left(1,2\right)$. Let
\begin{equation}
w\coloneqq\sqrt{\frac{\log\frac{4}{\delta}}{6\sigma^{2}\sum_{t=1}^{T}m_{t}^{2}}}\land\left(\frac{\log\frac{4}{\delta}}{2^{q-1}(q-1)\sigma^{q}\left(\sum_{t=1}^{T}m_{t}^{q}\right)}\right)^{\frac{1}{q}}.\label{eq:weibull-hp-inner-w-1}
\end{equation}
Now, we define the following non-negative sequence with $R_{0}\coloneqq1$
and
\[
R_{s}\coloneqq\exp\left(\sum_{t=1}^{s}w\left\langle \xi^{t},Z^{t}\right\rangle -6w^{2}\sigma^{2}\left\Vert Z^{t}\right\Vert ^{2}-2^{q-1}w^{q}\sigma^{q}\left\Vert Z^{t}\right\Vert ^{q}\right)\in\F^{s},\forall s\in\left[T\right].
\]
We prove that $R_{t}$ is a supermartingale by
\[
\E\left[R_{t}\mid\F^{t-1}\right]=R_{t-1}\E\left[\exp\left(w\left\langle \xi^{t},Z^{t}\right\rangle -6w^{2}\sigma^{2}\left\Vert Z^{t}\right\Vert ^{2}-2^{q-1}w^{q}\sigma^{q}\left\Vert Z^{t}\right\Vert ^{q}\right)\mid\F^{t-1}\right]\overset{(a)}{\leq}R_{t-1},
\]
where $(a)$ is by applying Lemma \ref{lem:weibull} to $\E\left[\exp\left(w\left\langle \xi^{t},Z^{t}\right\rangle \right)\mid\F^{t-1}\right]$.
Then, we consider a stopping time
\[
\tau\coloneqq\min\left\{ s\in\left[T\right]:R_{s}>\frac{4}{\delta}\right\} ,
\]
with the convention $\min\emptyset=\infty.$ Note that
\begin{align}
\Pr\left[\exists s\in\left[T\right],R_{s}>\frac{4}{\delta}\right] & =\Pr\left[\tau\leq T\right]=\E\left[\mathds{1}\left[\tau\leq T\right]\right]<\frac{\delta}{4}\E\left[R_{\tau}\mathds{1}\left[\tau\leq T\right]\right]\nonumber \\
 & =\frac{\delta}{4}\E\left[R_{\tau\land T}\mathds{1}\left[\tau\leq T\right]\right]\leq\frac{\delta}{4}\E\left[R_{\tau\land T}\right]\overset{(b)}{\leq}\frac{\delta R_{0}}{4}=\frac{\delta}{4},\label{eq:weibull-hp-inner-Ville}
\end{align}
where $(b)$ is due to the optional stopping theorem. Hence, with
probability at least $1-\delta/4$, for any $s\in\left[T\right]$
there is $R_{s}\leq\frac{4}{\delta}$, which implies
\begin{align*}
\sum_{t=1}^{s}\left\langle \xi^{t},Z^{t}\right\rangle  & \leq\sum_{t=1}^{T}\left(6w\sigma^{2}\left\Vert Z^{t}\right\Vert ^{2}+2^{q-1}w^{q-1}\sigma^{q}\left\Vert Z^{t}\right\Vert ^{q}\right)+\frac{1}{w}\log\frac{4}{\delta}\\
 & \overset{(\ref{eq:weibull-hp-inner-m})}{\leq}6w\sigma^{2}\left(\sum_{t=1}^{T}m_{t}^{2}\right)+2^{q-1}w^{q-1}\sigma^{q}\left(\sum_{t=1}^{T}m_{t}^{q}\right)+\frac{1}{w}\log\frac{4}{\delta}\\
 & \overset{(\ref{eq:weibull-hp-inner-w-1})}{\leq}2\sigma\sqrt{6\log\frac{4}{\delta}\sum_{t=1}^{T}m_{t}^{2}}+\sigma2^{\frac{1}{p}}p(q-1)^{\frac{1}{q}}\left(\sum_{t=1}^{T}m_{t}^{q}\right)^{\frac{1}{q}}\log^{\frac{1}{p}}\frac{4}{\delta}\\
 & \leq2\sigma\sqrt{6\log\frac{4}{\delta}\sum_{t=1}^{T}m_{t}^{2}}+3\sigma\left(\sum_{t=1}^{T}m_{t}^{q}\right)^{\frac{1}{q}}\log^{\frac{1}{p}}\frac{4}{\delta},
\end{align*}
where the last step is due to $\max_{p\in\left(1,2\right)}2^{\frac{1}{p}}p(q-1)^{\frac{1}{q}}=3$.
Similarly, we also have that, with probability at least $1-\delta/4$,
for any $s\in\left[T\right]$,
\[
-\sum_{t=1}^{s}\left\langle \xi^{t},Z^{t}\right\rangle \leq2\sigma\sqrt{6\log\frac{4}{\delta}\sum_{t=1}^{T}m_{t}^{2}}+3\sigma\left(\sum_{t=1}^{T}m_{t}^{q}\right)^{\frac{1}{q}}\log^{\frac{1}{p}}\frac{4}{\delta}.
\]
Hence, we conclude that
\begin{equation}
\Pr\left[\max_{s\in\left[T\right]}\left|\left\langle \xi^{t},Z^{t}\right\rangle \right|\leq2\sigma\sqrt{6\log\frac{4}{\delta}\sum_{t=1}^{T}m_{t}^{2}}+3\sigma\left(\sum_{t=1}^{T}m_{t}^{q}\right)^{\frac{1}{q}}\log^{\frac{1}{p}}\frac{4}{\delta}\right]\geq1-\frac{\delta}{2}.\label{eq:weibull-hp-inner-1}
\end{equation}

Next, we consider the case $p=1$. Let 
\begin{equation}
w\coloneqq\sqrt{\frac{\log\frac{4}{\delta}}{6\sigma^{2}\sum_{t=1}^{T}m_{t}^{2}}}\land\frac{1}{2\sigma\max_{t\in\left[T\right]}m_{t}}.\label{eq:weibull-hp-inner-w-2}
\end{equation}
We re-define the following non-negative sequence with $R_{0}\coloneqq1$
and
\[
R_{s}\coloneqq\exp\left(\sum_{t=1}^{s}w\left\langle \xi^{t},Z^{t}\right\rangle -6w^{2}\sigma^{2}\left\Vert Z^{t}\right\Vert ^{2}\right)\in\F^{s},\forall s\in\left[T\right].
\]
We prove that $R_{t}$ is a supermartingale by
\[
\E\left[R_{t}\mid\F^{t-1}\right]=R_{t-1}\E\left[\exp\left(w\left\langle \xi^{t},Z^{t}\right\rangle -6w^{2}\sigma^{2}\left\Vert Z^{t}\right\Vert ^{2}\right)\mid\F^{t-1}\right]\overset{(c)}{\leq}R_{t-1},
\]
where $(c)$ is by $w\left\Vert Z^{t}\right\Vert \overset{(\ref{eq:weibull-hp-inner-m})}{\leq}wm_{t}\overset{(\ref{eq:weibull-hp-inner-w-2})}{\leq}\frac{1}{2\sigma}$
and applying Lemma \ref{lem:weibull} to $\E\left[\exp\left(w\left\langle \xi^{t},Z^{t}\right\rangle \right)\mid\F^{t-1}\right]$.
Then similar to (\ref{eq:weibull-hp-inner-Ville}), we have $\Pr\left[\exists s\in\left[T\right],R_{s}>\frac{4}{\delta}\right]\leq\frac{\delta}{4}$,
implying that, with probability at least $1-\delta/4$, for any $s\in\left[T\right]$,
\[
\sum_{t=1}^{s}\left\langle \xi^{t},Z^{t}\right\rangle \leq\sum_{t=1}^{T}6w\sigma^{2}\left\Vert Z^{t}\right\Vert ^{2}+\frac{1}{w}\log\frac{4}{\delta}\overset{(\ref{eq:weibull-hp-inner-m})}{\leq}6w\sigma^{2}\left(\sum_{t=1}^{T}m_{t}^{2}\right)+\frac{1}{w}\log\frac{4}{\delta}\overset{(\ref{eq:weibull-hp-inner-w-2})}{\leq}2\sigma\sqrt{6\log\frac{4}{\delta}\sum_{t=1}^{T}m_{t}^{2}}+3\sigma\max_{t\in\left[T\right]}m_{t}\log\frac{4}{\delta}.
\]
Similarly, we can show the same bound holds with probability at least
$1-\delta/4$ for $-\sum_{t=1}^{s}\left\langle \xi^{t},Z^{t}\right\rangle ,\forall s\in\left[T\right]$.
Therefore, we conclude
\begin{equation}
\Pr\left[\max_{s\in\left[T\right]}\left|\sum_{t=1}^{s}\left\langle \xi^{t},Z^{t}\right\rangle \right|\leq2\sigma\sqrt{6\log\frac{4}{\delta}\sum_{t=1}^{T}m_{t}^{2}}+3\sigma\max_{t\in\left[T\right]}m_{t}\log\frac{4}{\delta}\right]\geq1-\frac{\delta}{2}.\label{eq:weibull-hp-inner-2}
\end{equation}

So, for $p\in\left[1,2\right)$, (\ref{eq:weibull-hp-inner-1}) and
(\ref{eq:weibull-hp-inner-2}) together imply that, with probability
at least $1-\delta/2$, for any $s\in\left[T\right]$ (where $\left(\sum_{t=1}^{T}m_{t}^{q}\right)^{\frac{1}{q}}=\max_{t\in\left[T\right]}m_{t}$
when $q=\infty\Leftrightarrow p=1$)
\[
\left|\sum_{t=1}^{s}\left\langle \xi^{t},Z^{t}\right\rangle \right|\leq2\sigma\sqrt{6\log\frac{4}{\delta}\sum_{t=1}^{T}m_{t}^{2}}+3\sigma\left(\sum_{t=1}^{T}m_{t}^{q}\right)^{\frac{1}{q}}\log^{\frac{1}{p}}\frac{4}{\delta},
\]
which gives us (by (\ref{eq:weibull-hp-inner-m}))
\[
\left|\sum_{t=1}^{s}\gamma_{t}v_{t-1}\left\langle \xi^{t},\frac{z^{t-1}-x^{t}}{\sqrt{y_{t}}}\right\rangle \right|\leq4\sqrt{3\log\frac{4}{\delta}\sum_{t=1}^{T}\frac{\gamma_{t}\eta_{t}v_{t}\sigma^{2}}{1-\mu_{f}\eta_{t}\mathds{1}\left[t\geq2\right]}}+3\sqrt{2}\left(\sum_{t=1}^{T}\left(\frac{\gamma_{t}\eta_{t}v_{t}\sigma^{2}}{1-\mu_{f}\eta_{t}\mathds{1}\left[t\geq2\right]}\right)^{\frac{q}{2}}\right)^{\frac{1}{q}}\log^{\frac{1}{p}}\frac{4}{\delta}.
\]

For $p\in\left(0,1\right)$, by Theorem 1 in \citet{li2018note},
for any $z>0$
\[
\Pr\left[\exists s\in\left[T\right],\left|\sum_{t=1}^{s}\gamma_{t}v_{t-1}\left\langle \xi^{t},\frac{z^{t-1}-x^{t}}{\sqrt{y_{t}}}\right\rangle \right|>z\right]\leq2\left[3+\left(\frac{3}{p}\right)^{\frac{2}{p}}\frac{64\sum_{t=1}^{T}\sigma_{t}^{2}}{z^{2}}\right]\exp\left(-\left(\frac{z^{2}}{32\sum_{t=1}^{T}\sigma_{t}^{2}}\right)^{\frac{p}{p+2}}\right),
\]
where $\sigma_{t\in\left[T\right]}\coloneqq\frac{m_{t}\sigma}{\log^{\frac{1}{p}}2}$.
Thus, we take $z=u^{\frac{p+2}{2p}}\sqrt{32\sum_{t=1}^{T}\sigma_{t}^{2}}$
and $u=1\lor\log\left(4\left[3+\left(\frac{3}{p}\right)^{\frac{2}{p}}2\right]/\delta\right)$
to get
\[
\Pr\left[\exists s\in\left[T\right],\left|\sum_{t=1}^{s}\gamma_{t}v_{t-1}\left\langle \xi^{t},\frac{z^{t-1}-x^{t}}{\sqrt{y_{t}}}\right\rangle \right|>u^{\frac{p+2}{2p}}\sqrt{32\sum_{t=1}^{T}\sigma_{t}^{2}}\right]\leq2\left[3+\left(\frac{3}{p}\right)^{\frac{2}{p}}\frac{2}{u^{\frac{p+2}{p}}}\right]\exp\left(-u\right)\leq\frac{\delta}{2},
\]
which gives us (by (\ref{eq:weibull-hp-inner-m}))
\[
\left|\sum_{t=1}^{s}\gamma_{t}v_{t-1}\left\langle \xi^{t},\frac{z^{t-1}-x^{t}}{\sqrt{y_{t}}}\right\rangle \right|\leq\frac{8}{\log^{\frac{1}{p}}2}\left[1\lor\log^{\frac{p+2}{2p}}\left(4\left[3+\left(\frac{3}{p}\right)^{\frac{2}{p}}2\right]/\delta\right)\right]\sqrt{\sum_{t=1}^{T}\frac{\gamma_{t}\eta_{t}v_{t}\sigma^{2}}{1-\mu_{f}\eta_{t}\mathds{1}\left[t\geq2\right]}}.
\]
\end{proof}

\section{General Convex Functions under Sub-Weibull Noise\label{sec:weibull-cvx}}

In this section, we provide the full version of Theorem \ref{thm:main-weibull-cvx-hp}.
\begin{thm}[Full version of Theorem \ref{thm:main-weibull-cvx-hp}]
\label{thm:weibull-cvx-hp}Under Assumptions \ref{enu:A1}-\ref{enu:A4}
and \ref{enu:A5D} with $\mu_{f}=\mu_{h}=0$ and let $\delta\in(0,1)$:

If $T$ is unknown, by taking $\eta_{t\in\left[T\right]}=\frac{1}{2L}\land\frac{\eta}{\sqrt{t}}$
with $\eta=\Theta\left(\sqrt{\frac{D_{\psi}(x^{*},x^{1})}{M^{2}+\sigma^{2}C(\delta,p)}}\right)$,
then with probability at least $1-\delta$, there is
\[
F(x^{T+1})-F(x^{*})\leq O\left(\frac{LD_{\psi}(x^{*},x^{1})}{T}+\frac{1}{\sqrt{T}}\left[\frac{D_{\psi}(x^{*},x^{1})}{\eta}+\eta\left(M^{2}+\sigma^{2}C(\delta,p)\right)\log T\right]\right).
\]
In particular, by choosing $\eta=\Theta\left(\sqrt{\frac{D_{\psi}(x^{*},x^{1})}{M^{2}+\sigma^{2}C(\delta,p)}}\right)$,
there is
\[
F(x^{T+1})-F(x^{*})\leq O\left(\frac{LD_{\psi}(x^{*},x^{1})}{T}+\frac{(M+\sigma\sqrt{C(\delta,p)})\sqrt{D_{\psi}(x^{*},x^{1})}\log T}{\sqrt{T}}\right).
\]

If $T$ is known, by taking $\eta_{t\in\left[T\right]}=\frac{1}{2L}\land\frac{\eta}{\sqrt{T}}$
with $\eta=\Theta\left(\sqrt{\frac{D_{\psi}(x^{*},x^{1})}{(M^{2}+\sigma^{2}C(\delta,p))\log T}}\right)$,
then with probability at least $1-\delta$, there is
\[
F(x^{T+1})-F(x^{*})\leq O\left(\frac{LD_{\psi}(x^{*},x^{1})}{T}+\frac{1}{\sqrt{T}}\left[\frac{D_{\psi}(x^{*},x^{1})}{\eta}+\eta\left(M^{2}+\sigma^{2}C(\delta,p)\right)\log T\right]\right).
\]
In particular, by choosing $\eta=\Theta\left(\sqrt{\frac{D_{\psi}(x^{*},x^{1})}{(M^{2}+\sigma^{2}C(\delta,p))\log T}}\right)$,
there is
\[
F(x^{T+1})-F(x^{*})\leq O\left(\frac{LD_{\psi}(x^{*},x^{1})}{T}+\frac{(M+\sigma\sqrt{C(\delta,p)})\sqrt{D_{\psi}(x^{*},x^{1})\log T}}{\sqrt{T}}\right).
\]
\end{thm}

\begin{proof}
From Lemma \ref{lem:weibull-core-hp}, if $\eta_{t\in\left[T\right]}\leq\frac{1}{2L\lor\mu_{f}}$,
with probability at least $1-\delta$, there is
\begin{equation}
F(x^{T+1})-F(x^{*})\leq\left(\max_{2\leq t\leq T}\frac{2}{1-\mu_{f}\eta_{t}}\right)\left[\frac{D_{\psi}(x^{*},x^{1})}{\sum_{t=1}^{T}\gamma_{t}}+2\left(M^{2}+\sigma^{2}C(\delta,p)\right)\sum_{t=1}^{T}\frac{\gamma_{t}\eta_{t}}{\sum_{s=t}^{T}\gamma_{t}}\right],\label{eq:weibull-cvx-hp-1}
\end{equation}
where $\gamma_{t\in\left[T\right]}=\eta_{t}\prod_{s=2}^{t}\frac{1+\mu_{h}\eta_{s-1}}{1-\mu_{f}\eta_{s}}$.
Note that $\mu_{f}=\mu_{h}=0$ now, hence, both $\eta_{t\in\left[T\right]}=\frac{1}{2L}\land\frac{\eta}{\sqrt{t}}$
and $\eta_{t\in\left[T\right]}=\frac{1}{2L}\land\frac{\eta}{\sqrt{T}}$
satisfy $\eta_{t\in\left[T\right]}\leq\frac{1}{2L\lor\mu_{f}}=\frac{1}{2L}$.
Besides, $\gamma_{t\in\left[T\right]}$ will degenerate to $\eta_{t\in\left[T\right]}$.
Then we can simplify (\ref{eq:weibull-cvx-hp-1}) into
\[
F(x^{T+1})-F(x^{*})\leq\frac{2D_{\psi}(x^{*},x^{1})}{\sum_{t=1}^{T}\eta_{t}}+4\left(M^{2}+\sigma^{2}C(\delta,p)\right)\sum_{t=1}^{T}\frac{\eta_{t}^{2}}{\sum_{s=t}^{T}\eta_{s}}.
\]

If $\eta_{t\in\left[T\right]}=\frac{1}{2L}\land\frac{\eta}{\sqrt{t}}$,
similar to (\ref{eq:cvx-exp-6}), we will have
\[
F(x^{T+1})-F(x^{*})\leq O\left(\frac{LD_{\psi}(x^{*},x^{1})}{T}+\frac{1}{\sqrt{T}}\left[\frac{D_{\psi}(x^{*},x^{1})}{\eta}+\eta\left(M^{2}+\sigma^{2}C(\delta,p)\right)\log T\right]\right).
\]
By plugging in $\eta=\Theta\left(\sqrt{\frac{D_{\psi}(x^{*},x^{1})}{M^{2}+\sigma^{2}C(\delta,p)}}\right)$,
we get the desired bound.

If $\eta_{t\in\left[T\right]}=\frac{1}{2L}\land\frac{\eta}{\sqrt{T}}$,
similar to (\ref{eq:cvx-exp-7}), we will get
\[
F(x^{T+1})-F(x^{*})\leq O\left(\frac{LD_{\psi}(x^{*},x^{1})}{T}+\frac{1}{\sqrt{T}}\left[\frac{D_{\psi}(x^{*},x^{1})}{\eta}+\eta\left(M^{2}+\sigma^{2}C(\delta,p)\right)\log T\right]\right).
\]
By plugging in $\eta=\Theta\left(\sqrt{\frac{D_{\psi}(x^{*},x^{1})}{(M^{2}+\sigma^{2}C(\delta,p))\log T}}\right)$,
we get the desired bound.
\end{proof}

\subsection{Optimal Rate under Sub-Weibull Noise\label{subsec:weibull-cvx-hp-optimal}}

In this section, we provide the full version of Theorem \ref{thm:main-weibull-cvx-hp-optimal}.
\begin{thm}[Full version of Theorem \ref{thm:main-weibull-cvx-hp-optimal}]
\label{thm:weibull-cvx-hp-optimal}Under Assumptions \ref{enu:A1}-\ref{enu:A4}
and \ref{enu:A5D} with $\mu_{f}=\mu_{h}=0$ and let $\delta\in(0,1)$,
if $T$ is known, by taking $\eta_{t\in\left[T\right]}=\frac{T-t+1}{2LT}\land\frac{\eta(T-t+1)}{T^{\frac{3}{2}}}$,
then with probability at least $1-\delta$, there is
\[
F(x^{T+1})-F(x^{*})\leq O\left(\frac{LD_{\psi}(x^{*},x^{1})}{T}+\frac{1}{\sqrt{T}}\left[\frac{D_{\psi}(x^{*},x^{1})}{\eta}+\eta\left(M^{2}+\sigma^{2}C(\delta,p)\right)\right]\right).
\]
In particular, by choosing $\eta=\Theta\left(\sqrt{\frac{D_{\psi}(x^{*},x^{1})}{M^{2}+\sigma^{2}C(\delta,p)}}\right)$,
there is
\[
F(x^{T+1})-F(x^{*})\leq O\left(\frac{LD_{\psi}(x^{*},x^{1})}{T}+\frac{(M+\sigma\sqrt{C(\delta,p)})\sqrt{D_{\psi}(x^{*},x^{1})}}{\sqrt{T}}\right).
\]
\end{thm}

\begin{proof}
From Lemma \ref{lem:weibull-core-hp}, if $\eta_{t\in\left[T\right]}\leq\frac{1}{2L\lor\mu_{f}}$,
with probability at least $1-\delta$, there is
\begin{equation}
F(x^{T+1})-F(x^{*})\leq\left(\max_{2\leq t\leq T}\frac{2}{1-\mu_{f}\eta_{t}}\right)\left[\frac{D_{\psi}(x^{*},x^{1})}{\sum_{t=1}^{T}\gamma_{t}}+2\left(M^{2}+\sigma^{2}C(\delta,p)\right)\sum_{t=1}^{T}\frac{\gamma_{t}\eta_{t}}{\sum_{s=t}^{T}\gamma_{s}}\right],\label{eq:weibull-cvx-hp-optimal-1}
\end{equation}
where $\gamma_{t\in\left[T\right]}=\eta_{t}\prod_{s=2}^{t}\frac{1+\mu_{h}\eta_{s-1}}{1-\mu_{f}\eta_{s}}$.
Note that $\mu_{f}=\mu_{h}=0$ now, hence, $\eta_{t\in\left[T\right]}=\frac{T-t+1}{2LT}\land\frac{\eta(T-t+1)}{T^{\frac{3}{2}}}$
satisfies $\eta_{t\in\left[T\right]}\leq\frac{1}{2L\lor\mu_{f}}=\frac{1}{2L}$.
Besides, $\gamma_{t\in\left[T\right]}$ will degenerate to $\eta_{t\in\left[T\right]}$.
Then we can simplify (\ref{eq:weibull-cvx-hp-optimal-1}) into
\[
F(x^{T+1})-F(x^{*})\leq\frac{2D_{\psi}(x^{*},x^{1})}{\sum_{t=1}^{T}\eta_{t}}+4\left(M^{2}+\sigma^{2}C(\delta,p)\right)\sum_{t=1}^{T}\frac{\eta_{t}^{2}}{\sum_{s=t}^{T}\eta_{s}}.
\]
For $\eta_{t\in\left[T\right]}=\frac{T-t+1}{2LT}\land\frac{\eta(T-t+1)}{T^{\frac{3}{2}}}$,
similar to (\ref{eq:cvx-exp-optimal-3}), we will have
\[
F(x^{T+1})-F(x^{*})\leq O\left(\frac{LD_{\psi}(x^{*},x^{1})}{T}+\frac{1}{\sqrt{T}}\left[\frac{D_{\psi}(x^{*},x^{1})}{\eta}+\eta\left(M^{2}+\sigma^{2}C(\delta,p)\right)\right]\right).
\]
By plugging in $\eta=\Theta\left(\sqrt{\frac{D_{\psi}(x^{*},x^{1})}{M^{2}+\sigma^{2}C(\delta,p)}}\right)$,
we get the desired bound.
\end{proof}